\pdfoutput=1
\documentclass{article}

\usepackage[preprint]{neurips_2026}

\usepackage[utf8]{inputenc}
\usepackage[T1]{fontenc}
\usepackage{hyperref}
\usepackage{url}
\usepackage{booktabs}
\usepackage{amsfonts,amsmath,amssymb,amsthm}
\usepackage{microtype}
\usepackage{xcolor}
\usepackage{graphicx}
\usepackage[capitalize,noabbrev]{cleveref}
\hypersetup{pdftitle={Two-Sample Testing for Random Graphs without Vertex Correspondence},pdfauthor={Soham Dan}}

\theoremstyle{plain}
\newtheorem{theorem}{Theorem}
\newtheorem{proposition}{Proposition}
\newtheorem{lemma}{Lemma}[section]
\newtheorem{corollary}{Corollary}
\newtheorem*{theorem*}{Theorem}
\newtheorem*{proposition*}{Proposition}
\newtheorem*{corollary*}{Corollary}
\newtheorem*{lemma*}{Lemma}
\theoremstyle{definition}

\theoremstyle{remark}
\newtheorem{remark}{Remark}
\newtheorem*{remark*}{Remark}

\providecommand{\E}{\mathbb{E}}
\providecommand{\Prob}{\mathbb{P}}
\providecommand{\Var}{\operatorname{Var}}
\providecommand{\Cov}{\operatorname{Cov}}
\providecommand{\TV}{d_{\mathrm{TV}}}
\providecommand{\ind}{\mathbf{1}}
\providecommand{\eps}{\varepsilon}
\providecommand{\artanh}{\operatorname{artanh}}
\providecommand{\Bern}{\operatorname{Bern}}
\providecommand{\Bin}{\operatorname{Bin}}
\providecommand{\Odd}{\operatorname{Odd}}
\providecommand{\tinj}{t_{\mathrm{inj}}}
\providecommand{\Evn}{\mathcal{E}_n}
\providecommand{\Cyc}{\mathcal{C}}
\providecommand{\Cop}{\mathcal{K}}
\providecommand{\Pal}{P^{\mathrm{al},m}_{\delta}}

\providecommand{\ff}[2]{(#1)_{#2}}

\providecommand{\E}{\mathbb{E}}
\providecommand{\Prob}{\mathbb{P}}
\providecommand{\Var}{\operatorname{Var}}
\providecommand{\Cov}{\operatorname{Cov}}

\providecommand{\ind}{\mathbf{1}}
\providecommand{\tr}{\operatorname{tr}}
\providecommand{\TV}{d_{\mathrm{TV}}}
\providecommand{\chisq}{\chi^2}
\providecommand{\Gnp}{\mathbb{G}(n,p)}
\providecommand{\Gnw}[1]{\mathbb{G}(n,#1)}
\providecommand{\tinj}{t_{\mathrm{inj}}}
\providecommand{\mun}{m^\ast_{\mathrm{un}}}
\providecommand{\mal}{m^\ast_{\mathrm{al}}}
\providecommand{\Tsig}{T_p}
\providecommand{\Csig}{S_{C_4}}
\providecommand{\Kexc}{\mathrm{K3R}}

\newcommand{\EfourteenMstarClusA}{196.8}
\newcommand{\EfourteenMstarGraphletA}{\ensuremath{>256}}

\newcommand{\EfourteenMstarTriA}{65.5}

\newcommand{\EfourteenMstarClusB}{40.3}

\newcommand{\EfourteenMstarTriB}{16.2}

\newcommand{\EfourteenBlindMin}{0.030}
\newcommand{\EfourteenBlindMax}{0.075}
\newcommand{\EfourteenTheoryTriA}{64.2}
\newcommand{\EfourteenTheoryTriB}{13.9}

\newcommand{\EfourteenbMstarClus}{19.3}

\newcommand{\EfourteenbMstarSpec}{61.2}

\newcommand{\EfourteenbMstarTri}{7.4}

\newcommand{\EfourteenbMaxGin}{0.090}
\newcommand{\PosCtrlDeg}{1.00}
\newcommand{\PosCtrlGin}{1.00}
\newcommand{\PosCtrlTri}{0.05}
\newcommand{\ElevenSlopeUone}{-3.04\pm0.03}

\newcommand{\ElevenSlopeAone}{-0.91\pm0.02}
\newcommand{\ElevenSlopeAtwo}{-1.07\pm0.03}
\newcommand{\ElevenSlopeRatioAone}{-2.13\pm0.04}
\newcommand{\ElevenSlopeRatioAtwo}{-1.97\pm0.02}
\newcommand{\ElevenUoneLo}{1241}
\newcommand{\ElevenUoneTheoryLo}{1215}

\newcommand{\ElevenAoneLo}{11.3}
\newcommand{\ElevenAtwoLo}{4.53}
\newcommand{\ElevenRatioAoneLo}{110}
\newcommand{\ElevenRatioAtwoLo}{274}
\newcommand{\ElevenLBfactorLo}{12.2}

\newcommand{\ElevenRatioAoneHi}{10.4}
\newcommand{\ElevenRatioAtwoHi}{31.6}
\newcommand{\ElevenLBfactorHi}{37.2}

\newcommand{\ElevenKthreeRfactor}{4.0}
\newcommand{\SparseSlopeUone}{-3.00\pm0.03}
\newcommand{\SparseSlopeRatioAone}{-2.15\pm0.02}
\newcommand{\ElevenbSlopeCfourRtwo}{-3.98\pm0.03}
\newcommand{\ElevenbSlopeCfourRone}{-4.02\pm0.02}
\newcommand{\ElevenbCfourRtwoLo}{1392}
\newcommand{\ElevenbCfourRtwoHi}{17.4}
\newcommand{\ElevenbTriRtwoMaxPower}{0.063}
\newcommand{\ElevenbGridMax}{16384}
\newcommand{\TwelveUnEr}{3991}

\newcommand{\TwelveUnSbm}{11.1}

\newcommand{\TwelveKappaSbm}{8.33\cdot10^{-4}}
\newcommand{\TwelveUnSmooth}{83.7}

\newcommand{\TwelveKappaSmooth}{6.25\cdot10^{-4}}
\newcommand{\TwelveSlopeUnEr}{-5.99\pm0.03}
\newcommand{\TwelveRatioSbm}{360}
\newcommand{\TwelveRatioSmooth}{48}
\newcommand{\TwelveLocalSbm}{\ensuremath{-2.21}}
\newcommand{\TwelveLocalSmooth}{\ensuremath{-2.17}}
\newcommand{\TwelveLocalSbmSe}{0.10}
\newcommand{\TwelveLocalSmoothSe}{0.09}
\newcommand{\EthirteenSplitPowSbm}{0.06}
\newcommand{\EthirteenSplitPowEr}{0.02}
\newcommand{\TwelveSlopeUnSbmSpec}{\ensuremath{-2.80}}
\newcommand{\TwelveSlopeUnSmoothSpec}{\ensuremath{-3.18}}
\newcommand{\EfifteenDelta}{0.07}
\newcommand{\EfifteenGinSmoothHA}{0.80}
\newcommand{\EfifteenGinDcHA}{1.00}
\newcommand{\EfifteenGinOrthMax}{0.095}
\newcommand{\EfifteenGinOrthMin}{0.015}
\newcommand{\EfifteenGinOrthCells}{24}
\newcommand{\EfifteenGinMstarSmoothHA}{44.3}
\newcommand{\EfifteenGinMstarDcHA}{25.7}
\newcommand{\EfifteenDcPowEightMin}{1.00}
\newcommand{\EfifteenDegMax}{0.115}
\newcommand{\EfifteenTriPooledDcHO}{37.6}
\newcommand{\EfifteenSizeMin}{0.005}
\newcommand{\EfifteenSizeMax}{0.090}
\newcommand{\EsixteenRho}{5}
\newcommand{\EsixteenGraphsDd}{918}
\newcommand{\EsixteenTriLossDdDeg}{25}

\newcommand{\EsixteenTriLossDdWltwo}{4}

\newcommand{\EsixteenGraphsEnz}{587}
\newcommand{\EsixteenTriLossEnzDeg}{23}

\newcommand{\EsixteenTriLossEnzWltwo}{6}

\newcommand{\EsixteenMstarCfourEnz}{5.7}
\newcommand{\EsixteenPowerTriDdFour}{0.83}
\newcommand{\EsixteenTreeMax}{0.075}

\newcommand{\EsixteenWltwoTriDdPow}{0.50}
\newcommand{\EsixteenZTriDdEight}{0.28}
\newcommand{\EsixteenTriDdEight}{1.00}
\newcommand{\EsixteenWltwoTreeMax}{0.080}
\newcommand{\EsixteenSizeMin}{0.037}
\newcommand{\EsixteenSizeMax}{0.063}
\newcommand{\EsixteenSizeReps}{1200}

\title{Two-Sample Testing for Random Graphs\\ without Vertex Correspondence}

\author{Soham Dan}

\begin{document}

\maketitle

\begin{abstract}
Two populations of graphs often have to be compared without any correspondence between their vertices, for instance when networks come from different communities, or when a graph generative model is evaluated against held-out graphs. We study how many graphs such an unaligned two-sample test needs, and which graph statistics can detect which differences. For an Erd\H{o}s--R\'enyi null and a planted two-block difference that leaves every expected degree unchanged, we show that $m\asymp t^{-3}$ graphs per group are necessary and sufficient when the per-graph signal-to-noise ratio is $t<1$. Signed triangle counts attain this rate, and the lower bound holds for every graph size. With aligned vertices $m\asymp t^{-1}$ graphs suffice, so misalignment costs a factor of order $t^{-2}$. When the triangle signal cancels, the rate becomes $t^{-4}$ and $4$-cycles are needed. Statistics built from trees have exactly the same expectation under both hypotheses, and tests based on finitely many of them have asymptotically no power. In the graphon limit, this class includes degree distributions and message-passing graph neural network features. For a non-constant null, a generic difference is visible at first order, and a simple motif test attains the aligned order of sample size, suggesting that misalignment is costly mainly for differences that are invisible at low orders. We also give an exactly valid test for one or two graphs per group, at a cost in power. In our simulations, the fitted exponents are close to the predicted ones ($-3.04$ against $-3$ for unaligned tests and $-1.07$ against $-1$ for aligned ones), and degree-based and random-GNN evaluation metrics stay at their level in a setting where signed triangles need about $65$ graphs.
\end{abstract}

\section{Introduction}\label{sec:intro}

A two-sample test for networks decides whether two samples of graphs come from the same distribution. Often the vertices of different graphs do not correspond, as with friendship networks of different schools, protein networks of different species, or molecules. The same is true when a graph generative model is evaluated, since its samples are compared with held-out graphs through degree distributions, clustering coefficients, subgraph counts or random graph-neural-network embeddings \citep{you2018graphrnn,obray2022evaluation,thompson2022evaluation,krimmel2025polygraph}. Without a correspondence, a test can only use quantities that do not depend on how each graph is labelled.

With aligned vertices, the sample complexity of two-sample testing for inhomogeneous random graphs is well understood \citep{ghoshdastidar2020two,chatterjee2023two}. Without alignment, many tests have been proposed \citep{maugis2020testing,bravo2023quantifying,shao2026higher,tang2017nonparametric,sabanayagam2021graphon}, but two basic questions remain open: \emph{how many graphs are needed}, and \emph{which statistics can detect which differences}? In practice, an evaluation metric that is insensitive to a class of failures reports a good model however many samples are drawn.

In this paper, we study both questions for graphon models, in which every graph has its own latent vertex positions. Our results suggest that both answers depend on the lowest order at which a difference becomes visible to label-invariant statistics, and this order can be high for differences that are easy to detect with aligned vertices.
\begin{itemize}
\item \textbf{Sample complexity (\cref{thm:price}).} For an Erd\H{o}s--R\'enyi null and a planted two-block difference that preserves every expected degree, with per-graph signal-to-noise ratio $t<1$, we prove that $m\asymp t^{-3}$ graphs per group are necessary and sufficient and that signed triangles are optimal, with a lower bound that holds for every graph size. With aligned vertices $m\asymp t^{-1}$ graphs suffice, so misalignment costs a factor of order $t^{-2}$. When the triangle signal cancels, the rate becomes $t^{-4}$ and $4$-cycles are needed (\cref{prop:signcancel}).
\item \textbf{Tree statistics (\cref{thm:trees}).} If the null is regular and the difference preserves degrees, every tree density is exactly unchanged, and tests using finitely many tree statistics have asymptotic power equal to their level. Since degree statistics, colour refinement and message-passing graph neural networks are determined by tree densities on graphons \citep{grebik2022fractional,boker2023wl,boker2023fine}, they cannot detect such differences either.
\item \textbf{Order of visibility (\cref{thm:firstorder}).} For a non-constant null, a difference is generically visible at first order, and a motif test that does not know the null needs only order $1/t$ graphs, the aligned rate of \cref{thm:price}. So the factor $t^{-2}$ in \cref{thm:price} reflects first-order invisibility.
\item \textbf{Few graphs and experiments.} We give an exact permutation test for one graph per group (\cref{prop:split}), and our simulations confirm the predicted exponents. Degree- and GNN-based metrics miss degree-preserving changes both in simulations and on protein graphs, where degree-corrected signed cycles reach power $1/2$ with at most six graphs.
\end{itemize}
The second moment behind our lower bound for one graph is due to \citet{mossel2015reconstruction} and \citet{banerjee2018contiguity}, \citet{ghoshdastidar2017two} use the same planted partition as a hard instance for one graph per group, and the sharp aligned threshold $mt=1$ is known for sparse multilayer block models \citep{chatterjee2022multilayer}. The population bounds, their comparison with alignment, the distributional invariance of tree statistics and the first-order analysis are new.

\section{Setup}\label{sec:setup}

\paragraph{Graphons and graph samples.} A graphon is a symmetric measurable function $W:[0,1]^2\to[0,1]$. A graph $G\sim\Gnw{W}$ on $[n]=\{1,\dots,n\}$ is generated by drawing latent positions $U_1,\dots,U_n$ independently and uniformly from $[0,1]$, and then independent edges $A_{ij}\sim\mathrm{Bernoulli}(W(U_i,U_j))$ for $i<j$. The degree function is $d_W(x)=\int_0^1 W(x,y)\,dy$, and $W$ is \emph{regular} if $d_W$ is constant. For a finite graph $F$, $t(F,W)$ denotes its homomorphism density \citep{lovasz2012large}. The injective density $\tinj(F,G)$ of a graph $G$ is the fraction of injective maps $V(F)\to[n]$ that send edges to edges, and it satisfies $\E\,\tinj(F,G)=t(F,W)$ for $G\sim\Gnw{W}$.

\paragraph{Unaligned two-sample testing.} We observe $G_1,\dots,G_m\sim P$ and $G_1',\dots,G_m'\sim Q$, all independent, and test $H_0:P=Q$ against $H_1:P\neq Q$. A test $\psi$ maps the $2m$ graphs to $\{0,1\}$, and its \emph{risk} against a pair of hypotheses is the sum of its type I and type II errors. The samples are \emph{unaligned}: every graph carries its own latent positions, so a vertex label in one graph says nothing about a vertex label in another. For a graphon model this is automatic, because the law of $\Gnw{W}$ is invariant under relabelling. Averaging any test over independent uniform relabellings of each graph therefore preserves both of its errors, and we may restrict attention to (possibly randomized) tests that are invariant under relabelling within each graph. Invariant polynomials of bounded degree are exactly linear combinations of subgraph counts \citep[Lemma 3.2]{carpentier2025low}. In the \emph{aligned} version of the problem, all graphs share one vertex labelling and one set of latent positions.

\paragraph{The planted-partition difference.} Our main example has an Erd\H{o}s--R\'enyi null $P_0=\Gnp$ and the alternative $P_\delta$ in which a vector $\sigma\in\{\pm1\}^n$ with independent uniform entries is drawn and, given $\sigma$, edges are independent with
\begin{equation}\label{eq:pp}
\Prob(A_{ij}=1\mid\sigma)=p+\delta\,\sigma_i\sigma_j,\qquad 0<\delta\le\min\{p,1-p\}.
\end{equation}
This is the graphon $p+H$ with $H(x,y)=\delta\,s(x)s(y)$ and $s=\ind_{[0,1/2)}-\ind_{[1/2,1]}$. Since $\int H(x,y)\,dy=0$, every vertex has the same expected degree under both hypotheses. We write
\begin{equation}\label{eq:snr}
\eta=\frac{\delta^2}{p(1-p)},\qquad t=n\eta,
\end{equation}
and call $t$ the per-graph signal-to-noise ratio. In the notation $a=n(p+\delta)$ and $b=n(p-\delta)$ of the community detection literature, $t=(a-b)^2/(2(1-p)(a+b))$. The value $t=1$ is the threshold below which a single graph from $P_\delta$ cannot be distinguished from $\Gnp$ \citep{mossel2015reconstruction,banerjee2018contiguity}. Unaligned samples draw an independent $\sigma$ for each graph. Aligned samples use one $\sigma$ for all graphs of the first group.

\paragraph{Order of visibility.} For graphons $W$ and $W+H$ and a motif $F$, expanding the product over the edges of $F$ gives
\begin{equation}\label{eq:expansion}
t(F,W+H)-t(F,W)=\sum_{\emptyset\neq S\subseteq E(F)} t\bigl(F;\,H \text{ on } S,\ W \text{ on } E(F)\setminus S\bigr),
\end{equation}
where each term integrates $H$ over the edges in $S$ and $W$ over the remaining edges of $F$. The terms with $|S|=1$ sum to $\sum_{e\in E(F)}\langle H,W_{F,e}\rangle$, where $W_{F,e}(x,y)$ is the density of $F$ with the endpoints of $e$ fixed at $(x,y)$ and the edge $e$ removed. They are closely related to the two-point conditional kernels of \citet{chatterjee2026higher}. We say that $H$ is \emph{visible at order $k$} (relative to $W$) if some motif has a non-vanishing term with $|S|=k$ and none has one with $|S|<k$. When $W$ is constant and $H$ has zero degree function, integrating out a vertex of degree one in $S$ gives zero, so only edge sets in which every vertex has degree at least two survive. The smallest such sets are triangles and $4$-cycles, with densities $\tr H^3$ and $\tr H^4>0$, and hence $H$ is visible at order $3$ if $\tr H^3\neq0$ and at order $4$ otherwise. For \eqref{eq:pp}, $\tr H^3=\delta^3$ and the order is $3$.

\section{Sample complexity under an Erd\H{o}s--R\'enyi null}\label{sec:price}

\begin{figure}[t]
\centering
\includegraphics[width=\linewidth]{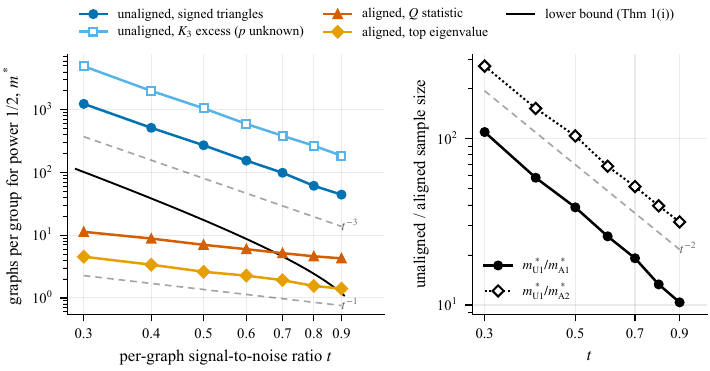}
\caption{Left: number of graphs per group at which each test reaches power $1/2$ at level $0.05$, for the planted partition \eqref{eq:pp} with $n=300$ and $p=1/2$. Unaligned tests (signed triangles with known $p$, excess triangle density with unknown $p$) follow the slope $-3$ of \cref{thm:price}. Aligned tests (cross-graph statistic $Q$, top eigenvalue of the averaged adjacency matrix) follow the slope $-1$. Below the black curve, the information bound of \cref{thm:price}(i) rules out power $1/2$ for every test. Right: the ratio of unaligned to aligned sample sizes, with a reference line proportional to $t^{-2}$.}
\label{fig:price}
\end{figure}

Write $q=p(1-p)$. Let $g(t)=\sum_{k\ge3}t^k/(2k)=-\tfrac12\log(1-t)-\tfrac t2-\tfrac{t^2}4$ for $t\in[0,1)$, and let $g_n(t)=\sum_{k=3}^n\frac{(n)_k}{2k\,n^k}t^k\le g(t)$ be its finite-$n$ version, where $(n)_k=n!/(n-k)!$. Let $c_{\alpha\beta}=\log(1+4(1-\alpha-\beta)^2)$.

\begin{theorem}[Unaligned and aligned sample complexity]\label{thm:price}
Let $n\ge 3$, $p\in(0,1)$, $0<\delta\le\min\{p,1-p\}$ and $t=n\delta^2/q$.
\begin{enumerate}
\item[(i)] \emph{Unaligned, lower bound.} For $L=dP_\delta/dP_0$ and every $m\ge1$,
\[
1+\chisq\bigl(P_\delta^{\otimes m}\,\big\|\,P_0^{\otimes m}\bigr)=\Bigl(\sum_{S\subseteq E(K_n)\ \mathrm{even}}\eta^{|S|}\Bigr)^m\le e^{m\,g_n(t)} .
\]
Hence a two-sample test with type I error at most $\alpha$ and type II error at most $\beta$ exists only if $m\ge c_{\alpha\beta}/g_n(t)$. For $t<1$ this requires $m\ge 6(1-t)\,c_{\alpha\beta}/t^3$.
\item[(ii)] \emph{Unaligned, upper bound.} Let $\Tsig(G)=\sum_{i<j<k}(A_{ij}-p)(A_{jk}-p)(A_{ik}-p)$. The test that rejects when
\[
\bigl|\textstyle\sum_{\ell}\Tsig(G_\ell)-\sum_\ell\Tsig(G_\ell')\bigr|\ge\bigl(2m\tbinom n3q^3/\alpha\bigr)^{1/2}
\]
has size at most $\alpha$ under $P=Q=\Gnp$. Its power is at least $1-\beta$ whenever
\[
m\,t^3\,\tfrac{(n-1)(n-2)}{6n^2}\ge\bigl(\sqrt{2/\alpha}+\sqrt{(2+\varepsilon_n)/\beta}\bigr)^2,
\]
where $\varepsilon_n\to0$ is explicit (\cref{app:price}).
\item[(iii)] \emph{Aligned.} Suppose the $m$ graphs of the first sample share one $\sigma$. With $x_{\alpha\beta}\in(0,1)$ the root of $x^2/4+g(x)=c_{\alpha\beta}$, no test has errors $(\alpha,\beta)$ if $mt\,\mathrm{artanh}(\eta)/\eta<x_{\alpha\beta}$, where $x_{0.05,0.5}\approx0.87$. The cross-graph statistic $Q=\sum_{\ell<\ell'}\sum_{i<j}(A^{\ell}_{ij}-p)(A^{\ell'}_{ij}-p)$ gives errors $(\alpha,\beta)$ against the null $P=Q=\Gnp$ once $m(m-1)t^2(n-1)/(4n)\ge\bigl(\alpha^{-1/2}+(\rho_m/\beta)^{1/2}\bigr)^2$, where $\rho_m\to1$ is explicit.
\end{enumerate}
\end{theorem}

Parts (ii) and (iii) use the null edge density $p$. For the triangle test, the pooled edge density of all $2m$ graphs can replace $p$ with no asymptotic loss when $np(1-p)\to\infty$ (\cref{prop:plugin}), and conditioning on the total number of edges gives an exactly valid Monte Carlo version.

\begin{corollary}\label{cor:price}
Fix $\alpha,\beta>0$ with $\alpha+\beta<1$. Let $\mun$ and $\mal$ be the smallest numbers of graphs per group for which some test has errors $(\alpha,\beta)$ in the unaligned and in the aligned problem. There are constants $0<c_1<c_2$, depending only on $(\alpha,\beta)$, such that for $n\ge n_0(\alpha,\beta,p)$ and $0<t\le1/2$
\[
c_1\,t^{-2}\;\le\;\mun/\mal\;\le\;c_2\,t^{-2}.
\]
Misalignment multiplies the number of graphs needed by a factor of order $t^{-2}$.
\end{corollary}

\paragraph{Proof idea.} With $L=\E_\sigma\prod_{i<j}(1+\delta\sigma_i\sigma_jX_{ij})$ and $X_{ij}=(A_{ij}-p)/q$, two independent partitions give $\E_0[L_\sigma L_\tau]=\prod_{i<j}(1+\eta\,\omega_i\omega_j)$ with $\omega=\sigma\circ\tau$. Averaging over $\omega$ removes every monomial with a vertex of odd degree. By Veblen's theorem each remaining edge set is an edge-disjoint union of cycles, and the sum is at most $\prod_C(1+\eta^{|C|})\le e^{g_n(t)}$. Edges and $2$-stars drop out, and the $\binom n3$ triangles contribute the leading $t^3/6$. With aligned graphs each edge contributes $(1+\eta\,\omega_i\omega_j)^m$, whose even part leaves a second-order term of about $m^2t^2/4$, so aligned tests succeed at order $t^{-1}$.

\paragraph{Sign cancellation.} In the following model, triangles carry no signal. Let $\sigma,\tau\in\{\pm1\}^n$ be independent with uniform entries, and
\begin{equation}\label{eq:rank2}
\Prob(A_{ij}=1\mid\sigma,\tau)=p+\delta(\sigma_i\sigma_j-\tau_i\tau_j),\qquad 2\delta\le\min\{p,1-p\}.
\end{equation}
The difference has eigenvalues $\pm\delta$, so $\tr H^3=0$ and $\tr H^4=2\delta^4$.

\begin{proposition}[Sign cancellation]\label{prop:signcancel}
Under \eqref{eq:rank2}, with $t=n\delta^2/q$:
\begin{enumerate}
\item[(i)] $\E[\Tsig(G)\mid\sigma,\tau]=0$, so triangles carry no signal. The two-sample triangle test has asymptotic power $\alpha$ for fixed $p$, $t$ and $m$.
\item[(ii)] For $t<1/2$, a test with errors $(\alpha,\beta)$ requires $m\ge(1-4t^2)\,c_{\alpha\beta}/(4t^4)$.
\item[(iii)] The signed $4$-cycle count $\Csig(G)=\sum_{4\text{-cycles}}\prod_{e}(A_e-p)$ has mean $6\binom n4\delta^4$ and null variance $3\binom n4q^4$. Its per-graph squared signal-to-noise ratio tends to $t^4/2$, and the two-sample $4$-cycle test has errors $(\alpha,\beta)$ once $n\ge n_0$ and $m\,t^4\ge C_{\alpha\beta}$.
\end{enumerate}
\end{proposition}

The lower bound rests on $\E\prod_{e\in S}(\sigma_i\sigma_j-\tau_i\tau_j)=2^{\beta(S)}\ind\{S\text{ even and bipartite}\}$, where $\beta(S)$ is the cyclomatic number, so odd cycles drop out of the second moment. It also explains why signed quadrilaterals can beat signed triangles \citep{jin2021optimal}, and why the statistic $t(K_2,G)^4-t(C_4,G)$ of \citet{chatterjee2026higher} is slower than triangles when $\tr H^3\neq0$ but cannot be cancelled.

\section{Tree statistics and degree-preserving differences}\label{sec:trees}

The expansion \eqref{eq:expansion} already suggests that statistics built from trees cannot detect degree-preserving differences, since every tree has a leaf and integrating the leaf of an $H$-edge returns the degree function of $H$, which is zero. The next theorem makes this exact and extends it to distributions.

\begin{theorem}[Invariance of tree statistics]\label{thm:trees}
\begin{enumerate}
\item[(i)] Let $W$ be a graphon with $d_W\equiv d$, and let $H$ be bounded and symmetric with $\int_0^1H(x,y)\,dy=0$ for almost every $x$ and $0\le W+H\le1$. Then $t(T,W+H)=t(T,W)=d^{\,e(T)}$ for every finite forest $T$. Consequently $\E\,\tinj(T,G)$ is the same under $\Gnw{W}$ and $\Gnw{W+H}$ for every $n$.
\item[(ii)] In the planted-partition model \eqref{eq:pp}, fix $p\in(0,1)$, $t>0$, $m\ge1$ and finitely many trees $T_1,\dots,T_r$. For each tree, let $Z_T(G)$ be the signed count $\sum\prod_{e}(A_e-p)$ over the copies of $T$ in $K_n$, divided by its null standard deviation. As $n\to\infty$, the vector $(Z_{T_j}(G_\ell),Z_{T_j}(G'_\ell))_{j,\ell}$ converges to $N(0,I_{2mr})$ under both $H_0$ and $H_1$. Hence every test that depends on the graphs only through these counts has asymptotic power equal to its asymptotic size. The signed triangle $z$-test has asymptotic power $1-\Phi\bigl(z_{1-\alpha}-(mt^3/12)^{1/2}\bigr)>\alpha$.
\end{enumerate}
\end{theorem}

We prove part (i) by removing leaves one at a time, and part (ii) by a conditional central limit theorem given $\sigma$. A signed tree count has mean zero under $P_\delta$ because every leaf has odd degree, and its covariances change only through pairs of trees whose union contains a cycle, which contribute a relative change of order $1/n$ (\cref{app:trees}). Interestingly, the statement also holds above the one-graph threshold ($t>1$), where a single graph already reveals $\sigma$ through its spectrum. Under $P_\delta$, every vertex degree is exactly $\mathrm{Binomial}(n-1,p)$, since the signs $\sigma_v\sigma_j$, $j\neq v$, are independent and uniform, and the edge count has exactly the null mean and variance.

\paragraph{Consequences for invariant learned features.} Colour refinement (the $1$-dimensional Weisfeiler--Leman algorithm) and message-passing graph neural networks cannot distinguish graphons with equal tree homomorphism densities, since on graphons their expressive power is exactly that of tree densities \citep{grebik2022fractional,boker2023wl,boker2023fine}, and the same holds for degree statistics. \Cref{thm:trees} therefore applies to many two-sample statistics in common use, including maximum mean discrepancies between degree distributions, Weisfeiler--Leman subtree kernels \citep{shervashidze2011weisfeiler}, and features from randomly initialized or trained message-passing networks \citep{thompson2022evaluation}. Against a regular null their population targets coincide for every degree-preserving difference, and for finitely many tree statistics \cref{thm:trees}(ii) gives zero asymptotic power at any fixed $m$, whereas clustering coefficients, orbit counts and spectra involve cycles and are not affected. One consequence concerns the evaluation of graph generative models \citep{obray2022evaluation,thompson2022evaluation}. Against a regular null, a model that reproduces the degree structure of its training graphs but gets their community or cycle structure wrong cannot be detected by metrics built from degrees or message passing, as \cref{sec:experiments} shows. Weisfeiler--Leman tests of dimension at least two do not share this limitation, since they determine homomorphism counts of all graphs of treewidth two, cycles included \citep{dell2018lovasz,morris2019weisfeiler}.

The limitation needs both a degree-preserving difference and a regular null. For a planted subgraph on top of $\Gnp$, which always increases degrees, signed star counts are optimal among constant-degree tests \citep{yu2025counting}. Under an irregular null, a vertex at position $x$ has degree $\mathrm{Binomial}(n-1,d_W(x))$ under both hypotheses, so the expected degree histogram is the same. Trees with an internal edge acquire a first-order term, equal to $\iint H(x,y)\,d_W(x)\,d_W(y)\,dx\,dy$ for the path with three edges and non-zero for generic $H$, and colour refinement separates the graphons (\cref{sec:firstorder}).

\section{Non-constant nulls and first-order visibility}\label{sec:firstorder}

The $t^{-3}$ rate in \cref{thm:price} comes from a special feature of the Erd\H{o}s--R\'enyi null. There every kernel $W_{F,e}$ in \eqref{eq:expansion} is constant, so a difference with zero degree function has $\langle H,W_{F,e}\rangle=0$ for every motif, and it cannot be seen before order three. For a non-constant null the first-order term is generically non-zero, and a simple motif test that does not need to know $W$ already detects the difference.

\begin{theorem}[First-order visibility]\label{thm:firstorder}
Let $F$ be a fixed graph with $v$ vertices and $e\ge1$ edges, and let $W$ be a graphon. Let $D$ be the difference between the two sample averages of $\tinj(F,\cdot)$.
\begin{enumerate}
\item[(i)] $\Var\,\tinj(F,G)\le v^2/(4n)$ for $G\sim\Gnw{W}$, for every graphon $W$ and $n\ge v$.
\item[(ii)] If $0\le W+H\le1$, then $\Delta_F:=t(F,W+H)-t(F,W)=\sum_{e\in E(F)}\langle H,W_{F,e}\rangle+R$, with $|R|\le(2^{e}-1-e)\|H\|_2^2$.
\item[(iii)] The test that rejects when $|D|\ge v/\sqrt{2nm\alpha}$ has size at most $\alpha$ under every graphon null $P=Q=\Gnw{W}$. Its power against $P=\Gnw{W+H}$, $Q=\Gnw{W}$ is at least $1-\beta$ whenever $n\,m\,\Delta_F^2\ge v^2(1/\alpha+1/\beta)$.
\end{enumerate}
\end{theorem}

Part (i) is the Efron--Stein inequality applied to the independent blocks $(U_i,(V_{ij})_{j<i})$, where $A_{ij}=\ind\{V_{ij}\le W(U_i,U_j)\}$ with i.i.d.\ uniform $V_{ij}$, since changing one block moves $\tinj(F,G)$ by at most $v/n$, and part (iii) is Chebyshev's inequality under both hypotheses. No constant depends on $W$.

\begin{corollary}[Order of visibility]\label{cor:visibility}
Let $H=cH_0$, and suppose the first-order coefficient $\kappa_F=\sum_{e\in E(F)}\langle H_0,W_{F,e}\rangle$ is non-zero. If $0<c\le|\kappa_F|/(2(2^e-1-e)\|H_0\|_2^2)$, the test of \cref{thm:firstorder}(iii) has power at least $1-\beta$ as soon as $m\ge4v^2(1/\alpha+1/\beta)/(nc^2\kappa_F^2)$. With $t_H=n\|H\|_2^2$, the unaligned sample size is therefore $O(1/t_H)$, the order of an aligned test under every null (\cref{prop:alignedW}) and at a constant null (\cref{thm:price}(iii)). If $W$ is constant and $H_0$ has zero degree function, then $\kappa_F=0$ for every $F$ and the rate becomes $t^{-3}$ or worse (\cref{thm:price,prop:signcancel}).
\end{corollary}

The corollary is only an upper bound, since for a non-constant null we do not prove matching lower bounds for either aligned or unaligned tests (\cref{sec:discussion}). Degeneracy of the null can improve the constant further. If $W$ is $F$-regular, the variance in (i) drops from order $n^{-1}$ to order $n^{-2}$ \citep{bhattacharya2023fluctuations,chatterjee2026higher}, and the test gains a factor of $n$.

Degree-corrected block models have finite rank. For a degree-preserving difference that is orthogonal to their eigenfunctions, every path density is unchanged and every cycle density moves by exactly $t(C_k,H)$, as under the Erd\H{o}s--R\'enyi null (\cref{prop:orthogonal}), so we expect the $t^{-3}$ rate to persist. Other degree-preserving differences are generically visible at first order.

\section{Testing with one to three graphs per group}\label{sec:calibration}

The rates above matter most for small $m$, which is where standard calibrations fail. Permuting whole graphs is exact, but for a two-sided statistic its smallest $p$-value is $2/\binom{2m}{m}$, which equals $1$, $1/3$ and $1/10$ for $m=1,2,3$. Normal approximations estimate a variance from only $2m-2$ degrees of freedom, and the network jackknife is conservative in expectation \citep{lin2020jackknife}. However, graphon models have a simple structure that makes exact tests possible with a single graph per group.

\begin{proposition}[Exact vertex-split permutation test]\label{prop:split}
Partition $[n]$ into fixed blocks $V_1,\dots,V_B$. If $G\sim\Gnw{W}$, the induced subgraphs $G[V_1],\dots,G[V_B]$ are independent, and $G[V_b]\sim\mathbb G(|V_b|,W)$. If both samples come from $\Gnw{W}$ and the blocks have equal size, the $2mB$ values of any statistic computed on the induced subgraphs are therefore independent and identically distributed. A permutation test over their group labels has level at most $\alpha$ for every $m$ and $B$, and a two-sided version can reject as soon as $\binom{2mB}{mB}\ge2/\alpha$, for instance when $m=1$ and $B=4$.
\end{proposition}

This uses only motifs inside a block. For signed triangles under an Erd\H{o}s--R\'enyi null, a block of size $n/B$ has per-graph signal-to-noise ratio $t/B$, so the squared signal-to-noise ratio of a graph drops from about $t^3/6$ to $t^3/(6B^2)$. With $m=1$ and $B=4$ its power in our experiments is only \EthirteenSplitPowSbm{} for a first-order visible difference and \EthirteenSplitPowEr{} under the Erd\H{o}s--R\'enyi null (\cref{tab:calibration} in \cref{app:e13}), so it mainly guarantees validity. For a mixture of graphons the blocks are only conditionally independent, so the test is not exact.

\section{Experiments}\label{sec:experiments}

Apart from one study of protein graphs, the experiments use simulated graphs, so the ground truth is known. Every number below is read from result files produced by the scripts in the arXiv ancillary files, which use only NumPy and SciPy. For each test we report $m^\ast$, the smallest number of graphs per group at which the power reaches $0.5$ at level $0.05$, interpolated on a geometric grid of $m$. Settings, replicate counts and run times are in \cref{app:experiments}.

\paragraph{Sample sizes with and without alignment.} \Cref{fig:price} uses the planted partition with $n=300$, $p=1/2$ and $t\in\{0.3,\dots,0.9\}$. The fitted slopes of $\log m^\ast$ on $\log t$ are $\ElevenSlopeUone$ for signed triangles (theory $-3$), $\ElevenSlopeAtwo$ for the aligned spectral test (theory $-1$), $\ElevenSlopeAone$ for $Q$, whose signal-to-noise ratio grows like $\sqrt{m(m-1)}$ at small $m$, and $\ElevenSlopeRatioAone$ and $\ElevenSlopeRatioAtwo$ for the unaligned-to-aligned ratio (theory $-2$).
At $t=0.3$ signed triangles need \ElevenUoneLo{} graphs per group, against \ElevenUoneTheoryLo{} predicted by the normal approximation. The aligned tests need \ElevenAoneLo{} and \ElevenAtwoLo{}, a ratio of \ElevenRatioAoneLo{} to \ElevenRatioAtwoLo{}, and the ratio is \ElevenRatioAoneHi{} to \ElevenRatioAtwoHi{} at $t=0.9$.

With a pooled edge density, an unknown $p$ costs nothing, whereas the per-graph excess triangle density needs \ElevenKthreeRfactor{} times more graphs because it keeps the degree-variance fluctuation. The information bound lies a factor of \ElevenLBfactorLo{} to \ElevenLBfactorHi{} below the triangle test, a gap due to the constant in Le Cam's inequality, the second sample and longer cycles, relevant near $t=1$ (\cref{rem:sharp}). A sparse panel ($n=2000$, average degree $15$) gives the same exponents, $\SparseSlopeUone$ and $\SparseSlopeRatioAone$ (\cref{app:e11}).

\paragraph{Sign cancellation.} Under the rank-two model \eqref{eq:rank2}, the triangle test never exceeds a rejection rate of \ElevenbTriRtwoMaxPower{} for any $t$ and any $m\le\ElevenbGridMax$. Signed $4$-cycles detect the difference with slope $\ElevenbSlopeCfourRtwo$ (theory $-4$), from \ElevenbCfourRtwoLo{} graphs at $t=0.3$ to \ElevenbCfourRtwoHi{} at $t=0.9$ (\cref{app:e11}).

\begin{figure}[t]
\centering
\includegraphics[width=\linewidth]{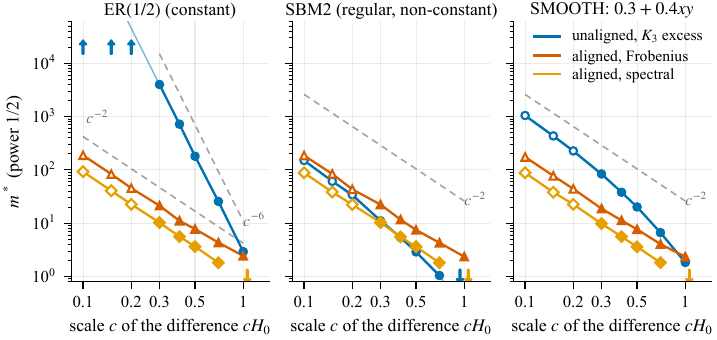}
\caption{The same degree-preserving difference $cH_0$ added to three nulls ($n=128$): Erd\H{o}s--R\'enyi, a regular two-block model (SBM2) and an irregular smooth graphon (SMOOTH). The plots show $m^\ast$ against the scale $c$ for the unaligned excess-triangle test and for two aligned tests, with thin curves for the prediction $2z^2\sigma_0^2/\Delta(c)^2$ from the exact shift. Under the constant null the sample size of the unaligned test grows like $c^{-6}$, while under the non-constant nulls it tracks the aligned tests, whose sample size grows like $c^{-2}$. Open markers extend the grid to small $c$, and arrows mark censored values.}
\label{fig:visibility}
\end{figure}

\paragraph{Order of visibility under non-constant nulls.} \Cref{fig:visibility} adds one degree-preserving difference $cH_0$ to three nulls. Its first-order coefficient $\kappa=3\langle H_0,W_{K_3}\rangle$ is zero under the Erd\H{o}s--R\'enyi null and $\TwelveKappaSbm$ and $\TwelveKappaSmooth$ under SBM2 and SMOOTH. At $c=0.3$ the unaligned excess-triangle test needs \TwelveUnEr{} graphs under the Erd\H{o}s--R\'enyi null, against \TwelveUnSbm{} under SBM2 and \TwelveUnSmooth{} under SMOOTH (\TwelveRatioSbm{} and \TwelveRatioSmooth{} times fewer). The aligned spectral test needs about $10$ graphs under each null.

Under the constant null the unaligned slope in $c$ is $\TwelveSlopeUnEr$ (theory $-6$), whereas under the non-constant nulls it approaches the first-order value $-2$ as $c\to0$, with local slopes \TwelveLocalSbm{} and \TwelveLocalSmooth{} on $c\in[0.1,0.15]$ (s.e.\ at most \TwelveLocalSbmSe{} and \TwelveLocalSmoothSe). Over $c\in[0.3,1]$ the slopes are steeper, \TwelveSlopeUnSbmSpec{} and \TwelveSlopeUnSmoothSpec, because higher-order terms of $t(K_3,W+cH_0)$ are still comparable to $\kappa c$ there. Every uncensored $m^\ast$ agrees with the prediction $2z^2\sigma_0^2/\Delta(c)^2$ within Monte Carlo error.

\begin{figure}[t]
\centering
\includegraphics[width=\linewidth]{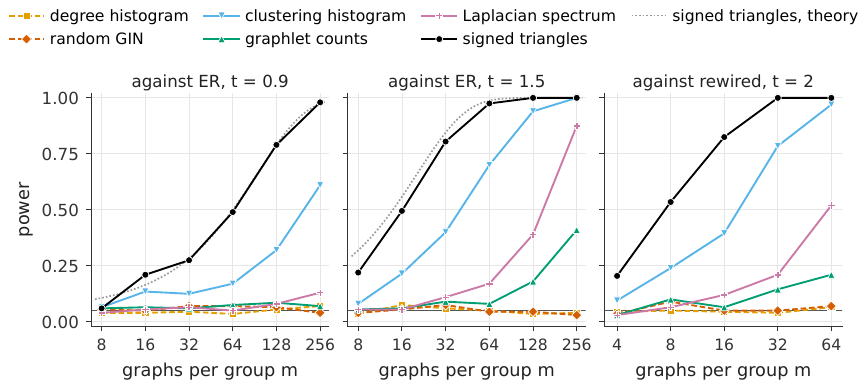}
\caption{Power of two-sample tests on descriptors used to evaluate graph generative models. Descriptors are compared by a Gaussian-kernel MMD calibrated by permutation, and signed triangles by a permutation test on the mean difference ($n=200$, $p=0.1$, $200$ replicates per point, horizontal line at level $0.05$). Left and middle: planted partition \eqref{eq:pp} against $\Gnp$. Right: planted-partition graphs against the same graphs after degree-preserving double-edge swaps. Tree-based descriptors (dashed) stay at the level, and the dotted curve is the normal-theory power of signed triangles.}
\label{fig:metrics}
\end{figure}

\paragraph{Graph evaluation metrics.} We test descriptors used to evaluate graph generative models \citep{you2018graphrnn,obray2022evaluation,thompson2022evaluation}, namely the degree histogram, the clustering histogram, six graphlet counts, the Laplacian spectrum and a randomly initialized graph isomorphism network (GIN) \citep{xu2019powerful}, each compared by a Gaussian-kernel MMD \citep{gretton2012kernel}, together with signed triangles (\cref{fig:metrics}).

Against the planted partition, the degree histogram and the random GIN reject at rates between \EfourteenBlindMin{} and \EfourteenBlindMax{} for every $m$ up to $256$, at both $t=0.9$ and $t=1.5$, so they stay at the level, as \cref{thm:trees} predicts. Both metrics have power \PosCtrlDeg{} with $m=16$ against a degree-heterogeneous graphon of the same edge density, where signed triangles have power \PosCtrlTri{}.

Signed triangles reach power $0.5$ at $m=$ \EfourteenMstarTriA{} for $t=0.9$ and $m=$ \EfourteenMstarTriB{} for $t=1.5$, against \EfourteenTheoryTriA{} and \EfourteenTheoryTriB{} predicted by the normal approximation with the moments of \cref{thm:price}(ii). The clustering histogram needs \EfourteenMstarClusA{} and \EfourteenMstarClusB{} graphs, and the graphlet vector needs \EfourteenMstarGraphletA{} because edge-density fluctuations, which move all counts together, dominate its kernel distance.

The right panel of \cref{fig:metrics} models a generator that reproduces degrees exactly but destroys cycle structure, with group 2 consisting of planted-partition graphs with $t=2$ after degree-preserving double-edge swaps. Signed triangles detect the change with $m=$ \EfourteenbMstarTri{} graphs, clustering with \EfourteenbMstarClus{} and the spectrum with \EfourteenbMstarSpec{}, but the degree and GIN metrics never exceed a rejection rate of \EfourteenbMaxGin{}.

\paragraph{Irregular nulls and protein graphs.} Under two irregular nulls, SMOOTH and a rank-one graphon, the random GIN detects a degree-preserving difference that is correlated with the degree profile, with power \EfifteenGinSmoothHA{} and \EfifteenGinDcHA{} at $m=64$, through the first-order term of the $3$-edge path (\cref{app:e15}, where the signal is above the one-graph threshold). It stays at the level, with rejection rates between \EfifteenGinOrthMin{} and \EfifteenGinOrthMax{} up to $m=256$ (Monte Carlo s.e.\ $0.015$), when the difference is orthogonal to the eigenfunctions of the null, which leaves every path density unchanged (\cref{prop:orthogonal}). On protein graphs from two generative-model benchmarks (\cref{app:e16}), rewiring $\EsixteenRho\%$ of the edges of each graph by degree-preserving swaps removes about a quarter of its triangles. The degree histogram and the random GIN never exceed a rejection rate of \EsixteenTreeMax{} for $m\le64$, whereas degree-corrected signed triangles reach power \EsixteenPowerTriDdFour{} with $m=4$ on D\&D and signed $4$-cycles need $m^\ast=\EsixteenMstarCfourEnz$ on ENZYMES. We chose their normalization in a pilot on D\&D (\cref{app:e16}).

\paragraph{Calibration with few graphs.} With one to three graphs per group (\cref{tab:calibration} in \cref{app:e13}), whole-graph permutation never rejects, the Welch normal approximation reaches size $0.21$ at $m=2$, and only the vertex-split test (\cref{prop:split}) has a finite-sample guarantee at $m=1$. Conservative variance estimates keep power for first-order visible differences (SBM2 null). Under the Erd\H{o}s--R\'enyi null, where the difference is third-order visible, the best method valid for all three nulls reaches power $0.15$, so with few graphs only first-order differences appear testable without strong assumptions.

\section{Related work}\label{sec:related}

With aligned vertices, minimax rates for two-sample testing of inhomogeneous random graphs are known \citep{ghoshdastidar2020two,chatterjee2023two}. Without alignment, tests compare network statistics of one graph per group \citep{ghoshdastidar2017two}, latent positions \citep{tang2017nonparametric}, subgraph counts and cumulants \citep{maugis2020testing,bravo2023quantifying,shao2026higher}, or estimated block models and graphons \citep{nguen2025network,sabanayagam2021graphon,qi2024multivariate}, and our excess triangle statistic is the $T_3$ of \citet{gao2017testing}. Subgraph counts are the classical moments of network models \citep{bickel2011method}, and populations of unlabeled networks have been studied through graph distances \citep{kolaczyk2020averages,lunagomez2021modeling}.

Misalignment has been studied for aligned tests with shuffled labels \citep{saxena2025lost} and for changepoints \citep{chen2026vertex}. The choice of cycle order for one graph is studied for signed polygons, latent geometry and planted subgraphs \citep{jin2021optimal,bubeck2016testing,bangachev2024detection,yu2025counting}, and low-degree analyses of invariant tests treat one graph \citep{kunisky2024tensor,carpentier2025low,rush2026count,skeja2026sharp}, whereas our lower bounds are information-theoretic and use independence across graphs. \Cref{app:related} discusses more work.

\section{Discussion and limitations}\label{sec:discussion}

Our results suggest that the number of unaligned graphs needed is set by the lowest order at which the difference is visible to invariant statistics. A motif test attains the aligned order of sample size for differences visible at first order, whereas degree-preserving differences from a constant null, first visible through cycles, cost a power of the signal-to-noise ratio and are invisible to tree statistics.

\paragraph{A recommendation for evaluation pipelines.} Metrics built from degrees or message passing should be complemented by cycle statistics. A cheap choice is the signed triangle and $4$-cycle counts of $B=A-dd^\top/(2|E|)$, the degree-corrected centring of \citet{jin2021optimal}, divided by the number of $2$-paths, which depends only on the degrees. They need one matrix product per graph, and permuting whole graphs makes the test exact. In our experiments they detected every degree-preserving difference under irregular nulls and the unrestricted rewiring of the protein graphs (\cref{app:e15,app:e16}).

\paragraph{Limitations.} The sharp rates of \cref{sec:price} are for an Erd\H{o}s--R\'enyi null and two explicit alternatives. For non-constant nulls we prove upper bounds only, since matching lower bounds would need a second-moment method for nulls that are mixtures over latent positions, which we leave open. \Cref{thm:trees} needs a regular null, and part (ii) covers finitely many tree statistics in the dense regime, so message-passing architectures at finite $n$ remain open. Partially aligned samples, such as networks registered to a common atlas, and tests that first estimate the partition of each graph when $t>1$ are outside our analysis. Apart from the protein graphs, whose alternative is a synthetic rewiring, our experiments are simulations, and the exact small-$m$ test gives up motifs that cross blocks.

\paragraph{Broader impact.} This methodological work mainly bears on the evaluation of generative models of networks, where metrics that ignore cycles can report a good fit for a model that is wrong in ways they cannot detect. We do not foresee negative societal consequences.

\bibliographystyle{plainnat}
{\small
\bibliography{references}
}

\newpage
\appendix

\section{Proofs}\label{app:proofs}

This appendix proves the results of \cref{sec:price,sec:trees,sec:firstorder,sec:calibration} in the order of the main text. Twelve numerical checks, numbered (i)--(xii), test identities and bounds below by exact enumeration on small graphs or by simulation. Their scripts and outputs are part of the code released with the paper.

\paragraph{Conventions.}
We write $[n]=\{1,\dots,n\}$, let $E_n$ be the set of unordered pairs $e=\{i,j\}\subset[n]$, and put $N=|E_n|=\binom n2$. A graph on $[n]$ is identified with its adjacency vector $A\in\{0,1\}^{E_n}$.

For $S\subseteq E_n$, $V(S)$ is the set of endpoints of edges in $S$, $\deg_S(i)$ the degree of $i$ in $S$, $\Odd(S)=\{i:\deg_S(i)\text{ odd}\}$, $c(S)$ the number of connected components of the graph $(V(S),S)$ and $\beta(S)=|S|-|V(S)|+c(S)$ its cyclomatic number. We call $S$ \emph{even} if $\Odd(S)=\emptyset$, and $\Evn$ denotes the set of even $S\subseteq E_n$, including $\emptyset$. For a finite graph $G$, $\Cyc(G)$ is its set of cycles, viewed as edge sets.

For $x\in\mathbb R^{E_n}$ we write $x^S=\prod_{e\in S}x_e$, with $x^\emptyset=1$. For $\sigma\in\{\pm1\}^n$ and $B\subseteq[n]$ we write $\sigma^B=\prod_{i\in B}\sigma_i$, so that
\begin{equation}\label{eq:oddset}
\prod_{\{i,j\}\in S}\sigma_i\sigma_j=\sigma^{\Odd(S)},\qquad \E\,\sigma^{B}=\ind[B=\emptyset]\quad\text{for uniform }\sigma .
\end{equation}
For a graph $F$ without isolated vertices, $\Cop_n(F)$ is the set of copies of $F$, that is, of the $S\subseteq E_n$ with $(V(S),S)\cong F$, and $|\Cop_n(F)|=\ff{n}{v(F)}/|\mathrm{Aut}(F)|$. Throughout, $p\in(0,1)$, $q=p(1-p)$, $\eta=\delta^2/q$ and $t=n\eta$. The null is $P_0=G(n,p)$. The planted partition law $P_\delta$ draws $\sigma$ uniformly from $\{\pm1\}^n$ and then, given $\sigma$, independent $A_{ij}\sim\Bern(p+\delta\sigma_i\sigma_j)$, where $0<\delta\le\min(p,1-p)$. Tests may be randomized, so a test $\psi$ takes values in $[0,1]$.

\subsection{Preliminaries}\label{app:prelim}

\begin{lemma}[Tensorization, two-sample reduction and Le Cam]\label{lem:tensor}
Let $P,Q,R,P_i,Q_i$ be probability measures on finite sets with $P\ll Q$, $P_i\ll Q_i$.
\begin{itemize}
\item[(a)] $1+\chi^2\big(\bigotimes_{i=1}^m P_i\,\|\,\bigotimes_{i=1}^mQ_i\big)=\prod_{i=1}^m\big(1+\chi^2(P_i\|Q_i)\big)$.
\item[(b)] $\TV(P\otimes R,Q\otimes R)=\TV(P,Q)$.
\item[(c)] For every test $\psi$, $\E_P\psi+\E_Q(1-\psi)\ge1-\TV(P,Q)$, and $\TV(P,Q)\le\frac12\chi^2(P\|Q)^{1/2}$.
\item[(d)] Let $\mathcal H_0$ be a set of pairs of laws containing $(P_0,P_0)$ and $\mathcal H_1$ a set containing $(P_1,P_0)$. Then every test $\psi$ of $(G_1,\dots,G_m,G'_1,\dots,G'_m)$, where $G_\ell\sim P$, $G'_\ell\sim Q$ are independent, satisfies
\[
\begin{aligned}
\sup_{(P,Q)\in\mathcal H_0}\E\psi+\sup_{(P,Q)\in\mathcal H_1}\E(1-\psi)\;&\ge\;1-\TV\big(P_1^{\otimes m},P_0^{\otimes m}\big)\\
&\ge\;1-\tfrac12\Big(\big(1+\chi^2(P_1\|P_0)\big)^m-1\Big)^{1/2}.
\end{aligned}
\]
\end{itemize}
\end{lemma}

\begin{proof}
(a) The likelihood ratio of the products is $\prod_iL_i$ with $L_i=dP_i/dQ_i$, and under $\bigotimes Q_i$ the factors are independent, so $\E\prod_iL_i^2=\prod_i\E_{Q_i}L_i^2=\prod_i(1+\chi^2(P_i\|Q_i))$.
(b) $\TV(P\otimes R,Q\otimes R)=\frac12\sum_{x,y}|P(x)-Q(x)|R(y)=\TV(P,Q)$.
(c) $\E_P\psi+\E_Q(1-\psi)=1-(\E_Q\psi-\E_P\psi)\ge1-\TV(P,Q)$ because $0\le\psi\le1$. By Cauchy--Schwarz, $2\TV(P,Q)=\E_Q|dP/dQ-1|\le(\E_Q(dP/dQ-1)^2)^{1/2}=\chi^2(P\|Q)^{1/2}$.
(d) The left side is at least $\E_{P_0^{\otimes 2m}}\psi+\E_{P_1^{\otimes m}\otimes P_0^{\otimes m}}(1-\psi)$, which by (c) and (b) is at least $1-\TV(P_1^{\otimes m}\otimes P_0^{\otimes m},P_0^{\otimes m}\otimes P_0^{\otimes m})=1-\TV(P_1^{\otimes m},P_0^{\otimes m})$. Finally apply (c) and (a).
\end{proof}

\begin{remark}
Lemma~\ref{lem:tensor}(d) shows that a two-sample test is never better than a one-sample test that knows the null law. The lower bounds below therefore hold for the one-sample problem, for tests that know $(p,\delta)$, and for tests that use the vertex labels.
\end{remark}

\begin{lemma}[Per-edge second moment]\label{lem:secondmoment}
Let $\Lambda$ be a random element with law $\pi$ and $h:\operatorname{supp}\pi\to\mathbb R^{E_n}$ with $p+h_e(\lambda)\in[0,1]$. Let $P_h$ be the law of $A$ when, given $\Lambda=\lambda$, the $A_e$ are independent $\Bern(p+h_e(\lambda))$. Put $X_e=(A_e-p)/q$ and $L_\lambda(A)=\prod_{e\in E_n}(1+h_e(\lambda)X_e)$. Then $P_h\ll P_0$, $dP_h/dP_0=\E_\Lambda L_\Lambda$, and for all $\lambda,\lambda'$,
\[
\E_0\big[L_\lambda L_{\lambda'}\big]=\prod_{e\in E_n}\Big(1+\frac{h_e(\lambda)h_e(\lambda')}{q}\Big),\qquad
\E_0\Big(\frac{dP_h}{dP_0}\Big)^2=\E\prod_{e\in E_n}\Big(1+\frac{h_e(\Lambda)h_e(\Lambda')}{q}\Big),
\]
where $\Lambda'$ is an independent copy of $\Lambda$.
\end{lemma}

\begin{proof}
Let $a\in\{0,1\}$ and $x=(a-p)/q$. If $a=1$, then $(p+h)/p=1+h(1-p)/q=1+hx$, and if $a=0$, then $(1-p-h)/(1-p)=1-hp/q=1+hx$. Multiplying over edges gives $dP_{h}(\cdot\mid\Lambda=\lambda)/dP_0=L_\lambda$ and, averaging over $\lambda$, $dP_h/dP_0=\E_\Lambda L_\Lambda$. Under $P_0$ the $X_e$ are independent with $\E_0X_e=0$ and $\E_0X_e^2=q/q^2=1/q$, so $\E_0[(1+h_eX_e)(1+h'_eX_e)]=1+h_eh'_e/q$ and the product formula follows. The last identity follows from $\E_0(\E_\Lambda L_\Lambda)^2=\E_{\Lambda,\Lambda'}\E_0[L_\Lambda L_{\Lambda'}]$ by Fubini's theorem, since all quantities are finite sums.
\end{proof}

\begin{lemma}[Even-subgraph expansion]\label{lem:evenexpansion}
Let $G=(V,E)$ be a finite graph, $\omega$ uniform on $\{\pm1\}^V$ and $(x_e)_{e\in E}$ real. Then
\[
\E_\omega\prod_{\{i,j\}\in E}\big(1+x_{ij}\,\omega_i\omega_j\big)=\sum_{S\subseteq E,\ S\text{ even}}x^S .
\]
In particular $\E_\omega\prod_{i<j}(1+x\,\omega_i\omega_j)=\sum_{S\in\Evn}x^{|S|}$.
\end{lemma}

\begin{proof}
Expanding the product, $\prod_{e\in E}(1+x_e\omega_i\omega_j)=\sum_{S\subseteq E}x^S\prod_{\{i,j\}\in S}\omega_i\omega_j=\sum_{S\subseteq E}x^S\omega^{\Odd(S)}$, and $\E\omega^{\Odd(S)}=\ind[S\text{ even}]$ by \eqref{eq:oddset}.
\end{proof}

\begin{lemma}[Veblen-type bound]\label{lem:veblen}
Let $G=(V,E)$ be a finite graph and $x_e\ge0$. Then
\[
\sum_{S\subseteq E,\ S\text{ even}}x^S\;\le\;\prod_{C\in\Cyc(G)}\big(1+x^C\big)\;\le\;\exp\Big(\sum_{C\in\Cyc(G)}x^C\Big).
\]
For $G=K_n$ and $x_e\equiv x=t/n\ge0$ this reads $\sum_{S\in\Evn}x^{|S|}\le\exp(g_n(t))$ with $g_n(t)=\sum_{k=3}^n\frac{\ff nk}{2k\,n^k}t^k$.
\end{lemma}

\begin{proof}
Every even edge set is a union of pairwise edge-disjoint cycles (Veblen's theorem). To see this, let $S\ne\emptyset$ be even. The graph $(V(S),S)$ has minimum degree at least two and so contains a cycle $C$, and $S\setminus C$ is again even, so induction on $|S|$ applies. Fix for each even $S$ one such decomposition $D(S)\subseteq\Cyc(G)$. The cycles in $D(S)$ are distinct and pairwise edge-disjoint, so $x^S=\prod_{C\in D(S)}x^C$, and $S\mapsto D(S)$ is injective because $S=\bigcup_{C\in D(S)}C$. Since all terms are non-negative,
$\sum_{S\text{ even}}x^S=\sum_{S}\prod_{C\in D(S)}x^C\le\sum_{\mathcal D\subseteq\Cyc(G)}\prod_{C\in\mathcal D}x^C=\prod_{C\in\Cyc(G)}(1+x^C)$, and $1+y\le e^y$. In $K_n$ there are $\ff nk/(2k)$ cycles of length $k$ (sequences of $k$ distinct vertices up to the $k$ rotations and $2$ reflections), which gives $\sum_{C}x^{|C|}=\sum_{k=3}^n\frac{\ff nk}{2k}\frac{t^k}{n^k}=g_n(t)$.
\end{proof}

\begin{lemma}[The function $g$]\label{lem:gfunction}
For $0\le t<1$ and $n\ge3$, $g_n(t)\le g(t):=\sum_{k\ge3}\frac{t^k}{2k}=-\frac12\log(1-t)-\frac t2-\frac{t^2}4\le\frac{t^3}{6(1-t)}$.
\end{lemma}

\begin{proof}
The bound $\ff nk\le n^k$ gives $g_n\le g$, the closed form follows from $-\log(1-t)=\sum_{k\ge1}t^k/k$, and the last inequality from $\sum_{k\ge3}t^k/(2k)\le\frac16\sum_{k\ge3}t^k$.
\end{proof}

\begin{lemma}[Curie--Weiss form and asymptotic exactness of the bound]\label{lem:curieweiss}
For $0\le x<1$ let $b=\artanh x$ and $M=\sum_{i=1}^n\omega_i$ with $\omega$ uniform on $\{\pm1\}^n$. Then
\[
\sum_{S\in\Evn}x^{|S|}=\cosh(b)^{-N}e^{-bn/2}\,\E\,e^{bM^2/2}.
\]
If $n\to\infty$ and $nx\to s\in[0,1)$, then $\sum_{S\in\Evn}x^{|S|}\to e^{g(s)}=e^{-s/2-s^2/4}(1-s)^{-1/2}$.
\end{lemma}

\begin{proof}
For $y\in\{\pm1\}$, $1+xy=(\cosh b+y\sinh b)/\cosh b=e^{by}/\cosh b$. Hence $\prod_{i<j}(1+x\omega_i\omega_j)=\cosh(b)^{-N}\exp(b\sum_{i<j}\omega_i\omega_j)$ and $\sum_{i<j}\omega_i\omega_j=(M^2-n)/2$, so Lemma~\ref{lem:evenexpansion} gives the identity. As $nx\to s$, $b=x+O(x^3)$, so $bn\to s$, $bn/2\to s/2$ and $N\log\cosh b=Nb^2/2+O(Nb^4)\to s^2/4$. Finally $\E\exp(bM^2/2)=\E\exp\big(\tfrac{bn}{2}(M/\sqrt n)^2\big)\to(1-s)^{-1/2}$, because $M/\sqrt n\to N(0,1)$ in law, $\E e^{sZ^2/2}=(1-s)^{-1/2}$, and the family is uniformly integrable. For the last point take $s<s'<1$. Then $bn\le s'$ eventually, and Hoeffding's inequality gives $\Prob(e^{s'M^2/(2n)}\ge K)\le2K^{-1/s'}$ with $1/s'>1$ \citep[arXiv version, Lemma~5.5]{mossel2015reconstruction}.
\end{proof}

\begin{lemma}[Moments of signed sums under latent-variable models]\label{lem:latentvar}
In the setting of Lemma~\ref{lem:secondmoment}, write $\mu_e=h_e(\Lambda)$, $\xi_e=A_e-p-\mu_e$ and $v_e=(p+\mu_e)(1-p-\mu_e)$. Given $\Lambda$, the $\xi_e$ are independent, centred, $|\xi_e|\le1$ and $\E[\xi_e^2\mid\Lambda]=v_e$. Let $\mathcal S$ be a finite set of edge sets and $Y=\sum_{\varphi\in\mathcal S}(A-p)^{\varphi}$. For $R\subseteq E_n$ put $c_R=\sum_{\varphi\in\mathcal S,\ \varphi\supseteq R}\mu^{\varphi\setminus R}$. Then
\[
Y=\sum_{R}c_R\,\xi^R,\qquad \E[Y\mid\Lambda]=c_\emptyset=\sum_{\varphi\in\mathcal S}\mu^\varphi,\qquad \Var(Y\mid\Lambda)=\sum_{R\ne\emptyset}c_R^2\,v^R,
\]
and $\Var Y=\Var\big(\sum_{\varphi\in\mathcal S}\mu^\varphi\big)+\sum_{R\ne\emptyset}\E\big[c_R^2v^R\big]$.
\end{lemma}

\begin{proof}
$(A-p)^\varphi=\prod_{e\in\varphi}(\mu_e+\xi_e)=\sum_{R\subseteq\varphi}\mu^{\varphi\setminus R}\xi^R$, and collecting terms by $R$ gives $Y=\sum_Rc_R\xi^R$. Given $\Lambda$, $\E[\xi^R\xi^{R'}\mid\Lambda]=\ind[R=R']\,v^R$ by independence and centring, which gives the conditional mean and variance. The last identity is the law of total variance.
\end{proof}

\subsection{Theorem~\ref{thm:price}: sample complexity under a constant null}\label{app:price}

\begin{theorem*}[Restatement of Theorem~\ref{thm:price}]
Let $n\ge3$, $m\ge1$, $p\in(0,1)$, $0<\delta\le\min(p,1-p)$, $L=dP_\delta/dP_0$, $c_{\alpha\beta}=\log(1+4(1-\alpha-\beta)^2)$ for $\alpha+\beta<1$.
\begin{itemize}
\item[(i)] $1+\chi^2(P_\delta^{\otimes m}\|P_0^{\otimes m})=(\E_0L^2)^m$ with $\E_0L^2=\sum_{S\in\Evn}\eta^{|S|}\le e^{g_n(t)}$. For $t<1$, $g_n(t)\le g(t)\le t^3/(6(1-t))$ and $\E_0L^2\to e^{g(t)}$ as $n\to\infty$. Every test of $H_0:P=Q$ against $H_1\ni(P_\delta,P_0)$ with $m$ graphs per group has $\sup_{H_0}\E\psi+\sup_{H_1}\E(1-\psi)\ge1-\frac12(e^{mg_n(t)}-1)^{1/2}$; errors $(\alpha,\beta)$ require $m\ge c_{\alpha\beta}/g_n(t)$, and for $t<1$, $m\ge c_{\alpha\beta}/g(t)\ge6(1-t)c_{\alpha\beta}/t^3$.
\item[(ii)] The moments of $T_p$ are as in Lemma~\ref{lem:trianglemoments}. The test $\psi_3=\ind\{|D_3|\ge\lambda_3\}$, $D_3=\sum_{\ell\le m}T_p(G_\ell)-\sum_{\ell\le m}T_p(G'_\ell)$, $\lambda_3=(2m\binom n3q^3/\alpha)^{1/2}$, has size $\le\alpha$ under $P=Q=P_0$ and type II error $\le\beta$ under $(P_\delta,P_0)$ whenever $m\,s_n^2\ge\big(\sqrt{2/\alpha}+\sqrt{(2+\eps_n)/\beta}\big)^2$, where $s_n^2=\binom n3\eta^3=t^3\frac{(n-1)(n-2)}{6n^2}$ and $\eps_n$ is as in Lemma~\ref{lem:trianglemoments}.
\item[(iii)] $1+\chi^2(\Pal\|P_0^{\otimes m})=A_m^N\sum_{S\in\Evn}(B_m/A_m)^{|S|}\le\exp\big(\binom m2N\eta^2+g_n(u_m)\big)$, $\binom m2N\eta^2=\frac{m(m-1)(n-1)}{4n}t^2$, $u_m=nB_m/A_m=n\tanh(m\artanh\eta)\le\kappa_\eta mt$ with $\kappa_\eta=\artanh(\eta)/\eta\le(1-\eta^2)^{-1}$. If $\kappa_\eta mt<x_{\alpha\beta}$, where $x_{\alpha\beta}\in(0,1)$ solves $x^2/4+g(x)=c_{\alpha\beta}$, no test has errors $(\alpha,\beta)$. The test $\psi_Q=\ind\{|Q|\ge(\binom m2Nq^2/\alpha)^{1/2}\}$ ($m\ge2$) has size $\le\alpha$ under $P_0^{\otimes m}$ and type II error $\le\beta$ under $\Pal$ whenever $\binom m2N\eta^2\ge(\alpha^{-1/2}+(\rho_m/\beta)^{1/2})^2$, $\rho_m=1+(2(m-1)+(1-2p)^2/q)\,t/n$.
\item[(iv)] Let $K'=\alpha^{-1/2}+(2/\beta)^{1/2}$, $K_{\mathrm{un}}=7(\sqrt{2/\alpha}+\sqrt{3/\beta})^2$ and $n_0=\max\{21,\lceil1/q\rceil,\lceil2+6K'+1/(2q)\rceil\}$. For $n\ge n_0$ and $0<t\le1/2$,
\[
\frac{3c_{\alpha\beta}}{t^3}\le m^*_{\mathrm{un}}\le1+\frac{K_{\mathrm{un}}}{t^3},\qquad \frac{3x_{\alpha\beta}}{4t}\le m^*_{\mathrm{al}}\le2+\frac{3K'}{t},
\]
\[
\frac{3c_{\alpha\beta}}{(1+3K')\,t^2}\le\frac{m^*_{\mathrm{un}}}{m^*_{\mathrm{al}}}\le\frac{4(1/8+K_{\mathrm{un}})}{3x_{\alpha\beta}\,t^2}.
\]
\end{itemize}
\end{theorem*}

\paragraph{Proof of (i).}
We compute $\E_0L^2$ exactly, bound it by a product over cycles, and tensorize over the $m$ graphs. By Lemma~\ref{lem:secondmoment} with $\Lambda=\sigma$ and $h_{ij}=\delta\sigma_i\sigma_j$, $\E_0[L_\sigma L_\tau]=\prod_{i<j}(1+\eta\,\sigma_i\sigma_j\tau_i\tau_j)$ and $\E_0L^2=\E_{\sigma,\tau}\prod_{i<j}(1+\eta\,\omega_i\omega_j)$ with $\omega=\sigma\circ\tau$. For independent uniform $\sigma,\tau$, $\omega$ is uniform on $\{\pm1\}^n$, so Lemma~\ref{lem:evenexpansion} gives $\E_0L^2=\sum_{S\in\Evn}\eta^{|S|}$, and Lemma~\ref{lem:veblen} with $x=\eta=t/n$ gives $\E_0L^2\le e^{g_n(t)}$. Lemma~\ref{lem:tensor}(a) gives $1+\chi^2(P_\delta^{\otimes m}\|P_0^{\otimes m})=(\E_0L^2)^m\le e^{mg_n(t)}$. Lemma~\ref{lem:gfunction} gives the comparison with $g$, and Lemma~\ref{lem:curieweiss} (with $x=\eta$, $n\eta=t$) the limit $e^{g(t)}$. The error bound is Lemma~\ref{lem:tensor}(d) with $P_1=P_\delta$. If a test has errors $(\alpha,\beta)$, then $\alpha+\beta\ge1-\frac12(e^{mg_n(t)}-1)^{1/2}$, which rearranges to $e^{mg_n(t)}\ge1+4(1-\alpha-\beta)^2$, that is, $m\ge c_{\alpha\beta}/g_n(t)$. \qed

\begin{remark}[Exact evaluation]
Lemma~\ref{lem:curieweiss} gives $\E_0L^2$ exactly as a sum of $n+1$ terms. For the E11 settings ($n=300$, $t\in\{0.3,\dots,0.9\}$) the exact $\log\E_0L^2$, $g_n(t)$ and $g(t)$ agree to within $10\%$ (numerical check (ii)). The lower-bound curve $m_{\mathrm{LB}}(t)=c_{\alpha\beta}/g(t)$ with $\alpha=0.05$, $\beta=0.5$ is $101.6,\,17.4,\,4.58$ at $t=0.3,\,0.5,\,0.7$.
\end{remark}

\begin{lemma}[Signed triangles and signed cherries under the planted partition]\label{lem:trianglemoments}
Let $a=q-\delta^2$ and $b=\delta(1-2p)$. Then $\E_0T_p=0$, $\Var_0T_p=\binom n3q^3$, $\E[T_p\mid\sigma]=\binom n3\delta^3$ for every $\sigma$, and
\[
\Var_\delta T_p=\tbinom n3\big(a^3+b^3+3\delta^2a^2\big)+N(n-2)^2\delta^4a\;\le\;\tbinom n3q^3\,(1+\eps_n),
\]
\[
\eps_n:=|1-2p|^3\Big(\frac\eta q\Big)^{3/2}+3\eta+3(n-2)\eta^2\;\le\;|1-2p|^3\Big(\frac{t}{nq}\Big)^{3/2}+\frac{3t(1+t)}{n}.
\]
For the signed cherry count $S_2=\sum_{j}\sum_{\{i,k\}\not\ni j}(A_{ij}-p)(A_{jk}-p)$ (sum over paths $i-j-k$): $\E_0S_2=\E_\delta S_2=0$, $\E_0S_2^2=3\binom n3q^2$ and $\E_\delta S_2^2=N(n-2)\big[(n-2)\delta^4+a^2+2\delta^2(a+b)\big]$.
\end{lemma}

\begin{proof}
We expand $T_p$ and $S_2$ in the orthogonal basis of Lemma~\ref{lem:latentvar} and add up the squared coefficients.

\emph{Signed triangles.} Apply Lemma~\ref{lem:latentvar} with $\Lambda=\sigma$, $\mu_{ij}=\delta\sigma_i\sigma_j$ and $\mathcal S$ the set of triangles. Since $(p+x)(1-p-x)=q+x(1-2p)-x^2$ and $\mu_e^2=\delta^2$, we have $v_{ij}=a+b\,\sigma_i\sigma_j$. By \eqref{eq:oddset}, $\mu^\Delta=\delta^3\sigma^{\Odd(\Delta)}=\delta^3$ for every triangle $\Delta$, so $c_\emptyset=\binom n3\delta^3$ is non-random. The non-empty $R$ contained in a triangle are of three kinds:
\begin{itemize}
\item the triangle itself, with $c_R=1$;
\item a path $\{ij,jk\}$, contained only in the triangle $ijk$, with $c_R=\mu_{ik}$ and $c_R^2=\delta^2$;
\item a single edge $ij$, contained in the $n-2$ triangles $ijk$, with $c_R=\sum_{k\ne i,j}\mu_{ik}\mu_{jk}=(n-2)\delta^2\sigma_i\sigma_j$ and $c_R^2=(n-2)^2\delta^4$.
\end{itemize}
By \eqref{eq:oddset}, $\E\prod_{e\in R}(a+b\sigma_i\sigma_j)=\sum_{S\subseteq R}a^{|R|-|S|}b^{|S|}\ind[S\text{ even}]$, which equals $a^3+b^3$ for a triangle, $a^2$ for a path and $a$ for an edge. There are $\binom n3$ triangles, $3\binom n3$ paths and $N$ edges, which gives the formula for $\Var_\delta T_p$. Setting $\delta=0$ gives $\Var_0T_p=\binom n3q^3$ and $\E_0T_p=0$. For the bound we use $0<a\le q$, $|b|^3=|1-2p|^3(\eta q)^{3/2}$, $3\delta^2a^2\le3\eta q^3$ and $N(n-2)^2\delta^4a\le\binom n3q^3\cdot3(n-2)\eta^2$, which holds because $N(n-2)^2=3(n-2)\binom n3$. Finally, $3(n-2)\eta^2\le3t^2/n$.

\emph{Signed cherries.} For $S_2$ take $\mathcal S$ the set of paths of length two. Now $\mu^\varphi=\delta^2\sigma_i\sigma_k$ for the path $i-j-k$, so $c_\emptyset=\delta^2(n-2)\sum_{i<k}\sigma_i\sigma_k$, with mean $0$ and second moment $\delta^4(n-2)^2N$. The non-empty $R$ are the paths themselves, with $c_R=1$ and $\E v^R=a^2$, and single edges $ij$. An edge $ij$ lies in the paths $ij,jk$ and $ij,ik$ ($k\ne i,j$), so $c_R=\delta(\sigma_i+\sigma_j)M'$ with $M'=\sum_{k\ne i,j}\sigma_k$. As $M'$ is independent of $(\sigma_i,\sigma_j)$, $\E M'^2=n-2$ and $\E[(\sigma_i+\sigma_j)^2(a+b\sigma_i\sigma_j)]=2(a+b)$, we get $\E[c_R^2v_{ij}]=2\delta^2(a+b)(n-2)$. Summing over the $3\binom n3=N(n-2)$ paths and the $N$ edges gives the claim, and $\delta=0$ gives $\E_0S_2^2$.
\end{proof}

\paragraph{Proof of (ii).}
Write $V_0=\binom n3q^3$ and $V_\delta=\Var_\delta T_p$. Under $P=Q=P_0$, $\E D_3=0$ and $\Var D_3=2mV_0$, so Chebyshev's inequality gives $\Prob(|D_3|\ge\lambda_3)\le2mV_0/\lambda_3^2=\alpha$. Under $(P_\delta,P_0)$, $\E D_3=m\binom n3\delta^3$ and $\Var D_3=m(V_\delta+V_0)$. If $\E D_3\ge\lambda_3+(m(V_\delta+V_0)/\beta)^{1/2}$, then $\Prob(|D_3|<\lambda_3)\le\Prob(|D_3-\E D_3|>\E D_3-\lambda_3)\le\beta$. Dividing by $(mV_0)^{1/2}$ and using $\binom n3\delta^3/V_0^{1/2}=s_n$ and $V_\delta/V_0\le1+\eps_n$ (Lemma~\ref{lem:trianglemoments}), this holds when $m^{1/2}s_n\ge\sqrt{2/\alpha}+\sqrt{(2+\eps_n)/\beta}$. \qed

\begin{remark}[Size of the constants]
$s_n^2=\binom n3\eta^3$ equals the $k=3$ term of $g_n(t)$ exactly, so the signed triangle count captures the leading term of the one-graph $\chi^2$ divergence. With a normal approximation (justified in Theorem~\ref{thm:trees}(ii), Remark~\ref{rem:trianglenormal}) the one-sided two-sample test has power $1-\Phi(z_\alpha-(ms_n^2/2)^{1/2})$, which reaches $1/2$ at $m\approx 2z_\alpha^2/s_n^2\approx12z_\alpha^2/t^3=32.5/t^3$ for $\alpha=0.05$. A cruder bound, $\binom n3(q+\delta)^3+N(n-2)(n-3)(q+\delta)\delta^4$, is also valid, since it exceeds the exact value by at least $N(n-2)q\delta(q-\delta^3)\ge0$.
\end{remark}

\begin{proposition}[Unknown $p$: the pooled plug-in statistic]\label{prop:plugin}
Let $\hat p$ be the edge density of the pooled sample of $2m$ graphs, $\hat D_3=\sum_\ell T_{\hat p}(G_\ell)-\sum_\ell T_{\hat p}(G'_\ell)$, and $\eps=\hat p-p$.
\begin{itemize}
\item[(a)] Under every null $P=Q$ (any law of random graphs), $\hat D_3$ has the same law as $-\hat D_3$; in particular $\E\hat D_3=0$.
\item[(b)] $T_{p+\eps}=T_p-\eps S_2+\eps^2(n-2)\sum_e(A_e-p)-\eps^3\binom n3$ identically, hence $\hat D_3-D_3=-\eps\,\Delta S_2+\eps^2(n-2)\,\Delta E$, where $\Delta S_2$ and $\Delta E$ are the between-group differences of $\sum_\ell S_2$ and $\sum_\ell\sum_e(A_e-p)$.
\item[(c)] Under $P=Q=P_0$ and under $(P_\delta,P_0)$: $\E\eps=0$, $\E\eps^2=q/(2mN)$, $\E(\Delta E)^2=2mNq$ and $\E(\Delta S_2)^2\le2m\Sigma_2$ with $\Sigma_2=\max(\E_0S_2^2,\E_\delta S_2^2)$. Consequently, for every $\gamma\in(0,1)$, with probability at least $1-\gamma$,
\[
|\hat D_3-D_3|\le\frac3\gamma\Big(\frac{q\,\Sigma_2}{N}\Big)^{1/2}+\frac{3\sqrt6}{2}\,\frac{(n-2)\,q^{3/2}}{\gamma^{3/2}(mN)^{1/2}} .
\]
\end{itemize}
In particular, if $nq\to\infty$ and $t$ is bounded, $|\hat D_3-D_3|=O_P(n^{1/2}q^{3/2})=o_P\big((mV_0)^{1/2}\big)$, and the test $\ind\{|\hat D_3|\ge(2m\binom n3\hat q^3/\alpha)^{1/2}\}$ with $\hat q=\hat p(1-\hat p)$ has asymptotic size at most $\alpha$ and the power guarantee of Theorem~\ref{thm:price}(ii) up to $o(1)$.
\end{proposition}

\begin{proof}
(a) Under $P=Q$ the $2m$ graphs are i.i.d. Exchanging the two groups leaves $\hat p$ unchanged and maps $\hat D_3$ to $-\hat D_3$.

(b) For a triangle with centred entries $x_1,x_2,x_3$, $\prod_i(x_i-\eps)=x_1x_2x_3-\eps(x_1x_2+x_1x_3+x_2x_3)+\eps^2(x_1+x_2+x_3)-\eps^3$. Summing over triangles, every path of length two lies in exactly one triangle and every edge in $n-2$ triangles. The $\eps^3$ terms cancel between the groups because both have $m$ graphs.

(c) We bound the random quantities $\eps$, $\Delta S_2$ and $\Delta E$ of (b) and then compare the error with the threshold $\lambda_3$.

\emph{Moments.} By Theorem~\ref{thm:trees}(i), applied to the planted partition as $G(n,W+H)$ with $W\equiv p$ and $H=\delta\,s\otimes s$ of zero degree, the edge count of one graph has mean $Np$ and variance $Nq$ under both $P_0$ and $P_\delta$. The $2m$ graphs are independent, which gives the moments of $\eps$ and $\Delta E$. Since $\E_0S_2=\E_\delta S_2=0$, $\E(\Delta S_2)^2$ is the sum of the $2m$ second moments.

\emph{Concentration.} Chebyshev's inequality with thresholds $a=(3q/(2mN\gamma))^{1/2}$, $b=(6m\Sigma_2/\gamma)^{1/2}$, $c=(6mNq/\gamma)^{1/2}$ for $|\eps|,|\Delta S_2|,|\Delta E|$ shows that all three bounds hold with probability at least $1-\gamma$. On this event $|\hat D_3-D_3|\le ab+(n-2)a^2c$, which is the display.

\emph{Order of the error.} By Lemma~\ref{lem:trianglemoments}, $\Sigma_2\le N(n-2)q^2\big(1+(n-2)\eta^2+2\eta(1+\delta/q)\big)$, so the first term is $O(n^{1/2}q^{3/2})$ when $t$ is bounded, and the second is $O(q^{3/2}m^{-1/2})$. Since $\lambda_3=(2mV_0/\alpha)^{1/2}\asymp m^{1/2}n^{3/2}q^{3/2}$, we get $|\hat D_3-D_3|/\lambda_3\to0$ in probability.

\emph{Size and power.} When $nq\to\infty$, $\hat q/q\to1$ in probability because $\E\eps^2=q/(2mN)$, so $\hat\lambda_3:=(2m\binom n3\hat q^3/\alpha)^{1/2}$ satisfies $\hat\lambda_3/\lambda_3\to1$. Fix $\zeta\in(0,1)$. On the event $\{|\hat D_3|\ge\hat\lambda_3\}\cap\{\hat\lambda_3>(1-\zeta/2)\lambda_3\}\cap\{|\hat D_3-D_3|<\zeta\lambda_3/2\}$ we have $|D_3|>(1-\zeta)\lambda_3$. Hence under $H_0$, $\Prob(|\hat D_3|\ge\hat\lambda_3)\le\alpha(1-\zeta)^{-2}+o(1)$ by Chebyshev, and letting $\zeta\downarrow0$ gives asymptotic size at most $\alpha$. The power statement follows in the same way from the proof of (ii), with $\lambda_3$ replaced by $(1+\zeta)\lambda_3$.
\end{proof}

\begin{remark}[Exact validity for the composite Erd\H{o}s--R\'enyi null]
Under $H_0:P=Q=G(n,p)$ with $p$ unknown, the total number of edges $E_{\mathrm{tot}}$ of the $2m$ graphs is sufficient, and conditionally on $E_{\mathrm{tot}}=k$ the $2m$ adjacency vectors are uniform over the $\binom{2mN}{k}$ configurations. A Monte Carlo test that compares $\hat D_3$ with its law under uniform redistribution of the $k$ edges is therefore exactly valid for every $p$. We do not use this test elsewhere.
\end{remark}

\begin{lemma}[Aligned second moment]\label{lem:aligned}
Let $\Pal$ be the law of $(G_1,\dots,G_m)$ when one $\sigma$ is drawn and the $G_\ell$ are i.i.d.\ planted-partition graphs given $\sigma$. Then, with $A_m,B_m$ as in Theorem~\ref{thm:price}(iii),
\[
1+\chi^2\big(\Pal\,\|\,P_0^{\otimes m}\big)=\E_\omega\prod_{i<j}(1+\eta\,\omega_i\omega_j)^m=A_m^N\sum_{S\in\Evn}\Big(\frac{B_m}{A_m}\Big)^{|S|}.
\]
Moreover $1\le A_m\le\exp\big(\binom m2\eta^2\big)$ and $B_m/A_m=\tanh(m\artanh\eta)\le m\artanh\eta$.
\end{lemma}

\begin{proof}
The likelihood ratio is $L^{\mathrm{al}}=\E_\sigma\prod_{\ell=1}^mL_\sigma(G_\ell)$ with $L_\sigma$ as in Lemma~\ref{lem:secondmoment}, so by independence of the $G_\ell$ under $P_0^{\otimes m}$ and Lemma~\ref{lem:secondmoment},
$\E_0(L^{\mathrm{al}})^2=\E_{\sigma,\tau}\prod_\ell\E_0[L_\sigma(G_\ell)L_\tau(G_\ell)]=\E_{\omega}\prod_{i<j}(1+\eta\omega_i\omega_j)^m$ with $\omega=\sigma\circ\tau$ uniform. For $y\in\{\pm1\}$, $(1+\eta y)^m=A_m+B_my=A_m(1+(B_m/A_m)y)$, and Lemma~\ref{lem:evenexpansion} with $x=B_m/A_m$ gives the identity.

For the bounds, $A_m=\sum_{j\ge0}\binom m{2j}\eta^{2j}$, and $\binom m{2j}\le\binom m2^j/j!$ because $\binom m{2j}(2j-1)!!$ is the number of sets of $j$ pairwise disjoint pairs in $[m]$, which is at most $\binom{\binom m2}{j}\le\binom m2^j/j!$. Hence $A_m\le\exp(\binom m2\eta^2)$. With $\eta=\tanh a$, $(1\pm\eta)^m=(1-\eta^2)^{m/2}e^{\pm ma}$, so $B_m/A_m=\tanh(ma)\le ma$.
\end{proof}

\begin{lemma}[Moments of the aligned statistic]\label{lem:Qmoments}
For $m\ge2$, $\E_0Q=0$, $\Var_0Q=\binom m2Nq^2$, $\E_{\mathrm{al}}Q=\binom m2N\delta^2$, and $\Var_{\mathrm{al}}Q=N\big[m(m-1)^2\delta^2a+\binom m2(a^2+b^2)\big]$ with $a,b$ as in Lemma~\ref{lem:trianglemoments}. In particular $\Var_{\mathrm{al}}Q\le\rho_m\Var_0Q$ with $\rho_m=1+(2(m-1)+(1-2p)^2/q)\,t/n$.
\end{lemma}

\begin{proof}
Given $\sigma$, $A^\ell_{ij}-p=\mu_{ij}+\xi^\ell_{ij}$ with $\mu_{ij}=\delta\sigma_i\sigma_j$ and the $\xi^\ell_{ij}$ independent, centred, with variance $v_{ij}=a+b\sigma_i\sigma_j$. Then $Q=\sum_{i<j}\big[\binom m2\mu_{ij}^2+(m-1)\mu_{ij}\sum_\ell\xi^\ell_{ij}+\sum_{\ell<\ell'}\xi^\ell_{ij}\xi^{\ell'}_{ij}\big]$. The three kinds of terms are conditionally uncorrelated and different pairs are independent, so $\E[Q\mid\sigma]=\binom m2N\delta^2$ and $\Var(Q\mid\sigma)=\sum_{i<j}[(m-1)^2m\delta^2v_{ij}+\binom m2v_{ij}^2]$. Taking expectations with $\E v_{ij}=a$ and $\E v_{ij}^2=a^2+b^2$ gives the moments, and $\delta=0$ gives the null moments. For the ratio, $a\le q$ and $b^2=(1-2p)^2\eta q$ give $\Var_{\mathrm{al}}Q/\Var_0Q\le1+2(m-1)\eta+(1-2p)^2\eta/q$.
\end{proof}

\paragraph{Proof of (iii).}
The identity is Lemma~\ref{lem:aligned}. By Lemma~\ref{lem:veblen} with $x=B_m/A_m$ (so $nx=u_m$) and $A_m^N\le\exp(\binom m2N\eta^2)$ we get the bound, where $\binom m2N\eta^2=\frac{m(m-1)}2\cdot\frac{n(n-1)}2\cdot\frac{t^2}{n^2}$. If $x:=\kappa_\eta mt<x_{\alpha\beta}<1$ then $\binom m2N\eta^2\le m^2t^2/4\le x^2/4$ and $u_m\le x$, so $g_n(u_m)\le g(x)$ (Lemma~\ref{lem:gfunction}, $g$ increasing) and $1+\chi^2\le e^{x^2/4+g(x)}<e^{c_{\alpha\beta}}=1+4(1-\alpha-\beta)^2$, which by Lemma~\ref{lem:tensor}(c),(d) rules out errors $(\alpha,\beta)$. The bound $\kappa_\eta\le(1-\eta^2)^{-1}$ follows from $\artanh\eta=\sum_{k\ge0}\eta^{2k+1}/(2k+1)\le\eta/(1-\eta^2)$.

For the test, Lemma~\ref{lem:Qmoments} and Chebyshev's inequality give size $\le\alpha$, and type II error $\le\beta$ once $\E_{\mathrm{al}}Q\ge(\Var_0Q/\alpha)^{1/2}+(\Var_{\mathrm{al}}Q/\beta)^{1/2}$. Dividing by $(\Var_0Q)^{1/2}$ gives $(\binom m2N\eta^2)^{1/2}\ge\alpha^{-1/2}+(\rho_m/\beta)^{1/2}$. \qed

\begin{remark}[Sharp aligned threshold]
By Lemma~\ref{lem:curieweiss} applied with $x=B_m/A_m$ ($nx\to mt$) and $N\log A_m\to\binom m2t^2/2$, for fixed $m$ and $t$ with $mt<1$ one gets $1+\chi^2(\Pal\|P_0^{\otimes m})\to e^{-mt/2-mt^2/4}(1-mt)^{-1/2}$, which is finite exactly when $mt<1$. This is consistent with the sharp aligned threshold $mt=1$ of \citet[Theorem~2.1]{chatterjee2022multilayer}. Their setting has $m$ layers sharing one planted bipartition, in the sparse regime with fixed $a_k,b_k$ and average degree $>1$. They prove contiguity if $\sum_k\lambda_k^2+\mu^2/\gamma<1$ and asymptotic singularity if $>1$, where $\lambda_k^2=(a_k-b_k)^2/(2(a_k+b_k))$. Identical layers and no covariates give $\lambda_k^2=t$ and $\mu=0$. We cite this result for the sharp constant and do not reprove it. In the dense regime we only claim the bounds of Theorem~\ref{thm:price}(iii).
\end{remark}

\paragraph{Proof of (iv).}
Let $n\ge n_0$ and $0<t\le1/2$. Then $\eta=t/n\le1/42$, so $\kappa_\eta\le4/3$, and $nq\ge1$.

\emph{Unaligned lower bound.} By (i), $m^*_{\mathrm{un}}\ge c_{\alpha\beta}/g(t)\ge6(1-t)c_{\alpha\beta}/t^3\ge3c_{\alpha\beta}/t^3$.

\emph{Unaligned upper bound.} Since $t\le1/2$, $nq\ge1$ and $n\ge21$, we have $\eps_n\le(2nq)^{-3/2}+2.25/n\le1$, and $6n^2/((n-1)(n-2))\le7$ for $n\ge21$. So the condition of (ii) holds for every $m\ge K_{\mathrm{un}}/t^3$, and $m^*_{\mathrm{un}}\le1+K_{\mathrm{un}}/t^3$.

\emph{Aligned lower bound.} By (iii), no test exists when $\frac43mt<x_{\alpha\beta}$, so $m^*_{\mathrm{al}}\ge3x_{\alpha\beta}/(4t)$.

\emph{Aligned upper bound.} Let $m=\lceil1+3K'/t\rceil\le2+3K'/t$. Then $m\ge2$, $m-1\ge3K'/t$, and $\rho_m\le2$ because $(2(m-1)+(1-2p)^2/q)t\le2t+6K'+t/q\le1+6K'+1/(2q)\le n$ by the choice of $n_0$. Hence $\binom m2N\eta^2\ge\frac{(m-1)^2(n-1)t^2}{4n}\ge\frac{9(n-1)}{4n}K'^2\ge K'^2\ge(\alpha^{-1/2}+(\rho_m/\beta)^{1/2})^2$, so (iii) gives a test with errors $(\alpha,\beta)$ and $m^*_{\mathrm{al}}\le2+3K'/t$.

\emph{Ratio.} Since $2t\le1$, $m^*_{\mathrm{un}}/m^*_{\mathrm{al}}\ge(3c_{\alpha\beta}/t^3)/(2+3K'/t)\ge3c_{\alpha\beta}/((1+3K')t^2)$. In the other direction, $m^*_{\mathrm{un}}/m^*_{\mathrm{al}}\le(1+K_{\mathrm{un}}/t^3)\cdot4t/(3x_{\alpha\beta})=4(t^3+K_{\mathrm{un}})/(3x_{\alpha\beta}t^2)\le4(1/8+K_{\mathrm{un}})/(3x_{\alpha\beta}t^2)$. \qed

\begin{remark}[Asymptotically sharp constants]\label{rem:sharp}
The constants in (iv) come from Chebyshev and $\chi^2$-to-TV conversions and are loose. For fixed $t<1$, dense $p\in(0,1)$ and fixed $m$, sharp constants follow from the likelihood-ratio limit of \citet[arXiv version, Theorem~2.2 with Proposition~3.4, display~(3.2)]{banerjee2018contiguity}. Under $P_0$, $dP_\delta/dP_0\to\exp\big(\sum_{i\ge3}(2t^{i/2}Z_i-t^i)/(4i)\big)$ in law, with independent $Z_i\sim N(0,2i)$. The hypotheses of that theorem are $a_n/n\to p\in(0,1)$ and $(a_n-b_n)^2/(a_n+b_n)\to c$ with $t=c/(2(1-p))<1$, which is our dense regime with $a_n=n(p+\delta)$ and $b_n=n(p-\delta)$. Thus $\log\prod_{\ell\le m}L(G_\ell)\to N(-mg(t)/2,\,mg(t))$ under $P_0^{\otimes m}$, and by Le~Cam's third lemma the most powerful level-$\alpha$ test of $P_0^{\otimes m}$ against $P_\delta^{\otimes m}$ has asymptotic power $1-\Phi(z_\alpha-(mg(t))^{1/2})$. The one-sample signed-triangle test with known $p$ has asymptotic power $1-\Phi(z_\alpha-(mt^3/6)^{1/2})$ (Remark~\ref{rem:trianglenormal}). Since $g(t)=t^3/6+O(t^4)$, this test is asymptotically optimal including the constant as $t\to0$. The two-sample version, which does not use the knowledge of $P_0$, pays a factor $2$ in $m$.
\end{remark}

\subsection{Proposition~\ref{prop:signcancel}: sign cancellation}\label{app:signcancel}

Throughout this subsection $\sigma,\tau$ are independent and uniform on $\{\pm1\}^n$, $\rho=\sigma\circ\tau$, $h_{ij}=\delta(\sigma_i\sigma_j-\tau_i\tau_j)=\delta\sigma_i\sigma_j(1-\rho_i\rho_j)\in\{0,\pm2\delta\}$, $0<2\delta\le\min(p,1-p)$, and $P^{(2)}_\delta$ is the law of $A$ when the $A_{ij}$ are independent $\Bern(p+h_{ij})$ given $(\sigma,\tau)$. Put $\bar v=q+2\delta|1-2p|$, $r=\bar v/q=1+2|1-2p|(\eta/q)^{1/2}$ and $v_{\min}=q-2\delta|1-2p|-4\delta^2$. Every conditional edge variance $v_{ij}=(p+h_{ij})(1-p-h_{ij})=q+h_{ij}(1-2p)-h_{ij}^2$ lies in $[v_{\min},\bar v]$.

\begin{proposition*}[Restatement of Proposition~\ref{prop:signcancel}]
\begin{itemize}
\item[(i)] For every $S\subseteq E_n$, $t_S:=\E\prod_{\{i,j\}\in S}(\sigma_i\sigma_j-\tau_i\tau_j)=2^{\beta(S)}\ind[S\text{ even and bipartite}]$.
\item[(ii)] $\E[T_p\mid\sigma,\tau]=0$; $\E\,S_{C_4}=6\binom n4\delta^4$, $\Var_0S_{C_4}=3\binom n4q^4$, and $\Var S_{C_4}\le\rho_4\Var_0S_{C_4}$ under $P^{(2)}_\delta$, with $\rho_4$ as in Lemma~\ref{lem:rank2moments}. With $s_4^2=12\binom n4\eta^4=t^4\frac{(n-1)(n-2)(n-3)}{2n^3}$, the test $\ind\{|D_4|\ge(6m\binom n4q^4/\alpha)^{1/2}\}$, $D_4=\sum_\ell S_{C_4}(G_\ell)-\sum_\ell S_{C_4}(G'_\ell)$, has size $\le\alpha$ under $P=Q=P_0$ and type II error $\le\beta$ under $(P^{(2)}_\delta,P_0)$ whenever $m s_4^2\ge(\sqrt{2/\alpha}+\sqrt{(1+\rho_4)/\beta})^2$.
\item[(iii)] $\E_0(dP^{(2)}_\delta/dP_0)^2=\sum_{S\in\Evn,\,S\text{ bipartite}}\eta^{|S|}4^{\beta(S)}$. If $t<1/2$ this is at most $\E_{s\sim\Bin(n,1/2)}\exp\{G(2s\eta)+G(2(n-s)\eta)\}\le\exp\{G(2t)\}$, where $G(x)=\sum_{l\ge2}x^{2l}/(2l)=-\frac12\log(1-x^2)-\frac{x^2}2$; the middle quantity tends to $\exp\{2G(t)\}=e^{-t^2}/(1-t^2)$. Errors $(\alpha,\beta)$ require $m\ge c_{\alpha\beta}/G(2t)\ge(1-4t^2)c_{\alpha\beta}/(4t^4)$.
\item[(iv)] Under $(P^{(2)}_\delta,P_0)$ the triangle test $\psi_3$ of Theorem~\ref{thm:price}(ii) has power at most $\alpha(1+\rho_3)/2$, with $\rho_3$ as in Lemma~\ref{lem:rank2moments}; for fixed $p$, $m$, $t$ the statistic $D_3/(2m\binom n3q^3)^{1/2}$ converges in law to $N(0,1)$ under both hypotheses, so the triangle z-test has asymptotic power exactly $\alpha$.
\end{itemize}
\end{proposition*}

\begin{lemma}[Sign-cancellation identity]\label{lem:signid}
For every $S\subseteq E_n$, $t_S=2^{\beta(S)}\ind[S\text{ even and bipartite}]$. In particular $t_S=0$ for every triangle, $t_S=2$ for every $4$-cycle and $t_S=0$ whenever $S$ has a vertex of odd degree.
\end{lemma}

\begin{proof}
Expanding the product and using independence of $\sigma$ and $\tau$ and \eqref{eq:oddset},
\[
t_S=\sum_{S_1\subseteq S}(-1)^{|S\setminus S_1|}\,\E\sigma^{\Odd(S_1)}\,\E\tau^{\Odd(S\setminus S_1)}=\sum_{S_1\subseteq S}(-1)^{|S\setminus S_1|}\ind[S_1\text{ even}]\,\ind[S\setminus S_1\text{ even}].
\]
Since $\deg_S=\deg_{S_1}+\deg_{S\setminus S_1}$, both indicators can equal one only if $S$ is even, so $t_S=0$ unless $S$ is even.

If $S$ is even, $S\setminus S_1$ is even iff $S_1$ is, and $t_S=(-1)^{|S|}\sum_{Z\in\mathcal Z(S)}(-1)^{|Z|}$, where $\mathcal Z(S)$ is the set of even subsets of $S$. Under symmetric difference $\mathcal Z(S)$ is the cycle space of the graph $(V(S),S)$, a vector space over $\mathrm{GF}(2)$ of dimension $\beta(S)$, and $\chi(Z)=(-1)^{|Z|}$ is a group homomorphism into $\{\pm1\}$ because $|Z\triangle Z'|\equiv|Z|+|Z'|\pmod 2$. The sum of a homomorphism over a finite group is the group order if it is trivial and $0$ otherwise (if $\chi(Z_0)=-1$, substituting $Z\mapsto Z\triangle Z_0$ changes the sign of the sum). $\chi$ is trivial iff every even subset of $S$ has even size iff every cycle of $S$ has even length (cycles span $\mathcal Z(S)$) iff $(V(S),S)$ is bipartite. In that case $|S|$ is even because $S\in\mathcal Z(S)$, and $t_S=2^{\beta(S)}$.
\end{proof}

\begin{lemma}[Moments under the rank-two model]\label{lem:rank2moments}
Under $P^{(2)}_\delta$:
\begin{itemize}
\item[(a)] $\E[T_p\mid\sigma,\tau]=0$, and $\binom n3v_{\min}^3\le\Var T_p\le\binom n3\bar v^3+6\binom n3\delta^2\bar v^2+2N(n-1)(n-2)\delta^4\bar v$; hence $\Var T_p\le\rho_3\Var_0T_p$ with
$\rho_3=r^3+6\eta r^2+6(n-1)\eta^2r\le r^3+6t(r^2+tr)/n$.
\item[(b)] $\E[S_{C_4}\mid\sigma,\tau]=16\delta^4K$, where $K$ is the number of $4$-cycles whose vertices alternate between $\{\rho=1\}$ and $\{\rho=-1\}$; $\E K=3\binom n4/8$, $\Var K\le\frac{27}8n^2\binom n4$, $\E S_{C_4}=6\binom n4\delta^4$, and $\Var S_{C_4}\le\rho_4\Var_0S_{C_4}$ with
\[
\begin{aligned}
\rho_4&=r^4+8\eta r^3+(8n-4)\eta^2r^2+8(n^2-5n+8)\eta^3r+288n^2\eta^4\\
&\le r^4+\frac{8t(r^3+tr^2+t^2r)}{n}+\frac{288t^4}{n^2}.
\end{aligned}
\]
\end{itemize}
\end{lemma}

\begin{proof}
We apply Lemma~\ref{lem:latentvar} with $\Lambda=(\sigma,\tau)$ and $\mu=h$, and bound the variance class by class. Three facts are used repeatedly:
\begin{itemize}
\item $h^S=\delta^{|S|}\,\sigma^{\Odd(S)}\prod_{\{i,j\}\in S}(1-\rho_i\rho_j)$ and $\E h^S=\delta^{|S|}t_S$ (Lemma~\ref{lem:signid});
\item $h_{ij}^2=2\delta^2(1-\rho_i\rho_j)$, so $\E h_{ij}^2=2\delta^2$;
\item $v^R\le\bar v^{|R|}$ pointwise.
\end{itemize}

(a) For a triangle $\Delta$, $\sigma^{\Odd(\Delta)}=1$ and $\prod_{e\in\Delta}(1-\rho_i\rho_j)=0$ because two of the three values $\rho_i,\rho_j,\rho_k$ coincide. Hence $c_\emptyset=\sum_\Delta h^\Delta\equiv0$. The non-empty $R$ are triangles ($c_R=1$), paths $i-j-k$ ($c_R=h_{ik}$, $\E c_R^2=2\delta^2$) and edges $ij$ ($c_R=\sum_{k\ne i,j}h_{ik}h_{jk}$). For an edge,
\[
\E c_R^2=\sum_k\E h_{ik}^2h_{jk}^2+\sum_{k\ne k'}\E[h_{ik}h_{jk}h_{ik'}h_{jk'}]=4\delta^4(n-2)+2\delta^4(n-2)(n-3)=2\delta^4(n-1)(n-2),
\]
using $\E h_{ik}^2h_{jk}^2=4\delta^4\E(1-\rho_i\rho_k)(1-\rho_j\rho_k)=4\delta^4$ and Lemma~\ref{lem:signid} for the $4$-cycle $i\,k\,j\,k'$. With $\binom n3$ triangles, $3\binom n3$ paths, $N$ edges and $v^R\le\bar v^{|R|}$ we get the upper bound. The lower bound keeps only the triangle terms and uses $v_e\ge v_{\min}$. Dividing by $\binom n3q^3$ gives $\rho_3$, since $2N(n-1)(n-2)=6(n-1)\binom n3$.

(b) \emph{Mean.} For a $4$-cycle $C$, $\sigma^{\Odd(C)}=1$ and $\prod_{e\in C}(1-\rho_i\rho_j)\in\{0,16\}$, with value $16$ iff $\rho$ alternates around $C$. So $c_\emptyset=16\delta^4K$. Each $4$-cycle alternates with probability $2/16$, which gives $\E K$ and $\E S_{C_4}=16\delta^4\E K=6\binom n4\delta^4$.

\emph{The variance $\Var K$.} The indicators of two $4$-cycles sharing at most one vertex $x$ are conditionally independent given $\rho_x$, with conditional means $1/8$ not depending on $\rho_x$, hence uncorrelated. A $4$-cycle meets at most $3\cdot3n^2$ $4$-cycles in two or more vertices, and each covariance is at most $1/8$. So $\Var K\le\frac18\cdot3\binom n4\cdot9n^2$.

\emph{Coefficients.} The non-empty $R$ contained in a $4$-cycle fall into five classes.
\begin{enumerate}
\item[(1)] The $4$-cycles themselves, with $c_R=1$.
\item[(2)] Paths with $3$ edges. Here $c_R$ is the closing edge and $\E c_R^2=2\delta^2$, and there are $12\binom n4$ of them.
\item[(3)] Paths $i-j-k$ with $2$ edges, contained in the cycles $i\,j\,k\,l$ with $l\notin\{i,j,k\}$. Here $c_R=\sum_lh_{kl}h_{li}$ and, as in (a), $\E c_R^2=4\delta^4(n-3)+2\delta^4(n-3)(n-4)=2\delta^4(n-2)(n-3)$. There are $3\binom n3$ of them.
\item[(4)] Pairs of disjoint edges $\{ij,kl\}$, contained in the cycles $i\,j\,k\,l$ and $i\,j\,l\,k$. Here $c_R=h_{jk}h_{li}+h_{jl}h_{ki}$ and $\E c_R^2=4\delta^4+4\delta^4+2\cdot2\delta^4=12\delta^4$. There are $3\binom n4$ of them.
\item[(5)] Single edges $ij$, contained in the cycles $i\,j\,k\,l$ with $(k,l)$ ordered, distinct and outside $\{i,j\}$. Here $c_R=\sum_{(k,l)}h_{jk}h_{kl}h_{li}$.
\end{enumerate}
For class (5) we expand $\E c_R^2$ over pairs $((k,l),(k',l'))$.
\begin{itemize}
\item Equal pairs give $\E h_{jk}^2h_{kl}^2h_{li}^2=8\delta^6$, with $(n-2)(n-3)$ terms.
\item Pairs with $\{k,l\}\cap\{k',l'\}=\emptyset$ form a $6$-cycle and give $2\delta^6$, with $(n-2)(n-3)(n-4)(n-5)$ terms.
\item Pairs with $k=k'$, $l\ne l'$ (or $l=l'$, $k\ne k'$) give $\E[h_{jk}^2\,h^{C'}]$, where $C'$ is the $4$-cycle $k\,l\,i\,l'$. This equals $2\delta^2\cdot16\delta^4\,\E[(1-\rho_j\rho_k)\ind\{\rho\text{ alternates on }C'\}]=4\delta^6$, with $2(n-2)(n-3)(n-4)$ terms.
\item Pairs with $k=l'$, $l\ne k'$ (or $l=k'$, $k\ne l'$) consist of six distinct edges containing the triangle $j\,k\,k'$ (or $j\,k\,l$), so they give $0$ by Lemma~\ref{lem:signid}.
\item The swapped pair $(k,l),(l,k)$ gives $\E[h_{kl}^2h^{C''}]$, where $C''$ is the $4$-cycle $j\,k\,i\,l$. It vanishes because $h^{C''}\ne0$ forces alternation on $C''$, hence $\rho_k=\rho_l$ and $h_{kl}=0$.
\end{itemize}
Altogether $\E c_R^2=2\delta^6(n-2)(n-3)(n^2-5n+8)$.

\emph{Conclusion.} Summing the five contributions with $v^R\le\bar v^{|R|}$, adding $\Var c_\emptyset=256\delta^8\Var K$, and dividing by $\Var_0S_{C_4}=3\binom n4q^4$ gives
\[
\rho_4=r^4+8\eta r^3+8(n-2)\eta^2r^2+12\eta^2r^2+8(n^2-5n+8)\eta^3r+288n^2\eta^4,
\]
using $\binom n3/\binom n4=4/(n-3)$ and $N/\binom n4=12/((n-2)(n-3))$. The final inequality uses $\eta=t/n$, $(8n-4)\eta^2\le8t^2/n$ and $n^2-5n+8\le n^2$.
\end{proof}

\begin{remark}
We verified the exact values of all coefficients $\E c_R^2$ in this proof by enumeration over $(\sigma,\tau)$ for $n=6,7$, and compared the resulting variance bounds with exact variances for $n=5$ (numerical check (vi), parts (c) and (d)).
\end{remark}

\begin{lemma}[Second moment under the rank-two model]\label{lem:rank2chi2}
$\E_0(dP^{(2)}_\delta/dP_0)^2=\sum_{S\in\Evn,\,S\text{ bipartite}}\eta^{|S|}4^{\beta(S)}$, and the bounds of Proposition~\ref{prop:signcancel}(iii) hold.
\end{lemma}

\begin{proof}
By Lemma~\ref{lem:secondmoment} with $\Lambda=(\sigma,\tau)$, $\E_0(dP^{(2)}_\delta/dP_0)^2=\E\prod_{i<j}(1+\eta z_{ij}z'_{ij})$, where $z_{ij}=\sigma_i\sigma_j-\tau_i\tau_j$ and $z'$ is built from an independent copy $(\sigma',\tau')$. Expanding, $\E\prod_{i<j}(1+\eta z_{ij}z'_{ij})=\sum_{S}\eta^{|S|}\E z^S\,\E z'^S=\sum_S\eta^{|S|}t_S^2$, and Lemma~\ref{lem:signid} gives the first formula.

For the bound, let $\rho'=\sigma'\circ\tau'$ and $\omega=\sigma\circ\sigma'$. The map $(\sigma,\tau,\sigma',\tau')\mapsto(\sigma,\rho,\rho',\omega)$ is a bijection of $(\{\pm1\}^n)^4$, so $\rho,\rho',\omega$ are independent and uniform. Since $z_{ij}=\sigma_i\sigma_j(1-\rho_i\rho_j)$, $z_{ij}z'_{ij}=4\,\omega_i\omega_j\,\ind[\{i,j\}\in\mathcal A]$ with $\mathcal A=\{\{i,j\}:\rho_i\ne\rho_j,\ \rho'_i\ne\rho'_j\}$. Conditioning on $(\rho,\rho')$ and applying Lemmas~\ref{lem:evenexpansion} and \ref{lem:veblen} to the graph $([n],\mathcal A)$ with $x_e\equiv4\eta$,
\[
\E_0\Big(\frac{dP^{(2)}_\delta}{dP_0}\Big)^2=\E_{\rho,\rho'}\sum_{S\subseteq\mathcal A,\ S\text{ even}}(4\eta)^{|S|}\le\E_{\rho,\rho'}\exp\Big(\sum_{C\in\Cyc([n],\mathcal A)}(4\eta)^{|C|}\Big).
\]
\emph{Cycle structure.} Classify vertices by the type $(\rho_i,\rho'_i)$, with class sizes $a=|\{++\}|$, $d=|\{--\}|$, $b=|\{+-\}|$, $c=|\{-+\}|$. Then $\{i,j\}\in\mathcal A$ iff the types are opposite, so $\mathcal A$ is the disjoint union of the complete bipartite graphs $K_{a,d}$ and $K_{b,c}$. All their cycles have even length $2l\ge4$, and $K_{a,d}$ has $\ff al\ff dl/(2l)\le(ad)^l/(2l)$ cycles of length $2l$.

\emph{The bound.} With $s=a+d\sim\Bin(n,1/2)$ (the events $\rho_i\rho'_i=1$ are independent with probability $1/2$), $ad\le s^2/4$ and $bc\le(n-s)^2/4$. Hence $\sum_C(4\eta)^{|C|}\le\sum_{l\ge2}\big[(2s\eta)^{2l}+(2(n-s)\eta)^{2l}\big]/(2l)=G(2s\eta)+G(2(n-s)\eta)$, which is finite because $2n\eta=2t<1$. This is the middle bound. Since $x^{2l}+y^{2l}\le(x+y)^{2l}$ for $x,y\ge0$, $G(x)+G(y)\le G(x+y)$, which gives $\exp G(2t)$.

\emph{The limit and the sample size.} As $n\to\infty$, $s/n\to1/2$ almost surely along the obvious coupling, so $G(2s\eta)+G(2(n-s)\eta)\to2G(t)$, and dominated convergence (the integrand is at most $e^{G(2t)}$) gives the limit. Finally $G(2t)=\sum_{l\ge2}(4t^2)^l/(2l)\le(4t^2)^2/(4(1-4t^2))$, and Lemma~\ref{lem:tensor}(d) gives the sample-size bound exactly as in Theorem~\ref{thm:price}(i).
\end{proof}

\paragraph{Proof of Proposition~\ref{prop:signcancel}.}
(i) is Lemma~\ref{lem:signid}, and (iii) is Lemma~\ref{lem:rank2chi2}.

(ii) The moments are Lemma~\ref{lem:rank2moments}(b) and, for $\delta=0$, the orthogonality of distinct $4$-cycles under $P_0$. The test analysis is that of Theorem~\ref{thm:price}(ii) with $V_0=3\binom n4q^4$, mean shift $m\cdot6\binom n4\delta^4$ and $6\binom n4\delta^4/V_0^{1/2}=s_4$.

(iv) Under $(P^{(2)}_\delta,P_0)$, $\E D_3=0$ by Lemma~\ref{lem:rank2moments}(a), so by Chebyshev's inequality $\Prob(|D_3|\ge\lambda_3)\le\Var D_3/\lambda_3^2=\alpha\,m(\Var T_p+V_0)/(2mV_0)\le\alpha(1+\rho_3)/2$, and $\rho_3\to1$ when $\eta/q\to0$ and $t^2/n\to0$. The central limit statement is proved in Remark~\ref{rem:trianglenormal}. \qed

\subsection{Theorem~\ref{thm:trees}: invariance of tree statistics}\label{app:trees}

For a graph $F$ with at least one edge let $\Cop_n(F)$ be its set of copies in $K_n$, $N_F=|\Cop_n(F)|$, and define the normalized signed count
$Z_F(G)=(N_Fq^{e(F)})^{-1/2}\sum_{\varphi\in\Cop_n(F)}(A-p)^\varphi$.

\begin{theorem*}[Restatement of Theorem~\ref{thm:trees}]
\begin{itemize}
\item[(i)] Let $W$ be a graphon with $d_W(x)=\int_0^1W(x,y)\,dy=d$ for a.e.\ $x$, and $H$ a bounded symmetric measurable function with $d_H\equiv0$ a.e.\ and $0\le W+H\le1$. Then $t(T,W+H)=t(T,W)=d^{\,e(T)}$ for every finite forest $T$. Consequently, for all $n\ge v(T)$, $\E\,\tinj(T,G(n,W+H))=\E\,\tinj(T,G(n,W))=d^{\,e(T)}$; under both $G(n,W)$ and $G(n,W+H)$ every vertex degree has law $\Bin(n-1,d)$ and the edge count has mean $Nd$ and variance $Nd(1-d)$.
\item[(ii)] Fix $p\in(0,1)$, $t>0$, $m\ge1$, pairwise non-isomorphic trees $T_1,\dots,T_r$ with at least one edge, and let $\delta=\delta_n=(tq/n)^{1/2}$ (admissible for $n\ge t/q$). Under $P_0$ and under $P_\delta$, $(Z_{T_1},\dots,Z_{T_r})\to N(0,I_r)$ in law. Moreover, under $P_\delta$, $(Z_{T_1},\dots,Z_{T_r},Z_{K_3})\to N\big((0,\dots,0,\sqrt{t^3/6}),I_{r+1}\big)$. Hence for $2m$ independent graphs the vector $(Z_{T_j}(G_\ell),Z_{T_j}(G'_\ell))_{j,\ell}$ converges to $N(0,I_{2mr})$ under $H_0$ and $H_1$, every test $\ind\{Z\in\mathcal A\}$ with $\partial\mathcal A$ Lebesgue-null has the same limiting rejection probability under both, and the one-sided signed-triangle z-test $\ind\{D_3\ge z_\alpha(2m\binom n3q^3)^{1/2}\}$ has asymptotic power $1-\Phi(z_\alpha-\sqrt{mt^3/12})$.
\end{itemize}
\end{theorem*}

\begin{lemma}[Leaf removal]\label{lem:treeleaf}
If $U$ is a graphon with $d_U\equiv d$ a.e., then $t(T,U)=d^{\,e(T)}$ for every finite forest $T$.
\end{lemma}

\begin{proof}
Induction on $e(T)$. If $e(T)=0$, $t(T,U)=1$. Otherwise $T$ has a leaf $l$ with unique neighbour $w$. By Fubini, integrating $x_l$ first, $\int U(x_w,x_l)\,dx_l=d_U(x_w)=d$ for a.e.\ $x_w$, so $t(T,U)=d\cdot t(T-l,U)$, where $T-l$ is the forest obtained by deleting $l$ and its edge.
\end{proof}

\paragraph{Proof of (i).}
$U:=W+H$ is a graphon and $d_U=d_W+d_H=d$ a.e., so Lemma~\ref{lem:treeleaf} applies to $W$ and to $U$. For $n\ge v(T)$ and an injective $\phi:V(T)\to[n]$, the latent values $(U_{\phi(u)})_{u}$ are i.i.d.\ uniform, so $\E\prod_{uv\in E(T)}A_{\phi(u)\phi(v)}=t(T,U)$ and $\E\,\tinj(T,G(n,U))=t(T,U)$. For the degree of vertex $i$, condition on $U_i=x$. Then the pairs $(U_j,A_{ij})_{j\ne i}$ are i.i.d.\ with $\Prob(A_{ij}=1\mid U_i=x)=d_U(x)=d$, so $\deg(i)\sim\Bin(n-1,d)$ independently of $U_i$. For the edge count, $\Var A_e=d(1-d)$, $\Cov(A_e,A_f)=t(P_3,U)-d^2=0$ if $e,f$ share one vertex ($P_3$, the path with two edges, is a tree), and $A_e,A_f$ are independent if $e\cap f=\emptyset$. \qed

\begin{remark}\label{rem:treesscope}
(a) The planted partition is $G(n,W+H)$ with $W\equiv p$ and $H=\delta\,s\otimes s$, $s=\ind_{[0,1/2)}-\ind_{[1/2,1)}$, so (i) contains the identities $\E_\delta(\#\text{edges})=Np$ and $\Var_\delta(\#\text{edges})=Nq$, which numerical check (ix) confirms by enumeration. The first difference in the moments of the edge count appears in the third central moment, $6\binom n3\delta^3$, and for the rank-two model only in the fourth, $144\binom n4\delta^4$.

(b) Only \emph{injective} densities are invariant. Homomorphism densities of the sampled graph include non-injective maps whose images may contain cycles (for example, a path with three edges whose endpoints coincide), and joint laws of several degrees are not invariant. For irregular $W$, trees may detect a zero-degree $H$ at first order. For the path $P_4$ with three edges the first-order coefficient is $\langle H,d_W\otimes d_W\rangle$, which is $1.56\cdot10^{-4}$ for the SMOOTH null and the $H_0$ of the experiments (numerical check (xi)).

(c) Message-passing networks and colour refinement on graphons are characterized by tree homomorphism densities \citep{boker2023fine,boker2023wl}. By (i), the graphons $W$ and $W+H$ have identical tree homomorphism densities, so by that characterization no message-passing network and no colour-refinement statistic separates them at the level of graphons.
\end{remark}

\begin{lemma}[Method of moments for signed subgraph chaoses]\label{lem:moments}
Let $F_1,\dots,F_r$ be pairwise non-isomorphic connected graphs with at least one edge, $v_j=v(F_j)$, $e_j=e(F_j)$, and let $q>0$. Let $(\xi_e)_{e\in E_n}$ be independent with $\E\xi_e=0$, $|\xi_e|\le1$ and $\E\xi_e^2=v_e$. Put
\[
W_j=(N_{F_j}q^{e_j})^{-1/2}\sum_{\varphi\in\Cop_n(F_j)}\xi^\varphi,\qquad s_j=(N_{F_j}q^{e_j})^{-1}\sum_{\varphi\in\Cop_n(F_j)}v^\varphi .
\]
For every $k=(k_1,\dots,k_r)\in\mathbb N^r$ there is $C<\infty$ depending only on $F_1,\dots,F_r$, $k$ and $q$ (not on $n$ or on the law of $\xi$) such that for $n\ge2\max_jv_j$,
\[
\Big|\E\prod_{j}W_j^{k_j}-\prod_jM_{k_j}\,s_j^{k_j/2}\Big|\le C\,n^{-1/2},
\]
where $M_k=\E Z^k$ for $Z\sim N(0,1)$ ($M_k=(k-1)!!$ for even $k$ and $M_k=0$ for odd $k$).
\end{lemma}

\begin{proof}
We expand the mixed moment as a sum over tuples of copies, show that only tuples whose copies coincide in pairs contribute at leading order, and count the remaining tuples. Let $K=\sum_jk_j$, list the factors as $i=1,\dots,K$ with shapes $j(i)$, and let $V=\sum_iv_{j(i)}$. Then $\E\prod_jW_j^{k_j}=\prod_j(N_{F_j}q^{e_j})^{-k_j/2}\sum\E\prod_{i=1}^K\xi^{\varphi_i}$, the sum over tuples $(\varphi_1,\dots,\varphi_K)$ with $\varphi_i\in\Cop_n(F_{j(i)})$. For a tuple let $U=\bigcup_i\varphi_i$ and $m_e=|\{i:e\in\varphi_i\}|$. By independence $\E\prod_i\xi^{\varphi_i}=\prod_{e\in U}\E\xi_e^{m_e}$, which vanishes unless $m_e\ge2$ for all $e\in U$ (call such tuples admissible). For admissible tuples $|\E\prod_i\xi^{\varphi_i}|\le\prod_{e\in U}\E\xi_e^2\le\prod_e v_e\le1$, using $|\xi_e|\le1$ (so $v_e\le1$).

\emph{Claim.} For an admissible tuple, $|V(U)|\le V/2$, with equality iff the tuple is \emph{paired}: $[K]$ splits into pairs $\{i,i'\}$ with $\varphi_i=\varphi_{i'}$ and the vertex sets of different pairs are disjoint. To see this, let $x\in V(\varphi_i)$. As $F_{j(i)}$ has no isolated vertex, some $e\ni x$ lies in $\varphi_i$, and as $m_e\ge2$, $e$ lies in some $\varphi_{i'}$, $i'\ne i$. So every vertex of $U$ lies in at least two of the $\varphi_i$, whence $V=\sum_i|V(\varphi_i)|\ge2|V(U)|$. In case of equality every vertex lies in exactly two copies. Take $i$ and $x\in V(\varphi_i)$, and let $\varphi_{i'}$ be the other copy containing $x$. Every edge $xy\in\varphi_i$ lies in a second copy, which contains $x$ and hence is $\varphi_{i'}$. So $y\in V(\varphi_{i'})$, and the two copies containing $y$ are again $\varphi_i,\varphi_{i'}$. By connectedness of $F_{j(i)}$, $E(\varphi_i)\subseteq E(\varphi_{i'})$, and symmetrically $E(\varphi_{i'})\subseteq E(\varphi_i)$. So $\varphi_i=\varphi_{i'}$, their common vertex set meets no other copy, and $j(i)=j(i')$ because the $F_j$ are non-isomorphic. The converse is clear.

\emph{Counting.} A tuple is determined by $K$ injective maps from the vertex sets of the $F_{j(i)}$ into $[n]$. Grouping these maps by the partition of the disjoint union of the vertex sets that they induce, we see that the number of tuples with $|V(U)|=u$ is at most $c\,n^u$ with $c$ depending only on the $F_j$ and $k$. Since $N_{F_j}=\ff n{v_j}/|\mathrm{Aut}F_j|\ge(n/2)^{v_j}/v_j!$ for $n\ge2v_j$, the normalizing factor is at most $c'\,n^{-V/2}q^{-\sum_jk_je_j/2}$.

Admissible tuples that are not paired have $|V(U)|\le(V-1)/2$ and contribute at most $c\,c'\,n^{-1/2}q^{-\sum_jk_je_j/2}$ in total. For a perfect matching $\pi$ of $[K]$ pairing indices of equal shape, the paired tuples with pairing $\pi$ contribute $\prod_{P\in\pi}\E(\xi^{\varphi_P})^2=\prod_{P\in\pi}v^{\varphi_P}$ each (the pairs have disjoint edge sets), and different $\pi$ give disjoint sets of tuples. Dropping the disjointness constraint between different pairs adds tuples with at most $V/2-1$ vertices, which after normalization contribute $O(n^{-1})$. Without the constraint the sum over tuples factorizes and equals $\prod_{P\in\pi}N_{F_{j(P)}}q^{e_{j(P)}}s_{j(P)}$. After normalization the paired contribution is therefore $\sum_\pi\prod_{P\in\pi}s_{j(P)}+O(n^{-1})$, and $\sum_\pi\prod_Ps_{j(P)}=\prod_jM_{k_j}s_j^{k_j/2}$ because a set of size $k_j$ has $M_{k_j}$ perfect matchings.
\end{proof}

\begin{lemma}[Reduction to the chaos under the planted partition]\label{lem:treeremainder}
Let $G\sim P_\delta$, $\mu_{ij}=\delta\sigma_i\sigma_j$, $\xi_e=A_e-p-\mu_e$ and $\bar v_1=q+\delta|1-2p|$. For a graph $F$ with at least one edge let $W_F=(N_Fq^{e(F)})^{-1/2}\sum_{\varphi\in\Cop_n(F)}\xi^\varphi$.
\begin{itemize}
\item[(a)] Let $\varphi$ be a tree with at least one edge and $R\subsetneq E(\varphi)$, and let $O(\varphi,R)=\Odd(\varphi\setminus R)\setminus V(R)$. Then $c(R)+|O(\varphi,R)|\ge2$ (with $c(\emptyset)=0$).
\item[(b)] For a tree $T$ with $e\ge1$ edges and $v=e+1$ vertices, $\E_\delta Z_T=0$ and
$\E_\delta(Z_T-W_T)^2\le2^ev^v\,\max(1,\bar v_1/q)^e\max(1,t)^e\,n^{-1}$.
\item[(c)] For the cycle $C_k$ ($k\ge3$), $\E_\delta Z_{C_k}=(N_{C_k}\eta^{k})^{1/2}$ and $\E_\delta(Z_{C_k}-\E_\delta Z_{C_k}-W_{C_k})^2\le2^kk^k\max(1,\bar v_1/q)^k\max(1,t)^k\,n^{-1}$.
\end{itemize}
\end{lemma}

\begin{proof}
Part (a) is a combinatorial fact about the subforests of a tree. Parts (b) and (c) bound the terms of the expansion in Lemma~\ref{lem:latentvar} that lie outside the chaos $W_F$, and each of these terms carries a factor of order $1/n$, by (a) for trees and because $c(R)\ge1$ for cycles.

(a) If $c(R)\ge2$ there is nothing to prove. If $R=\emptyset$, $O=\Odd(\varphi)$ contains the leaves of $\varphi$, of which there are at least two. If $c(R)=1$, let $K$ be a connected component of the forest $\varphi\setminus R$ (non-empty since $R\ne E(\varphi)$). $K$ meets $V(R)$: otherwise $K$ would be vertex-disjoint from $R$ and from the other components, contradicting connectedness of $\varphi$. $K$ meets $V(R)$ in at most one vertex: two common vertices would be joined by a path in $K$ and an edge-disjoint path in the connected $R$, creating a cycle in the tree $\varphi$. $K$ is a tree with at least one edge, so it has at least two leaves, one of which, $y$, is not in $V(R)$. Since $y$ has degree one in $\varphi\setminus R$, $y\in O(\varphi,R)$.

(b) By Lemma~\ref{lem:latentvar} with $\mathcal S=\Cop_n(T)$, $Z_T-W_T=(N_Tq^e)^{-1/2}\sum_{R}c_R\xi^R$ over edge sets $R$ strictly contained in some copy ($R=\emptyset$ included), and these terms are conditionally orthogonal to each other and to $W_T$ (no proper subset of a copy is itself a copy). Here $c_R=\sum_{\varphi\supsetneq R}\mu^{\varphi\setminus R}=\delta^{e-|R|}\sum_{\varphi\supsetneq R}\sigma^{\Odd(\varphi\setminus R)}$ by \eqref{eq:oddset}. In particular $\E[Z_T\mid\sigma]=(N_Tq^e)^{-1/2}c_\emptyset$ has mean $0$ because $\Odd(\varphi)\ne\emptyset$. By \eqref{eq:oddset}, $\E c_R^2=\delta^{2(e-|R|)}\,|\{(\varphi,\varphi'):\varphi,\varphi'\supsetneq R,\ \Odd(\varphi\setminus R)=\Odd(\varphi'\setminus R)\}|$, and $v^R\le\bar v_1^{|R|}$. Fix $\varphi$ and $R\subsetneq\varphi$. Every admissible $\varphi'$ contains the $|V(R)|+|O(\varphi,R)|$ vertices of $B=V(R)\cup O(\varphi,R)$ (as $O(\varphi,R)\subseteq\Odd(\varphi'\setminus R)\subseteq V(\varphi')$), so there are at most $v^{v}n^{v-|B|}$ of them. Hence
\[
\E_\delta(Z_T-W_T)^2\le\frac{1}{N_T}\sum_{\varphi\in\Cop_n(T)}\ \sum_{R\subsetneq\varphi}v^vn^{v-|V(R)|-|O(\varphi,R)|}\Big(\frac{\bar v_1}{q}\Big)^{|R|}\eta^{e-|R|}.
\]
With $\eta=t/n$, $v-e=1$ and $|V(R)|-|R|=c(R)$ (valid for forests, including $R=\emptyset$), the power of $n$ is $1-c(R)-|O(\varphi,R)|\le-1$ by (a). There are at most $2^e$ sets $R$, which gives (b).

(c) We repeat the computation with $\mathcal S=\Cop_n(C_k)$. Now $c_\emptyset=\sum_C\mu^C=N_{C_k}\delta^k$ is non-random, because every vertex of a cycle has even degree, and this gives the mean. For $\emptyset\ne R\subsetneq C$, $R$ is a union of $c(R)\ge1$ paths and the power of $n$ is $v-|V(R)|-|O|-(k-|R|)=-c(R)-|O|\le-1$ since $v=k$.
\end{proof}

\paragraph{Proof of (ii).}
We prove a conditional central limit theorem given $\sigma$ by the method of moments, transfer it from $W_T$ to $Z_T$, and then add the triangle.

\emph{Conditional moments.} Condition on $\sigma$. The $\xi_e$ are independent, centred, bounded by one, with variances $v_e=q+\mu_e(1-2p)-\delta^2\in[q-2\delta,q+\delta]$. With $s_j$ as in Lemma~\ref{lem:moments}, $|s_j-1|\le(1+2\delta/q)^{e_j}-1\to0$ deterministically, because $\delta=(tq/n)^{1/2}\to0$ with $q$ fixed. Lemma~\ref{lem:moments}, whose constant does not depend on the $v_e$, gives $\E[\prod_jW_{T_j}^{k_j}\mid\sigma]=\prod_jM_{k_j}+o(1)$ uniformly in $\sigma$, hence $\E_\delta\prod_jW_{T_j}^{k_j}\to\prod_jM_{k_j}$ for all $k$.

\emph{Limit law.} These are the mixed moments of $N(0,I_r)$, which is determined by its moments, so $(W_{T_1},\dots,W_{T_r})\to N(0,I_r)$ in law. By Lemma~\ref{lem:treeremainder}(b) and Slutsky's lemma the same holds for $(Z_{T_1},\dots,Z_{T_r})$. The case $\delta=0$ is the null ($Z_T=W_T$, $s_j=1$).

\emph{Adding the triangle.} Adding $F_{r+1}=K_3$ (connected and containing a cycle, hence non-isomorphic to the $T_j$) and using Lemma~\ref{lem:treeremainder}(c) with $\E_\delta Z_{K_3}=(\binom n3\eta^3)^{1/2}\to(t^3/6)^{1/2}$ gives the joint statement.

\emph{Several graphs and tests.} The $2m$ graphs are independent, so the vectors of the individual graphs converge jointly to independent limits. The portmanteau theorem gives the statement on tests. For the triangle test, $D_3/(2m\binom n3q^3)^{1/2}=(2m)^{-1/2}\big(\sum_\ell Z_{K_3}(G_\ell)-\sum_\ell Z_{K_3}(G'_\ell)\big)\to N\big((m/2)^{1/2}(t^3/6)^{1/2},1\big)$ under $H_1$ and $N(0,1)$ under $H_0$. \qed

\begin{remark}[The triangle statistic under the alternatives]\label{rem:trianglenormal}
The proof shows $Z_{K_3}-\E_\delta Z_{K_3}\to N(0,1)$ under $P_\delta$, which justifies the normal calibration used in Remark~\ref{rem:sharp} and in the experiments. Under the rank-two model the same argument applies conditionally on $(\sigma,\tau)$, because $\E[T_p\mid\sigma,\tau]=0$, the non-chaos terms have relative variance at most $6\eta r^2+6(n-1)\eta^2r\to0$ (Lemma~\ref{lem:rank2moments}(a)), and the conditional variances satisfy $|v_e-q|\le2\delta+4\delta^2\to0$. Hence $Z_{K_3}\to N(0,1)$ under $P^{(2)}_\delta$, which completes the proof of Proposition~\ref{prop:signcancel}(iv).
\end{remark}

\begin{remark}[Scope of (ii)]\label{rem:scopeii}
(a) The result holds for every fixed $t>0$, including $t>1$, where a single graph suffices to detect the planted partition by cycle statistics, so tree statistics fail to detect the difference even above the one-graph threshold. It requires $\eta=t/n\to0$. For fixed $\delta$ the edge count has the non-Gaussian limit $\sqrt{1-\eta}\,N(0,1)+\sqrt{\eta/2}\,(\chi^2_1-1)$ under $P_\delta$, although its mean and variance are exactly those of the null.

(b) The proof uses $p$ fixed only through the constant of Lemma~\ref{lem:moments}, which grows like $q^{-\sum_jk_je_j/2}$. Extending (ii) to sparse graphs ($q\to0$, $nq\to\infty$) requires a sharper count of non-paired tuples, as in \citet[arXiv version, Proposition~4.1]{banerjee2018contiguity} for signed cycles. We do not claim it.

(c) Part (ii) concerns fixed finite families of signed tree counts. Unsigned tree counts, degree moments and colour-refinement histograms are polynomials in signed forest counts, and in the dense regime their centred and normalized versions are dominated by the edge count. We expect, but do not prove, that every fixed finite family of such statistics also has the same limit under $H_0$ and $H_1$, and that the same holds for fixed-depth message-passing embeddings with Lipschitz readouts.
\end{remark}

\subsection{Theorem~\ref{thm:firstorder}: non-constant nulls and first-order visibility}\label{app:firstorder}

A graphon is a symmetric measurable $W:[0,1]^2\to[0,1]$. We realize $G(n,W)$ with independent uniform variables $U_1,\dots,U_n$ and $(V_{ij})_{i<j}$, setting $A_{ij}=\ind\{V_{ij}\le W(U_i,U_j)\}$. For a graph $F$ with $v$ vertices, $\tinj(F,G)=\ff nv^{-1}\sum_{\phi}\prod_{uw\in E(F)}A_{\phi(u)\phi(w)}$, the sum over injective $\phi:V(F)\to[n]$. For an edge $e=\{a,b\}\in E(F)$ (with a fixed orientation $(a,b)$) the \emph{edge-rooted kernel} is
\[
W_{F,e}(x,y)=\int_{[0,1]^{V(F)\setminus\{a,b\}}}\ \prod_{f\in E(F)\setminus\{e\}}W(x_f)\prod_{w\ne a,b}dx_w\ \Big|_{x_a=x,\ x_b=y},
\]
where $W(x_f)=W(x_u,x_w)$ for $f=\{u,w\}$. For symmetric $H$, $\langle H,W_{F,e}\rangle$ does not depend on the orientation. For $F=K_3$, $W_{K_3,e}(x,y)=\int W(x,z)W(z,y)\,dz$ for each edge. This kernel differs from the symmetrized two-point conditional kernel $W_H$ of \citet{chatterjee2026higher}, which includes the factor $W(x,y)$ of the rooted edge.

\begin{theorem*}[Restatement of Theorem~\ref{thm:firstorder}]
Let $F$ have $v$ vertices and $e\ge1$ edges.
\begin{itemize}
\item[(i)] For every graphon $W$ and $n\ge v$: $\E\,\tinj(F,G(n,W))=t(F,W)$ and $\Var\,\tinj(F,G(n,W))\le v^2/(4n)$.
\item[(ii)] For every graphon $W$ and symmetric measurable $H$ with $0\le W+H\le1$: $t(F,W+H)-t(F,W)=\sum_{e\in E(F)}\langle H,W_{F,e}\rangle+R$ with $|R|\le(2^e-1-e)\|H\|_2^2$.
\item[(iii)] With $D=\frac1m\sum_{\ell}\tinj(F,G_\ell)-\frac1m\sum_\ell\tinj(F,G'_\ell)$ and $\lambda=v/\sqrt{2nm\alpha}$, the test $\ind\{|D|\ge\lambda\}$ has size at most $\alpha$ under $P=Q=G(n,W)$ for every graphon $W$, and type II error at most $\beta$ under $P=G(n,W+H)$, $Q=G(n,W)$ whenever $|\Delta_F|\ge\frac{v}{\sqrt{2nm}}(\alpha^{-1/2}+\beta^{-1/2})$, in particular whenever $nm\Delta_F^2\ge v^2(1/\alpha+1/\beta)$; here $\Delta_F=t(F,W+H)-t(F,W)$. (The threshold $\sqrt{v^2/(nm\alpha)}$ is also valid; it requires $nm\Delta_F^2\ge v^2(2/\alpha+1/\beta)$.)
\item[(iv)] Let $H_0$ be symmetric and bounded with $0\le W+cH_0\le1$ for $c\in(0,c_0]$, and $\kappa_F=\sum_{e\in E(F)}\langle H_0,W_{F,e}\rangle\ne0$. For $0<c\le\min\{c_0,\,|\kappa_F|/(2(2^e-1-e)\|H_0\|_2^2)\}$, $|\Delta_F|\ge c|\kappa_F|/2$, and the test of (iii) has errors at most $(\alpha,\beta)$ against $W+cH_0$ as soon as $m\ge4v^2(1/\alpha+1/\beta)/(nc^2\kappa_F^2)$. With $t_H=n\|cH_0\|_2^2$, $m^*_{\mathrm{un}}\le1+4v^2(1/\alpha+1/\beta)\|H_0\|_2^2/(\kappa_F^2\,t_H)$. If $W$ is constant and $d_{H_0}\equiv0$, then $\kappa_F=0$ for every $F$.
\end{itemize}
\end{theorem*}

\begin{lemma}[Efron--Stein with vertex blocks]\label{lem:efronstein}
For every graphon $W$, graph $F$ with $v$ vertices and $n\ge v$: $\E\,\tinj(F,G(n,W))=t(F,W)$ and $\Var\,\tinj(F,G(n,W))\le v^2/(4n)$.
\end{lemma}

\begin{proof}
For injective $\phi$ the variables $U_{\phi(u)}$ are i.i.d.\ uniform, so $\E\prod_{uw}A_{\phi(u)\phi(w)}=\E\prod_{uw}W(U_{\phi(u)},U_{\phi(w)})=t(F,W)$. For the variance, let $X_j=(U_j,(V_{ij})_{i<j})$, $j=1,\dots,n$. The $X_j$ are independent and $G$ is a function of $(X_1,\dots,X_n)$. Moreover, $A_{ik}$ ($i<k$) depends only on $(U_i,U_k,V_{ik})$, hence only on $X_i$ and $X_k$. So when $X_j$ varies and the other blocks are fixed, only the terms of $\tinj$ with $j\in\phi(V(F))$ change. There are $v\ff{n-1}{v-1}=v\ff nv/n$ such $\phi$ and each term lies in $[0,1]$. Hence, as a function of $X_j$ alone, $\tinj$ ranges in an interval of length at most $v/n$ and its conditional variance given $(X_i)_{i\ne j}$ is at most $(v/n)^2/4$. The Efron--Stein inequality $\Var f(X)\le\sum_{j=1}^n\E\Var(f(X)\mid(X_i)_{i\ne j})$ \citep{efron1981jackknife,boucheron2013concentration} gives $\Var\,\tinj\le n\cdot v^2/(4n^2)$.
\end{proof}

\begin{lemma}[First-order expansion]\label{lem:expansion}
For every graphon $W$ and symmetric measurable $H$ with $0\le W+H\le1$,
$t(F,W+H)=t(F,W)+\sum_{e\in E(F)}\langle H,W_{F,e}\rangle+R$ with $|R|\le(2^{e(F)}-1-e(F))\|H\|_2^2$.
\end{lemma}

\begin{proof}
Note $|H|\le1$ since $W,W+H\in[0,1]$. Expanding the product over edges,
$t(F,W+H)=\sum_{S\subseteq E(F)}\int\prod_{f\in S}H(x_f)\prod_{f\notin S}W(x_f)\,dx$.
The term $S=\emptyset$ is $t(F,W)$, and the term $S=\{e\}$ is $\langle H,W_{F,e}\rangle$ by Fubini. For $|S|\ge2$ pick two edges $e_1\ne e_2$ in $S$ and bound all other factors by one in absolute value: the term is at most $\int|H(x_{e_1})||H(x_{e_2})|\,dx$. If $e_1,e_2$ share a vertex $a$ this is $\int\big(\int|H(x_a,y)|dy\big)^2dx_a\le\|H\|_2^2$ by Cauchy--Schwarz, and if they are disjoint it is $(\int|H|)^2\le\|H\|_2^2$. There are $2^e-1-e$ sets with $|S|\ge2$.
\end{proof}

\paragraph{Proof of Theorem~\ref{thm:firstorder}.}
(i) and (ii) are Lemmas~\ref{lem:efronstein} and \ref{lem:expansion}.

(iii) The $2m$ graphs are independent, so $\Var D=(\Var_P\tinj+\Var_Q\tinj)/m\le v^2/(2nm)$ by (i), for every $P,Q$ of graphon type. Under $P=Q$, $\E D=0$ and Chebyshev's inequality gives $\Prob(|D|\ge\lambda)\le v^2/(2nm\lambda^2)=\alpha$. Under the alternative $\E D=\Delta_F$ by (i). If $|\Delta_F|\ge\lambda+(v^2/(2nm\beta))^{1/2}$, then $\Prob(|D|<\lambda)\le\Prob(|D-\Delta_F|>|\Delta_F|-\lambda)\le\beta$. The sufficient condition follows from $(a+b)^2\le2a^2+2b^2$. The variant with threshold $\sqrt2\lambda$ is analogous.

(iv) By (ii), $|\Delta_F-c\kappa_F|\le(2^e-1-e)c^2\|H_0\|_2^2\le c|\kappa_F|/2$ under the stated restriction on $c$, so $|\Delta_F|\ge c|\kappa_F|/2$, and (iii) applies once $nm c^2\kappa_F^2/4\ge v^2(1/\alpha+1/\beta)$. If $W\equiv w$ is constant, $W_{F,e}\equiv w^{e-1}$ and $\langle H_0,W_{F,e}\rangle=w^{e-1}\int\!\!\int H_0=w^{e-1}\int d_{H_0}=0$. \qed

\begin{remark}[The E12 nulls]
For $F=K_3$, $\kappa_{K_3}=3\langle H_0,W\circ W\rangle$ with $W\circ W(x,y)=\int W(x,z)W(z,y)dz$. For the difference $H_0$ of the experiments ($D_{\mathrm{IN}}=1/4$ on $[0,1/4)^2$, $-D_{\mathrm{OUT}}=-1/12$ across, $D_{\mathrm{CC}}=1/36$ on $[1/4,1)^2$, zero degree function), the zero degree of $H_0$ kills every term of $W\circ W$ that depends on one variable only, and one finds the following values.
\begin{itemize}
\item $\kappa_{K_3}=0$ for $\mathrm{ER}(1/2)$.
\item For SBM2, $W\circ W=0.25+0.01\,s(x)s(y)$ and $\kappa_{K_3}=0.03(D_{\mathrm{IN}}+2D_{\mathrm{OUT}}+D_{\mathrm{CC}})/16=1/1200$.
\item For SMOOTH, $W\circ W=0.09+0.06(x+y)+\frac{0.16}3xy$ and $\kappa_{K_3}=0.16\,\langle H_0,xy\rangle=0.16/256=1/1600$.
\end{itemize}
For $c\in\{0.3,\dots,1\}$ the remainder of Lemma~\ref{lem:expansion} is well within its bound (numerical check (xi)). At $c=1$ the second-order terms are of the same size as $c\kappa_{K_3}$ for these nulls.
\end{remark}

\begin{remark}[Regular nulls and calibration]
The bound in (i) holds for every $W$ and is tight in order when $W$ is $F$-irregular. If $W$ is $F$-regular in the sense of \citet{bhattacharya2023fluctuations}, the count of $F$ fluctuates on the smaller scale $n^{v-1}$ (compared with $n^{v-1/2}$ in the irregular case), so that $\Var\,\tinj(F,G(n,W))=O(n^{-2})$, and its limit law may contain weighted $\chi^2$ components. \citet{chatterjee2026higher} give the joint limits and a multiplier bootstrap that is consistent in both regimes. The $W$-agnostic threshold of (iii) is then conservative by a factor of order $n$, and a two-sample test calibrated at the $n^{-1}$ scale would need only $m\,n^2\Delta_F^2\gtrsim1$. We do not prove the two-sample version.
\end{remark}

\begin{remark}[Scope of (iv)]
(iv) is an upper bound. A first-order visible difference is detected by an unaligned motif test with $m=O(1/t_H)$ graphs. This is the order reached by an aligned test under the same null (Proposition~\ref{prop:alignedW} below) and the order of the aligned sample size at a constant null (Theorem~\ref{thm:price}(iii), where $t=t_H/q$), compared with $m\asymp t^{-3}$ for the first-order invisible difference of Theorem~\ref{thm:price}. We do not prove a matching lower bound for non-constant nulls (the null is then itself a mixture over latent positions, and the second-moment method of Section~\ref{app:price} does not apply directly), nor an aligned lower bound for non-constant nulls. The oracle that knows all latent positions gives only $m^*\gtrsim1/(n\,t_H)$.
\end{remark}

\begin{proposition}[An aligned test under every null]\label{prop:alignedW}
Let $W$ be a graphon and $H$ symmetric with $0\le W+H\le1$, and suppose that all $2m$ graphs share the latent positions $U=(U_1,\dots,U_n)$. Given $U$, the edges $A^\ell_{ij}$ of $G_\ell$ are independent $\Bern((W+H)(U_i,U_j))$, and the edges $A'^\ell_{ij}$ of $G'_\ell$ are independent $\Bern(W(U_i,U_j))$, $\ell=1,\dots,m$, $m\ge2$. Let $D^\ell=A^\ell-A'^\ell$, $S=\sum_{i<j}\sum_{\ell\ne\ell'}D^\ell_{ij}D^{\ell'}_{ij}$, $t_H=n\|H\|_2^2$ and $r_H=\sup_x\int H(x,y)^2\,dy/\|H\|_2^2$.
\begin{itemize}
\item[(i)] If $H=0$, then $\E[S\mid U]=0$ and $\Var(S\mid U)\le Nm(m-1)/2$ for every $W$, so the test $\ind\{S\ge\lambda\}$ with $\lambda=(Nm(m-1)/(2\alpha))^{1/2}$ has size at most $\alpha$ under every null. The test that compares $S$ with its law under independent random signs of $D^1,\dots,D^m$ is exact.
\item[(ii)] If $(m-1)\,t_H\ge4\sqrt2\,(\alpha^{-1/2}+\beta^{-1/2})$ and $m(n-1)\,t_H\ge32/\beta$, the test of (i) has power at least $1-\beta-8/((n-1)t_H)-16r_H/n$.
\end{itemize}
In particular an aligned test needs $m=O(1/t_H)$ graphs per group under every null.
\end{proposition}

\begin{proof}
Given $U$, the variables $D^\ell_{ij}$ are independent over edges and over $\ell$. Fix an edge $e=\{i,j\}$ and write $D^\ell_e=\mu+Z^\ell$ with $\mu=H(U_i,U_j)$ and centred $Z^\ell$ of variance $s^2\le1/2$. Then
\[
Y_e:=\sum_{\ell\ne\ell'}D^\ell_eD^{\ell'}_e=m(m-1)\mu^2+2(m-1)\mu\sum_\ell Z^\ell+\sum_{\ell\ne\ell'}Z^\ell Z^{\ell'},
\]
and the last two terms are uncorrelated. Hence $\E Y_e=m(m-1)\mu^2$ and $\Var Y_e=4m(m-1)^2\mu^2s^2+2m(m-1)s^4\le2m(m-1)^2\mu^2+m(m-1)/2$. Summing over the $N$ edges, with $\Sigma_U=\sum_{i<j}H(U_i,U_j)^2$,
\[
\E[S\mid U]=m(m-1)\,\Sigma_U,\qquad\Var(S\mid U)\le2m(m-1)^2\,\Sigma_U+Nm(m-1)/2 .
\]

(i) For $H=0$ these are the stated moments, and Chebyshev's inequality gives the size. Given $U$ and $H=0$, $A^\ell$ and $A'^\ell$ are exchangeable, so the law of $(D^1,\dots,D^m)$ does not change when the $D^\ell$ are multiplied by independent random signs. This makes the sign-change test exact.

(ii) The variable $\Sigma_U$ has mean $N\|H\|_2^2$. Two edges without a common vertex give independent terms. For two edges with one common vertex the covariance is at most $\E[h(U_1)^2]\le r_H\|H\|_2^4$, where $h(x)=\int H(x,y)^2\,dy$, and $\Var H(U_1,U_2)^2\le\|H\|_2^2$ because $|H|\le1$. Hence $\Var\Sigma_U\le N\|H\|_2^2+2(n-2)N\,r_H\|H\|_2^4$, and Chebyshev's inequality gives
\[
\Prob\big(\Sigma_U<N\|H\|_2^2/2\big)\le\frac{8}{(n-1)\,t_H}+\frac{16\,r_H}{n}.
\]
On the complementary event, let $x=m(m-1)\Sigma_U\ge m(m-1)N\|H\|_2^2/2$. The second condition gives $x\ge8(m-1)/\beta$, so $(2(m-1)x/\beta)^{1/2}\le x/2$. Since $(m(m-1)N)^{1/2}\|H\|_2^2=t_H\big(m(m-1)(n-1)/(2n)\big)^{1/2}\ge(m-1)t_H/2$, the first condition gives $x/2\ge K:=(Nm(m-1)/2)^{1/2}(\alpha^{-1/2}+\beta^{-1/2})$. Therefore
\[
x-\lambda\ge\frac x2+\Big(\frac{Nm(m-1)}{2\beta}\Big)^{1/2}\ge\Big(\frac{2(m-1)x}{\beta}\Big)^{1/2}+\Big(\frac{Nm(m-1)}{2\beta}\Big)^{1/2}\ge\Big(\frac{\Var(S\mid U)}{\beta}\Big)^{1/2},
\]
and Chebyshev's inequality gives $\Prob(S<\lambda\mid U)\le\beta$. Adding the probability of the exceptional event proves (ii).
\end{proof}

\begin{proposition}[Differences orthogonal to a null of finite rank]\label{prop:orthogonal}
Let $W(x,y)=\sum_{k=1}^r\lambda_k\phi_k(x)\phi_k(y)$ be a graphon of finite rank, and let $H$ be bounded and symmetric with $0\le W+H\le1$, zero degree function and $\int H(x,y)\phi_k(y)\,dy=0$ for almost every $x$ and $k=1,\dots,r$. Then $t(P,W+H)=t(P,W)$ for every path $P$, and $t(C_k,W+H)=t(C_k,W)+t(C_k,H)$ for every cycle $C_k$, $k\ge3$. In particular the first-order coefficient $\sum_{e}\langle H,W_{F,e}\rangle$ vanishes for every path and every cycle $F$, and every cycle density moves by exactly $t(C_k,H)$, as under a constant null.
\end{proposition}

\begin{proof}
Let $T_W$ and $T_H$ be the integral operators of $W$ and $H$ on $L^2[0,1]$. The range of $T_W$ is spanned by $\phi_1,\dots,\phi_r$, so $T_HT_W=0$, and $T_WT_H=(T_HT_W)^*=0$ because both kernels are symmetric. The zero degree function gives $T_H1=d_H=0$.

For a path with $j$ edges, $t(P,W+H)=\langle1,(T_W+T_H)^j1\rangle$. In a product of the factors $T_W$ and $T_H$ that contains $T_H$, the rightmost $T_H$ acts either on $1$ or on the range of $T_W$, so the product annihilates $1$. Hence $(T_W+T_H)^j1=T_W^j1$.

For a cycle, $t(C_k,W+H)=\mathrm{tr}\,(T_W+T_H)^k$. A product of $k$ factors that contains both operators has, after a cyclic rotation that does not change its trace, a factor $T_HT_W$ or $T_WT_H$, so its trace is zero, and $\mathrm{tr}\,(T_W+T_H)^k=\mathrm{tr}\,T_W^k+\mathrm{tr}\,T_H^k$.

The hypotheses also hold for $cH$, $c\in[0,1]$, so $t(P,W+cH)$ does not depend on $c$ and $t(C_k,W+cH)=t(C_k,W)+c^kt(C_k,H)$. Both derivatives vanish at $c=0$.
\end{proof}

\begin{remark}[Degree-corrected nulls]\label{rem:dcbm}
A degree-corrected block model $W(x,y)=\theta(x)\theta(y)B_{z(x)z(y)}$ with $K$ blocks has rank at most $K$, and its eigenfunctions lie in the span of the functions $\theta\,\ind_{\{z=k\}}$. For a degree-preserving difference $H$ orthogonal to these functions, \cref{prop:orthogonal} shows that the first-order test of Theorem~\ref{thm:firstorder} has no signal for any path or cycle, and that the triangle density moves by exactly $t(K_3,H)$, as in Theorem~\ref{thm:price}. We therefore expect the $t^{-3}$ rate to persist for such differences. The lower bound would need the second-moment method for a null that is a mixture over latent positions, which we have not carried out. A degree-preserving difference that is not orthogonal to the eigenfunctions is generically visible at first order, through the path with three edges and through the triangle, and \cref{cor:visibility} applies. Trees in which two adjacent vertices both have degree at least three can keep a first-order term in the orthogonal case. For the double star with two leaves at each centre and $H=\delta\,s\otimes s$ it equals $\delta\langle s,d_W^2\rangle^2$, which is at most $10^{-4}\delta$ for the nulls of E15 (\cref{app:e15}).
\end{remark}

\subsection{Proposition~\ref{prop:split}: exact calibration by vertex splitting}\label{app:split}

\begin{proposition*}[Restatement of Proposition~\ref{prop:split}]
Let $V_1,\dots,V_B\subseteq[n]$ be fixed, pairwise disjoint, of equal size $n_B$.
\begin{itemize}
\item[(i)] Under $G(n,W)$ with a fixed graphon $W$, the induced subgraphs $G[V_1],\dots,G[V_B]$ (relabelled increasingly by $[n_B]$) are independent with law $G(n_B,W)$.
\item[(ii)] Let $Y_1,\dots,Y_K$, $K=2mB$, be the induced subgraphs of the $2m$ graphs, $I_0\subseteq[K]$ the $mB$ indices from group one, and $S(y;I)$ a statistic satisfying $S(y\circ\pi;I)=S(y;\pi(I))$ for all permutations $\pi$ of $[K]$ (e.g.\ a difference of group means of a graph functional). With $\hat p=\binom K{mB}^{-1}\sum_{|I|=mB}\ind\{S(Y;I)\ge S(Y;I_0)\}$, $\Prob_{H_0}(\hat p\le\alpha)\le\alpha$ for all $\alpha\in[0,1]$ under every null $P=Q=G(n,W)$; the randomized p-value is exactly uniform. Since $\hat p\ge\binom K{mB}^{-1}$, a rejection at level $\alpha$ requires $\binom K{mB}\ge1/\alpha$.
\item[(iii)] Under the planted partition, $\sum_{b}T_p(G[V_b])$ has mean $B\binom{n_B}3\delta^3$ and null variance $B\binom{n_B}3q^3$; its squared signal-to-noise ratio is $B\binom{n_B}3\eta^3=\frac{(n_B-1)(n_B-2)}{6n^2}t^3\approx t^3/(6B^2)$.
\item[(iv)] If the common null law is a mixture of graphon laws (a graphon drawn at random for each graph), the $Y_i$ need not be exchangeable and the test in (ii) need not be valid.
\end{itemize}
\end{proposition*}

\begin{lemma}[Induced subgraphs on disjoint vertex sets]\label{lem:induced}
Statement (i) holds.
\end{lemma}

\begin{proof}
$G[V_b]$ is a function of $(U_i)_{i\in V_b}$ and $(V_{ij})_{i<j,\,i,j\in V_b}$. These collections are disjoint for different $b$, hence independent. The law of $G[V_b]$ is that of $G(n_B,W)$ by definition.
\end{proof}

\begin{lemma}[Validity of permutation p-values]\label{lem:permutation}
Let $Y=(Y_1,\dots,Y_K)$ be exchangeable, $1\le k<K$, $I_0=\{1,\dots,k\}$, and $S$ as in (ii). Then $\hat p=\binom Kk^{-1}\sum_{|I|=k}\ind\{S(Y;I)\ge S(Y;I_0)\}$ satisfies $\Prob(\hat p\le\alpha)\le\alpha$. With $\zeta\sim U(0,1)$ independent of $Y$, the randomized p-value $\tilde p=\binom Kk^{-1}\big(|\{I:S(Y;I)>S(Y;I_0)\}|+\zeta\,|\{I:S(Y;I)=S(Y;I_0)\}|\big)$ is uniform on $[0,1]$.
\end{lemma}

\begin{proof}
Let $J$ be a uniform $k$-subset of $[K]$ independent of $Y$ and $\pi_J$ a permutation with $\pi_J(I_0)=J$. By exchangeability $Y\circ\pi_J$ has the law of $Y$, so $\hat p(Y)$ has the law of $\hat p(Y\circ\pi_J)=\binom Kk^{-1}\sum_{I}\ind\{S(Y;\pi_J(I))\ge S(Y;J)\}=\binom Kk^{-1}|\{I':S(Y;I')\ge S(Y;J)\}|$, because $\pi_J$ permutes the $k$-subsets.

Given $Y$, write $s_I=S(Y;I)$. For any finite list of reals, the number of indices $J$ with $|\{I':s_{I'}\ge s_J\}|\le a$ is at most $a$ (if $J^*$ has the smallest value among them, all of them satisfy $s_J\ge s_{J^*}$). Hence $\Prob(\hat p\le\alpha\mid Y)\le\lfloor\alpha\binom Kk\rfloor/\binom Kk\le\alpha$.

For $\tilde p$ the same substitution shows that, given $Y$, $\tilde p$ has the law of $\binom Kk^{-1}(|\{s_{I'}>s_J\}|+\zeta|\{s_{I'}=s_J\}|)$ with $J$ uniform, which is uniform on $[0,1]$ (the ranks of $s_J$ among the $\binom Kk$ values, randomized within ties).
\end{proof}

\paragraph{Proof of Proposition~\ref{prop:split}.}
(i) is Lemma~\ref{lem:induced}.

(ii) Under $P=Q=G(n,W)$ the $2m$ graphs are i.i.d. By (i) the $2mB$ induced subgraphs are i.i.d., hence exchangeable, and Lemma~\ref{lem:permutation} applies with $k=mB$. Also $\hat p\ge\binom K{mB}^{-1}$, because $I=I_0$ is counted.

(iii) Given $\sigma$, the edges inside $V_b$ are independent $\Bern(p+\delta\sigma_i\sigma_j)$ with $(\sigma_i)_{i\in V_b}$ i.i.d.\ uniform. So each $G[V_b]$ is a planted partition on $n_B$ vertices with the same $(p,\delta)$, and the $B$ blocks are independent. Lemma~\ref{lem:trianglemoments} with $n_B$ in place of $n$ gives the moments, and finally $B\binom{n_B}3\eta^3=\frac{n(n_B-1)(n_B-2)}6\cdot\frac{t^3}{n^3}$.

(iv) Take $P=Q$ the law of $G(n,p_\Theta)$ with $\Theta$ drawn afresh for each graph and $\Var p_\Theta>0$. The edge densities of two blocks of the same graph have covariance $\Var(p_\Theta)>0$, while blocks of different graphs are independent, so $(Y_1,\dots,Y_K)$ is not exchangeable. In a simulation with $p_\Theta\sim U[0.3,0.7]$, $m=1$, $B=4$, $n=32$ and the edge-density statistic (numerical check (xii)), the level-$0.05$ test rejects with probability $0.22$. Permuting whole graphs remains valid. \qed

\subsection{What is not proved}\label{app:notproved}

\begin{itemize}
\item \emph{Sharp constants.} Theorem~\ref{thm:price}(iv) holds with explicit but loose constants. The sharp aligned threshold $mt=1$ is cited from \citet{chatterjee2022multilayer} (sparse regime) and not reproved, and in the dense regime we only prove $0.87/t\lesssim m^*_{\mathrm{al}}\lesssim C/t$. The asymptotically sharp unaligned constant (Remark~\ref{rem:sharp}) relies on the likelihood-ratio limit of \citet{banerjee2018contiguity}.
\item \emph{General finite-rank differences.} Only the rank-one and the specific rank-two model are treated. The formula $\log(1+\chi^2_1)\approx\sum_{j\ge3}(\sum_i\nu_i^j)^2/(2j)$ for a general zero-degree finite-rank $H$ with normalized eigenvalues $\nu_i$ is consistent with both cases but is not proved. The Veblen argument needs the positivity of the coefficients $t_S$, which fails in general.
\item \emph{Sparse graphs.} Theorem~\ref{thm:price}(i)--(iii) and Proposition~\ref{prop:signcancel}(i)--(iii) are non-asymptotic and hold for all $p$. The distributional statements (Theorem~\ref{thm:trees}(ii), Proposition~\ref{prop:signcancel}(iv) central limit part) are proved for fixed $p$ only.
\item \emph{Beyond fixed families of signed tree counts.} The distributional statement for unsigned tree counts, for degree and colour-refinement histograms with growing support, and for message-passing embeddings at finite $n$ remains a conjecture (Remark~\ref{rem:scopeii}). Only the graphon-level statement (Remark~\ref{rem:treesscope}(c)) follows from Theorem~\ref{thm:trees}(i) and the cited characterization.
\item \emph{Non-constant nulls.} There is no lower bound for non-constant $W$, aligned or unaligned. For differences that are visible at first order we prove upper bounds of the same order $1/t_H$ for unaligned and aligned tests (Theorem~\ref{thm:firstorder}(iv) and Proposition~\ref{prop:alignedW}). For degree-preserving differences orthogonal to a null of finite rank, the $t^{-3}$ rate is a conjecture (Remark~\ref{rem:dcbm}). The improved $n^{-2}$ variance for $F$-regular nulls is cited.
\item \emph{Power of the vertex-split test.} Proposition~\ref{prop:split} proves validity and computes the signal-to-noise ratio of the within-block statistic. We do not analyse the power of the permutation test itself at small $m$.
\end{itemize}

\section{Experimental details}\label{app:experiments}
\subsection{Common settings}
All experiments ran with NumPy and SciPy on one laptop CPU (10 cores, 16\,GB memory), with four worker processes for E11--E14, seven for E15 and three for E16. All data are simulated, except the protein graphs of E16. Wall-clock times were about $24$ minutes for E11 and E11b together (E11b reuses the graph pools of E11), $11$ minutes for E12 at $n=128$ and $66$ minutes at $n=256$, and $19$ minutes for the sparse panel. E13, E14 and E14b took $2$, $5$ and $2.5$ minutes, E15 took $14$ minutes and E16 took $8$ minutes. The numerical checks of the proofs take about $2$ minutes in total. Seeds are fixed and recorded with the results. For every reported rejection rate $\hat r$ based on $R$ replicates, the Monte Carlo standard error is $(\hat r(1-\hat r)/R)^{1/2}$, which is at most $0.016$ for $R=1000$ and $0.035$ for $R=200$. The sample size $m^\ast$ is the smallest $m$ at which the power reaches $0.5$. It is obtained by log-linear interpolation of the running maximum of the power curve on the grid of $m$ values, and is reported as $>m_{\max}$ when the power never reaches $0.5$ on the grid.

\subsection{Sample complexity, sign cancellation and visibility (E11, E11b, E12)}\label{app:e11}
\paragraph{E11 (sample sizes with and without alignment).}
\begin{itemize}
\item \emph{Model.} The null is $\mathbb G(300,1/2)$, and the alternative is the planted partition with $\delta=(tq/n)^{1/2}$ and $t\in\{0.3,0.4,\dots,0.9\}$. All tests are one-sided at level $0.05$, because every signal is positive.
\item \emph{Unaligned tests} (two-sample, $m$ graphs per group):
  \begin{itemize}
  \item signed triangles with known $p$ and the exact null variance $\binom n3q^3$;
  \item signed triangles with the pooled edge density;
  \item the per-graph excess triangle density $6T/n^3-(2E/n^2)^3$.
  \end{itemize}
  The last two are calibrated by simulated nulls.
\item \emph{Aligned tests} (one $\sigma$ shared by the first sample, known $p$): the statistic $Q$ with its exact null variance, and the top eigenvalue of $(mnq)^{-1/2}\sum_\ell(A^\ell-p(J-I))$ calibrated by simulation.
\item \emph{Pool bootstrap.} Unaligned statistics are sums of independent per-graph terms. Their sampling distribution at each $m$ is approximated by resampling $m$ per-graph records with replacement from pools of $5\cdot10^5$ simulated graphs per model, with $4000$ resamples per $m$. Pools of different models share random numbers. Each pool is still an independent sample from its own model, and sharing reduces the variance of the Monte Carlo error in the pool-mean difference by a factor of $6$ to $11$.
\item \emph{Aligned simulation.} Aligned tests are simulated replicate by replicate, with $1000$ replicates per $t$. Thresholds use $2000$ null replicates and sizes use $1000$ held-out ones.
\item \emph{Checks.} All sizes lie in $[0.035,0.066]$. The exact null variances of $T_p$ and $\Csig$ are confirmed to within $0.3\%$, and every simulated mean shift agrees with its closed form to within $2.1$ standard errors.
\item \emph{Error bars.} Monte Carlo standard errors of $\log m^\ast$ come from five batches and are a few per cent. Slopes are least-squares fits with their standard errors.
\end{itemize}
The sparse panel repeats E11 with $n=2000$ and $np=15$ (\cref{fig:e11sparse}).

\paragraph{E11b (sign cancellation).} The setting is that of E11, with the rank-two model \eqref{eq:rank2} and the same values of $t$. Tests are two-sample $z$-tests with exact null variances for signed triangles and signed $4$-cycles. $\Csig$ is computed from $\tr B^4$, $\sum_i((B^2)_{ii})^2$ and $\sum_{ij}B_{ij}^4$, where $B=A-p(J-I)$. \Cref{fig:e11b} shows the results. Under the rank-one model, $4$-cycles also detect the difference, but they need more graphs than triangles (slope $\ElevenbSlopeCfourRone$).

\paragraph{E12 (order of visibility).}
\begin{itemize}
\item \emph{Model.} The nulls are ER$(1/2)$, SBM2 and SMOOTH, and the alternative is $W+cH_0$ with $H_0$ as in E13 and $c\in\{0.1,0.15,0.2,0.3,0.4,0.5,0.7,1\}$. We use $n=128$ and repeat the experiment at $n=256$ (\cref{fig:e12n256}).
\item \emph{Unaligned test.} The two-sample excess-triangle test, calibrated by simulated nulls from pools of $5\cdot10^5$ graphs per null.
\item \emph{Aligned tests.} One latent vector is shared by all $2m$ graphs of a replicate. The Frobenius statistic is $\sum_{i<j}\sum_{\ell\neq\ell'}D^\ell_{ij}D^{\ell'}_{ij}$ with $D^\ell=A^\ell-A'^\ell$. The spectral statistic is $\|\bar A-\bar A'\|_{\mathrm{op}}$. Both are calibrated by simulated nulls with fresh latent vectors.
\item \emph{Coefficients.} First-order coefficients and exact shifts of $t(K_3,\cdot)$ are computed by Gauss--Legendre quadrature, which is exact for these piecewise bilinear graphons.
\end{itemize}

\begin{figure}[h]
\centering
\includegraphics[width=0.95\linewidth]{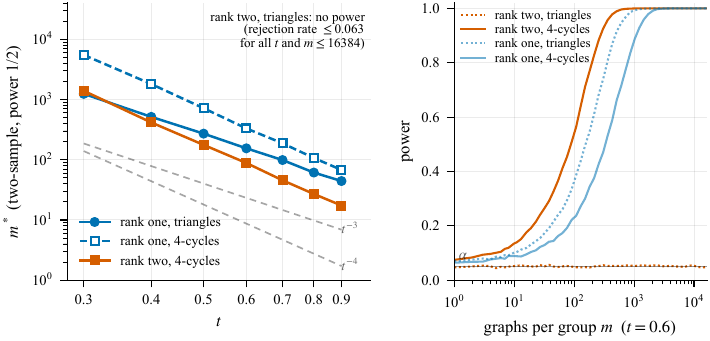}
\caption{Sign cancellation (E11b), $n=300$, $p=1/2$. Left: $m^\ast$ against $t$ for signed triangles and signed $4$-cycles under the rank-one planted partition, and for $4$-cycles under the rank-two model \eqref{eq:rank2}, where triangles never reach power $0.5$. Right: power against $m$ at $t=0.6$.}
\label{fig:e11b}
\end{figure}

\begin{figure}[h]
\centering
\includegraphics[width=0.95\linewidth]{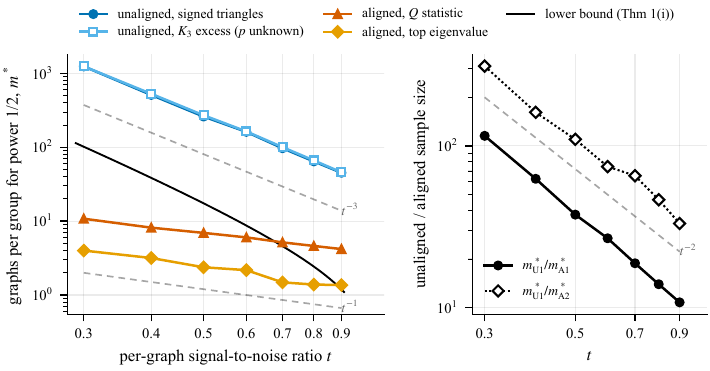}
\caption{E11 in the sparse regime ($n=2000$, average degree $15$). The layout is as in \cref{fig:price}.}
\label{fig:e11sparse}
\end{figure}

\begin{figure}[h]
\centering
\includegraphics[width=0.95\linewidth]{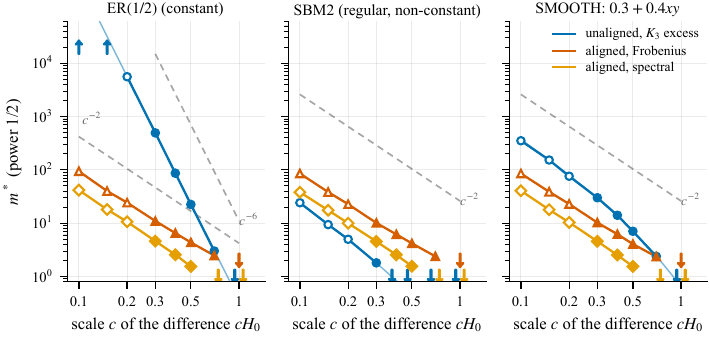}
\caption{E12 repeated with $n=256$. The layout is as in \cref{fig:visibility}.}
\label{fig:e12n256}
\end{figure}

\subsection{Calibration with few graphs (E13)}\label{app:e13}
\paragraph{Models.} The nulls are ER$(1/2)$, SBM2 with $W=0.6$ within the two halves of $[0,1]$ and $0.4$ across (regular, non-constant), and SMOOTH with $W(x,y)=0.3+0.4xy$ (irregular). The alternative adds $H_0$ to the null of the first group, where $H_0=0.25$ on $[0,\tfrac14)^2$, $-\tfrac1{12}$ on $[0,\tfrac14)\times[\tfrac14,1]$ and its reflection, and $\tfrac1{36}$ on $[\tfrac14,1]^2$. Its degree function is zero. We use $n\in\{64,128\}$ and $m\in\{1,2,3,5,10\}$, with 1000 replicates for size and 500 for power.

\paragraph{Statistic.} The per-graph statistic is $\Kexc(G)=6T/n^3-(2E/n^2)^3$ (with $T$ triangles and $E$ edges), the negative of the $T_3$ statistic of \citet{gao2017testing}. The test statistic is the difference of the group means, and all tests are two-sided at level $0.05$.

\paragraph{Methods.}
\begin{itemize}
\item \emph{Graph-label permutation}: all $\binom{2m}{m}$ splits, enumerated.
\item \emph{Welch}: normal and Student-$t$ (Satterthwaite) calibration.
\item \emph{Vertex jackknife}: leave-one-vertex-out, closed form, normal calibration.
\item \emph{Vertex split} (\cref{prop:split}): contiguous blocks, $B\in\{2,4\}$, enumeration when there are at most $5000$ splits and $999$ random splits otherwise.
\item \emph{Rate-extrapolated variance}: block variances at sizes $n/4$ and $n/2$, a fitted exponent clipped to $[1,3]$, and extrapolation to size $n$ \citep[in the spirit of][]{bertail1999subsampling}.
\item \emph{ER plug-in}: the exact variance of $\Kexc$ under $\mathbb G(n,\hat p)$, derived from its orthogonal expansion in $A_e-p$.
\item \emph{CDB}: the graphon multiplier bootstrap of \citet{chatterjee2026higher} for $t(K_2)$ and $t(K_3)$. It is applied to each graph with independent Gaussian multipliers and combined by the delta method, with $999$ draws. The regular or irregular form is chosen by their test $\sqrt n\,\hat R>1$.
\end{itemize}

\Cref{tab:calibration} summarizes the case $n=128$, $m\le3$. When the difference is visible at first order (SBM2 null), two conservative variance estimates keep useful power, namely the jackknife and a two-sample version of the multiplier bootstrap of \citet{chatterjee2026higher} (CDB). A parametric Erd\H{o}s--R\'enyi calibration has more power under the Erd\H{o}s--R\'enyi null but size up to $0.49$ under SMOOTH, and the Student-$t$ version of the Welch test is valid.

\begin{table}[h]
\centering
\caption{Calibration of the two-sample excess-triangle test with $m\le3$ graphs per group ($n=128$, two-sided level $0.05$). Size is the largest rejection rate under the three nulls ER$(1/2)$, SBM2 and SMOOTH (1000 replicates, Monte Carlo s.e.\ $0.007$ at $0.05$). Power is under $W+H_0$ (500 replicates). Parentheses mark tests whose size in that cell exceeds $0.064$. \Cref{fig:e13full} shows all $n\in\{64,128\}$ and $m\le10$.}
\label{tab:calibration}
\small
\resizebox{\linewidth}{!}{\begin{tabular}{l ccc ccc ccc}
\toprule
 & \multicolumn{3}{c}{Size (worst null)} & \multicolumn{3}{c}{Power, SBM2 null} & \multicolumn{3}{c}{Power, ER null} \\
\cmidrule(lr){2-4}\cmidrule(lr){5-7}\cmidrule(lr){8-10}
Method & $m=1$ & $2$ & $3$ & $1$ & $2$ & $3$ & $1$ & $2$ & $3$ \\
\midrule
Graph-label permutation (exact) & 0.000 & 0.000 & 0.000 & 0.00 & 0.00 & 0.00 & 0.00 & 0.00 & 0.00 \\
Welch $z$ & -- & 0.206 & 0.127 & -- & (0.85) & (0.96) & -- & (0.40) & (0.40) \\
Welch $t$ & -- & 0.024 & 0.040 & -- & 0.22 & 0.69 & -- & 0.06 & 0.15 \\
Vertex jackknife & 0.010 & 0.009 & 0.014 & 0.21 & 0.58 & 0.89 & 0.00 & 0.02 & 0.04 \\
Vertex split, $B=4$ (exact) & 0.031 & 0.055 & 0.062 & 0.06 & 0.18 & 0.25 & 0.02 & 0.07 & 0.06 \\
CDB bootstrap & 0.018 & 0.024 & 0.035 & 0.21 & 0.54 & 0.87 & 0.00 & 0.01 & 0.01 \\
ER plug-in (valid under ER only) & 0.482 & 0.488 & 0.458 & (0.82) & (0.96) & (1.00) & 0.18 & 0.30 & 0.41 \\
\bottomrule
\end{tabular}
}
\end{table}

\paragraph{Two caveats about CDB.} First, $\Kexc$ lies outside the theory of \citet{chatterjee2026higher} under an Erd\H{o}s--R\'enyi null. There its $n^{-1}$ components cancel, and the bootstrap overestimates the variance of the test statistic by a factor of $4.3$--$5.1$, so the test is valid but conservative. Second, at $n\le128$ the regularity test never flagged a graph as irregular, even under the irregular SMOOTH null. The bootstrap applied to $t(K_3)$ alone is then anti-conservative under SMOOTH (size $0.18$--$0.36$). With the correct regime its size is $0.023$--$0.039$. This caveat concerns the pre-test at finite $n$, and the asymptotic theorem is unaffected. \Cref{fig:e13full} shows size and power for all methods, nulls, $n$ and $m$.

\begin{figure}[h]
\centering
\includegraphics[width=\linewidth]{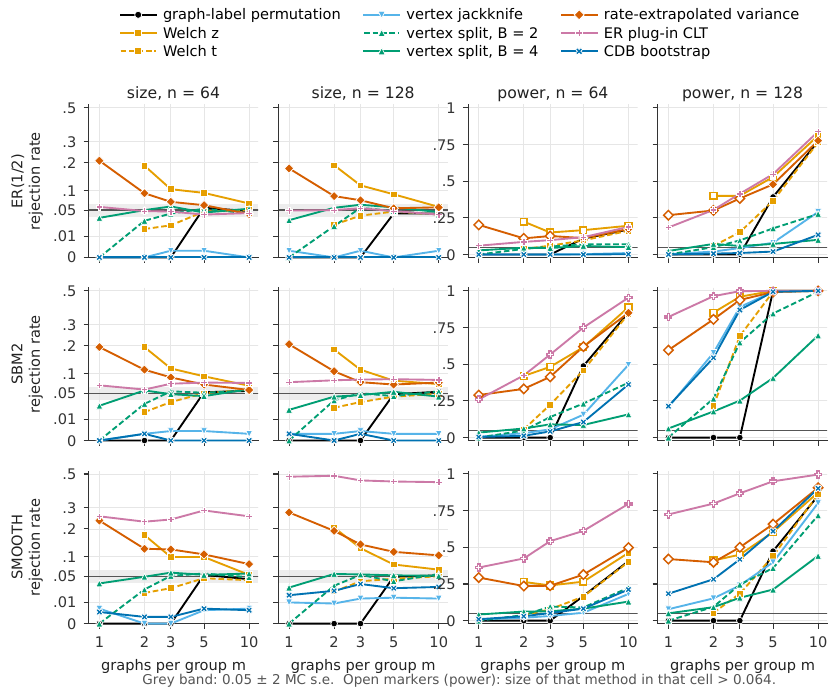}
\caption{Size (left two columns, on a square-root axis with a grey band at $0.05\pm2$ Monte Carlo s.e.) and power under $W+H_0$ (right two columns) of the calibration methods, for the three nulls (rows), $n\in\{64,128\}$ and $m\in\{1,2,3,5,10\}$. Open markers in the power panels flag methods whose size in that cell exceeds $0.064$.}
\label{fig:e13full}
\end{figure}

\subsection{Graph evaluation metrics (E14, E14b)}\label{app:e14}
\paragraph{Models.} In E14 the null is $\mathbb G(200,0.1)$, and the alternative is the planted partition with $t\in\{0.9,1.5\}$ ($\delta=0.0201$ and $0.0260$). In E14b, group 1 is the planted partition with $t=2$ and group 2 consists of independent planted-partition graphs, each rewired by $10|E|$ successful degree-preserving double-edge swaps (swaps creating loops or multiple edges are rejected). Fresh graphs are drawn in every replicate. We use 200 replicates per point, and sizes are checked at $m\in\{8,64\}$.

\paragraph{Descriptors.}
\begin{itemize}
\item Degree histogram (degrees $0$--$60$).
\item Local clustering histogram ($50$ bins).
\item Counts of edges, $2$-paths, triangles, $3$-stars, $3$-paths and $4$-cycles, each divided by its expectation under $\mathbb G(n,\hat p)$.
\item Normalized-Laplacian eigenvalue histogram ($100$ bins on $[0,2]$).
\item A random GIN \citep{xu2019powerful} with three layers of width $32$, constant input and sum readout of all layers. Without biases, a ReLU GIN with constant input is positively homogeneous, and its features are exact linear functions of three walk counts. We therefore add fixed $N(0,1)$ biases and scale the input by $1/(1+(n-1)p)$, in the spirit of \citet{thompson2022evaluation}.
\end{itemize}
The MMD uses the unbiased estimator with a Gaussian kernel whose bandwidth is the median pairwise distance of the pooled sample, and is calibrated by $200$ label permutations. Signed triangles with the pooled edge density use the same permutations.

\paragraph{Positive controls.} Every descriptor rejects in $20/20$ replicates against a $6\%$ change of edge density with $m=32$. Against a degree-heterogeneous graphon $0.1\,\theta(x)\theta(y)$ with $\theta(x)=1+0.3(2x-1)$ and the same density, the degree histogram and the random GIN have power \PosCtrlDeg{} and \PosCtrlGin{} with $m=16$, and signed triangles have power \PosCtrlTri. The low power in \cref{fig:metrics} is therefore caused by the difference itself, since the same embeddings detect a degree-heterogeneous difference.

\subsection{Evaluation metrics under irregular nulls (E15)}\label{app:e15}
\paragraph{Models.} We use three nulls with edge density $0.4$ on $n=200$ vertices, namely ER with $W\equiv0.4$, SMOOTH with $W(x,y)=0.3+0.4xy$ (the irregular null of E12) and the degree-heterogeneous rank-one graphon $W(x,y)=0.4\,\theta(x)\theta(y)$ with $\theta(x)=0.5+x$, which we call the rank-one null. Their degree functions are $0.4$, $0.3+0.2x$ and $0.4\,\theta(x)$. The alternative adds $H=\delta\,s\otimes s$ with $\delta=\EfifteenDelta$, where either $s_A=\ind_{[1/2,1)}-\ind_{[0,1/2)}$, which is correlated with the degree profile of the two irregular nulls, or $s_O=\ind_{[0,1/4)\cup[3/4,1)}-\ind_{[1/4,3/4)}$, which is orthogonal to $1$ and $x$ and hence to every eigenfunction of the three nulls. Both differences have zero degree function, and under the Erd\H{o}s--R\'enyi null both are the planted partition with $t=4.1$. For $H_A$, exact quadrature gives first-order coefficients $\kappa_{P_4}=2.5\cdot10^{-3}$ under SMOOTH and $10^{-2}$ under the rank-one null, and $\kappa_{K_3}=10^{-2}$ and $3.25\cdot10^{-2}$. For $H_O$, every path density is unchanged and the triangle and $4$-cycle densities move by exactly $\delta^3$ and $\delta^4$ under all three nulls (\cref{prop:orthogonal}).

\paragraph{Descriptors and tests.} The descriptors and tests are those of E14 (\cref{app:e14}), with one histogram bin per degree and the GIN input scaled by $1/(1+0.4(n-1))$. We add the degree-corrected signed triangle and $4$-cycle counts per wedge. With $B=A-dd^\top/(2|E|)$, set to zero on the diagonal, and $P_2=\sum_i\binom{d_i}2$ the number of paths of length two, they are $\mathrm{tr}(B^3)/(6P_2)$ and $\big(\mathrm{tr}(B^4)-2\sum_i(B^2)_{ii}^2+\sum_{i,j}B_{ij}^4\big)/(8P_2)$, compared by the permutation test of the mean difference. We use $m\in\{8,16,\dots,256\}$, $200$ replicates per point, and sizes at $m\in\{8,64\}$, which lie between \EfifteenSizeMin{} and \EfifteenSizeMax{} for all descriptors and nulls. The value of $\delta$ was chosen in a pilot run with $\delta\in\{0.04,0.07,0.1\}$.

\paragraph{Results.} \Cref{tab:e15} gives $m^\ast$, and the top row of \cref{fig:e15e16} shows power curves. The random GIN detects $H_A$ under both irregular nulls, with $m^\ast=\EfifteenGinMstarSmoothHA$ under SMOOTH and \EfifteenGinMstarDcHA{} under the rank-one null, through the first-order term of the path with three edges. It stays at the level for $H_O$ under every null and for $H_A$ under the Erd\H{o}s--R\'enyi null, with rejection rates between \EfifteenGinOrthMin{} and \EfifteenGinOrthMax{} over the \EfifteenGinOrthCells{} cells up to $m=256$ (Monte Carlo s.e.\ $0.015$ at $0.05$). The degree histogram has the same expected value under both hypotheses for every null, and its rejection rates are at most \EfifteenDegMax. The largest of them, at $m=256$ under SMOOTH with $H_A$, reflects the joint law of the degrees, which changes at first order when the difference is correlated with the degree profile. Degree-corrected signed triangles and $4$-cycles detect every alternative, with power at least \EfifteenDcPowEightMin{} already at $m=8$. Signed triangles centred at the pooled edge density, the statistic of \cref{thm:price}, lose power under the rank-one null, where they need $m^\ast=\EfifteenTriPooledDcHO$ graphs for $H_O$, because the degree heterogeneity of the null enters their variance at first order.

\begin{table}[h]
\centering
\caption{E15: graphs per group needed for power $0.5$ ($m^\ast$, interpolated on $m\in\{8,\dots,256\}$, with ${\le}8$ when reached at $m=8$ and ${>}256$ when not reached), and the largest rejection rate under the three nulls ($m\in\{8,64\}$, $200$ replicates each).}
\label{tab:e15}
\small
\begin{tabular}{l cc cc cc c}
\toprule
 & \multicolumn{2}{c}{ER} & \multicolumn{2}{c}{SMOOTH} & \multicolumn{2}{c}{rank-one} & Size \\
\cmidrule(lr){2-3}\cmidrule(lr){4-5}\cmidrule(lr){6-7}\cmidrule(lr){8-8}
Descriptor & $H_A$ & $H_O$ & $H_A$ & $H_O$ & $H_A$ & $H_O$ & max \\
\midrule
Degree histogram (MMD) & \ensuremath{>256} & \ensuremath{>256} & \ensuremath{>256} & \ensuremath{>256} & \ensuremath{>256} & \ensuremath{>256} & 0.070 \\
Random GIN (MMD) & \ensuremath{>256} & \ensuremath{>256} & 44.3 & \ensuremath{>256} & 25.7 & \ensuremath{>256} & 0.085 \\
Clustering histogram (MMD) & 28.0 & 38.7 & \ensuremath{\le8} & 191 & \ensuremath{\le8} & \ensuremath{>256} & 0.080 \\
Graphlet counts (MMD) & \ensuremath{>256} & \ensuremath{>256} & 176 & \ensuremath{>256} & 124 & \ensuremath{>256} & 0.085 \\
Laplacian spectrum (MMD) & 14.5 & 13.5 & 11.7 & 16.7 & 8.9 & 25.8 & 0.090 \\
Signed triangles, pooled $\hat p$ & \ensuremath{\le8} & \ensuremath{\le8} & \ensuremath{\le8} & \ensuremath{\le8} & \ensuremath{\le8} & 37.6 & 0.070 \\
Signed triangles, degree-corrected & \ensuremath{\le8} & \ensuremath{\le8} & \ensuremath{\le8} & \ensuremath{\le8} & \ensuremath{\le8} & \ensuremath{\le8} & 0.085 \\
Signed $4$-cycles, degree-corrected & \ensuremath{\le8} & \ensuremath{\le8} & \ensuremath{\le8} & \ensuremath{\le8} & \ensuremath{\le8} & \ensuremath{\le8} & 0.080 \\
\bottomrule
\end{tabular}

\end{table}

\subsection{Protein graphs against rewired copies (E16)}\label{app:e16}
\paragraph{Data.} The TUDataset collection \citep{morris2020tudataset} provides the D\&D protein structures \citep{dobson2003distinguishing}, in which two amino acids are joined when they are less than $6$\,\AA{} apart, and the ENZYMES tertiary structures \citep{borgwardt2005protein}. As in the Proteins benchmark of graph generative models \citep{you2018graphrnn,liao2019efficient} we keep the \EsixteenGraphsDd{} D\&D graphs with $100$ to $500$ vertices, and as in \citet{jo2022score} the \EsixteenGraphsEnz{} ENZYMES graphs with $10$ to $125$ vertices after removing isolated vertices. The graphs of a dataset differ in size and shape, and both groups are drawn from the same dataset.

\paragraph{Design.} Each replicate draws $2m$ distinct graphs at random. The first $m$ form group 1, and the other $m$ form group 2 after $\EsixteenRho\%$ of their edges have been rewired by successful degree-preserving double-edge swaps. The three kinds of swaps differ in the pairs of edges they allow. Unrestricted swaps (deg) pick any two edges. Swaps of type wl1 pick two edges whose endpoints have equal degrees, and they preserve the colour of every vertex after two rounds of colour refinement. Swaps of type wl2 pick two edges whose endpoints have equal colours after two rounds, and they preserve the colours after three rounds. Sequences of such swaps preserve these colours, so under wl2 the three-layer GIN returns exactly the embedding of the original graph, which we checked for every rewired graph. Without rewiring the $2m$ graphs are exchangeable and all tests are exact. In \EsixteenSizeReps{} replicates at $m\in\{8,64\}$ their sizes lie between \EsixteenSizeMin{} and \EsixteenSizeMax. The descriptors are those of E14, adapted to graphs of different sizes. We add a GIN with mean readout, and centre the first signed triangle statistic at each graph's own edge density. With $m=2$ no two-sided permutation test can reject at level $0.05$, since the smallest attainable $p$-value is $1/3$, so we report $m\ge4$.

\paragraph{Results.} \Cref{tab:e16} gives $m^\ast$, and the bottom row of \cref{fig:e15e16} shows power curves. Unrestricted rewiring removes \EsixteenTriLossDdDeg\% of the triangles of the D\&D graphs and \EsixteenTriLossEnzDeg\% of those of ENZYMES. The degree histogram and both random GINs stay at the level under deg and wl1, with rejection rates of at most \EsixteenTreeMax{} for all $m\le64$. The signed cycle statistics detect the rewiring with a few graphs. Triangles are the most sensitive on D\&D, whose contact graphs are highly clustered, and $4$-cycles on ENZYMES. The clustering histogram and the spectrum also detect it on D\&D, but need more graphs on ENZYMES. On D\&D the $z$-scores of the same signed counts, which grow with the size of the graph, are much weaker, with power \EsixteenZTriDdEight{} against \EsixteenTriDdEight{} at $m=8$ for triangles, while on ENZYMES the two normalizations perform alike. We chose the normalization per wedge after comparing the two on rewired D\&D graphs in a pilot, which also fixed the rewiring fraction among $\{2,5,25\}\%$, so the D\&D results favour it by construction. Under wl2 only \EsixteenTriLossDdWltwo\% and \EsixteenTriLossEnzWltwo\% of the triangles are removed, because two rounds of colour refinement nearly identify the vertices of these graphs. The degree histogram and the random GINs then stay at the level, with rejection rates of at most \EsixteenWltwoTreeMax, and the signed triangles reach power \EsixteenWltwoTriDdPow{} at $m=64$ on D\&D.

\begin{table}[h]
\centering
\caption{E16: graphs per group needed for power $0.5$ ($m^\ast$, interpolated on $m\in\{4,\dots,64\}$, with ${\le}4$ when reached at $m=4$ and ${>}64$ when not reached) for the three kinds of rewiring, the largest rejection rate without rewiring ($m\in\{8,64\}$, $1200$ replicates each), and the share of triangles removed.}
\label{tab:e16}
\small
\resizebox{\linewidth}{!}{\begin{tabular}{l ccc ccc c}
\toprule
 & \multicolumn{3}{c}{D\&D} & \multicolumn{3}{c}{ENZYMES} & Size \\
\cmidrule(lr){2-4}\cmidrule(lr){5-7}\cmidrule(lr){8-8}
Descriptor & deg & wl1 & wl2 & deg & wl1 & wl2 & max \\
\midrule
Degree histogram (MMD) & \ensuremath{>64} & \ensuremath{>64} & \ensuremath{>64} & \ensuremath{>64} & \ensuremath{>64} & \ensuremath{>64} & 0.058 \\
Random GIN, sum (MMD) & \ensuremath{>64} & \ensuremath{>64} & \ensuremath{>64} & \ensuremath{>64} & \ensuremath{>64} & \ensuremath{>64} & 0.062 \\
Random GIN, mean (MMD) & \ensuremath{>64} & \ensuremath{>64} & \ensuremath{>64} & \ensuremath{>64} & \ensuremath{>64} & \ensuremath{>64} & 0.063 \\
Clustering histogram (MMD) & \ensuremath{\le4} & \ensuremath{\le4} & \ensuremath{>64} & 20.4 & 22.4 & \ensuremath{>64} & 0.058 \\
Graphlet counts (MMD) & 17.2 & 17.6 & \ensuremath{>64} & 15.1 & 20.0 & \ensuremath{>64} & 0.060 \\
Laplacian spectrum (MMD) & 4.2 & 4.2 & \ensuremath{>64} & 22.6 & 34.1 & \ensuremath{>64} & 0.056 \\
Signed triangles, own density & \ensuremath{\le4} & \ensuremath{\le4} & 59.3 & 15.3 & 18.1 & \ensuremath{>64} & 0.062 \\
Signed triangles, degree-corrected & \ensuremath{\le4} & \ensuremath{\le4} & 64.0 & 15.6 & 18.4 & \ensuremath{>64} & 0.058 \\
Signed $4$-cycles, degree-corrected & 4.1 & 4.1 & \ensuremath{>64} & 5.7 & 7.6 & \ensuremath{>64} & 0.058 \\
\midrule
Triangles removed & 25\% & 25\% & 4\% & 23\% & 21\% & 6\% & \\
\bottomrule
\end{tabular}
}
\end{table}

\begin{figure}[h]
\centering
\includegraphics[width=\linewidth]{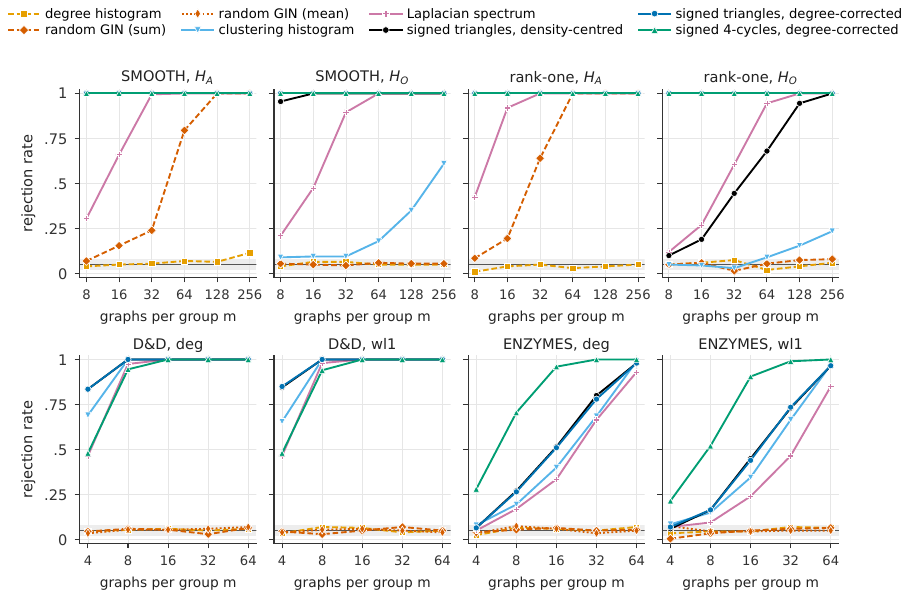}
\caption{Top (E15): rejection rates against $m$ for the degree-preserving differences $H_A$, aligned with the degree profile of the null, and $H_O$, orthogonal to its eigenfunctions, under SMOOTH and the rank-one null ($n=200$, $\delta=\EfifteenDelta$). Bottom (E16): D\&D and ENZYMES graphs against copies with $\EsixteenRho\%$ of their edges rewired by unrestricted swaps (deg) or by swaps between vertices of equal degree (wl1). Dashed curves are tree statistics, and the grey band is $0.05\pm2$ Monte Carlo s.e.}
\label{fig:e15e16}
\end{figure}

\section{Extended related work}\label{app:related}

\paragraph{Two-sample testing for graphs.} With aligned vertices, minimax rates are known for two-sample testing of inhomogeneous random graphs \citep{ghoshdastidar2020two,chatterjee2023two}. Without alignment, the existing tests compare different kinds of summaries:
\begin{itemize}
\item network statistics of one graph per group \citep{ghoshdastidar2017two};
\item latent positions of random dot product graphs \citep{tang2017nonparametric};
\item average subgraph counts, with $n$ fixed and $m\to\infty$ \citep{maugis2020testing};
\item graph cumulants, which remove lower-order effects and improve power at small sample sizes \citep{bravo2023quantifying};
\item Edgeworth-corrected motif statistics for two networks \citep{shao2026higher};
\item block-model and graphon estimates \citep{nguen2025network,sabanayagam2021graphon,chen2019hypothesis};
\item joint laws of subsampled moments \citep{qi2024multivariate}.
\end{itemize}
\citet{ghoshdastidar2017two} prove minimax optimality for tests built on a given statistic, with a single graph per group. Our lower bounds cover all tests and any number of graphs. Their hard instance (Theorem~10 of \citet{ghoshdastidar2017two}) is essentially our alternative. It compares the Erd\H{o}s--R\'enyi graph with a uniform mixture of two-block models over exactly balanced labellings, with edge probabilities $p+\gamma$ within and $p-\gamma$ between the blocks, which is the planted partition \eqref{eq:pp} with $\delta=\gamma$ up to the balancing of the labels. They show that the total variation distance between the two models tends to zero when $\gamma=o((p(1-p)/n)^{1/2})$, that is, when $t\to0$, so that no test with one graph per group has a small total error. \Cref{thm:price}(i) makes this quantitative for $m$ graphs per group, since for every graph size and every $t<1$ a test with errors $(\alpha,\beta)$ needs $m\ge6(1-t)\,c_{\alpha\beta}/t^3$ graphs, and signed triangles attain this order (\cref{thm:price}(ii)). The excess triangle statistic in our experiments is the $T_3$ statistic of \citet{gao2017testing}. We use existing statistics and characterize when they are optimal and what they cost. The effect of shuffled labels on the power of aligned tests is studied by \citet{saxena2025lost}. \citet{chen2026vertex} observe that misalignment hurts changepoint localization when the change lies in the joint law of the latent positions, and is nearly harmless when it lies in their marginal law. This is close in spirit to our distinction between orders of visibility.

\paragraph{Detection thresholds and small subgraphs.} Contiguity of block models and Erd\H{o}s--R\'enyi graphs below the Kesten--Stigum threshold, and the role of cycle counts, are due to \citet{mossel2015reconstruction}, \citet{banerjee2018contiguity} and \citet{banerjee2017optimal}, building on small subgraph conditioning \citep{janson1995random}. The multilayer, aligned analogue is treated by \citet{chatterjee2022multilayer}. Other one-graph problems where the choice of cycle order matters include:
\begin{itemize}
\item signed polygons for degree-corrected models \citep{jin2021optimal};
\item latent geometry, where signed triangles can be optimal \citep{bubeck2016testing} or suboptimal compared with $4$-cycles \citep{bangachev2024detection};
\item planted subgraphs \citep{yu2025counting,elimelech2025detecting}.
\end{itemize}
Correlation detection between unlabeled graphs \citep{wu2023testing,mao2024testing,ding2025low,ameen2025sharp} asks whether two graphs share a hidden alignment, which is a different problem from ours.

\paragraph{Invariant polynomials and low-degree methods.} Low-degree lower bounds exploit invariance under relabelling, and invariant polynomials are combinations of subgraph counts \citep{kunisky2019notes,kunisky2024tensor,carpentier2025low}. Testing between two structured models has recently been studied in this framework for one graph \citep{rush2026count,skeja2026sharp}. Our lower bounds are information-theoretic and exploit independence across graphs.

\paragraph{Fluctuations of subgraph counts.} The fluctuations of subgraph counts in graphon models are Gaussian at scale $n^{-1/2}$ when the graphon is irregular with respect to the motif. When it is regular, they are of smaller order and possibly non-Gaussian \citep{hladky2021limit,bhattacharya2023fluctuations,kaur2021higher,chatterjee2026higher}. Resampling methods for network moments include \citet{green2022bootstrapping}, \citet{lin2020jackknife} and \citet{zhang2022edgeworth}.

\paragraph{Learned two-sample tests and graph neural networks.} Kernel and learned two-sample tests \citep{gretton2012kernel,liu2020learning,liu2021meta,kirchler2020two,kubler2022automl} become permutation-invariant graph tests once the features are invariant. The expressive power of message passing on graphons is characterized by tree densities \citep{grebik2022fractional,boker2023wl,boker2023fine}. The same features are used to evaluate graph generative models \citep{you2018graphrnn,obray2022evaluation,thompson2022evaluation,krimmel2025polygraph}.

\end{document}